\documentclass[10pt]{article}

\usepackage[preprint]{tmlr}
\usepackage{amsmath,amssymb,amsfonts}
\usepackage{graphicx}
\usepackage{algorithm}
\usepackage{algorithmic}
\usepackage{placeins}
\usepackage{hyperref}
\usepackage{url}
\usepackage{bm}
\usepackage{amsthm}
\usepackage{thmtools}
\usepackage{mathtools}
\usepackage{xcolor}
\usepackage{booktabs}
\usepackage{array}
\usepackage{enumitem}
\usepackage{pgfplots}
\usepgfplotslibrary{groupplots}
\usetikzlibrary{arrows.meta,positioning,fit,backgrounds}
\pgfplotsset{compat=1.18}
\definecolor{gdcblue}{HTML}{1F5AA6}
\definecolor{gdcgreen}{HTML}{2A7F62}
\definecolor{gdcorange}{HTML}{C26A2E}
\definecolor{gdcgray}{HTML}{5F6673}
\pgfplotsset{
  gdc axis/.style={
    font=\small,
    label style={font=\small},
    tick label style={font=\scriptsize},
    title style={font=\small\bfseries},
    legend style={font=\scriptsize, draw=none, fill=none, cells={anchor=west}},
    axis line style={black!55},
    tick style={black!55},
    major grid style={black!10},
    grid=major,
    line width=0.9pt,
    tick align=outside
  }
}
\hypersetup{colorlinks=true, linkcolor=blue, citecolor=blue, urlcolor=blue}

\newtheorem{theorem}{Theorem}[section]
\newtheorem{proposition}[theorem]{Proposition}

\newtheorem{corollary}[theorem]{Corollary}
\newtheorem{definition}[theorem]{Definition}
\newtheorem{remark}[theorem]{Remark}
\newtheorem{assumption}[theorem]{Assumption}

\DeclareMathOperator{\grad}{grad}

\DeclareMathOperator{\dive}{div}
\DeclareMathOperator*{\argmin}{arg\,min}

\newcommand{\R}{\mathbb{R}}
\newcommand{\E}{\mathbb{E}}
\newcommand{\norm}[1]{\left\| #1 \right\|}
\newcommand{\inner}[2]{\langle #1,\, #2 \rangle}
\newcommand{\dd}{\mathrm{d}}
\newcommand{\Mhat}{\widehat{\mathcal{M}}^*}
\newcommand{\Cknown}{\mathcal{C}_{\mathrm{known}}}
\newcommand{\Cimpl}{\mathcal{C}_{\mathrm{implicit}}}
\newcommand{\volG}{\mathrm{vol}_G}
\newcommand{\TV}{\mathrm{TV}}

\title{Geometric Distributional Control: Learning Progress with Partial Structural Knowledge}

\author{Tong Wu\\
\addr University of Central Florida}
\date{}

\makeatletter
\def\subsubsection{\@startsection{subsubsection}{3}{\z@}{-1.2ex
plus -0.4ex minus -.2ex}{0.4ex plus .2ex}{\normalsize\bfseries\itshape\raggedright}}
\makeatother
\begin{document}

\maketitle
\makeatletter
\let\AND\@undefined
\makeatother

\begin{abstract}
Real-time control often lies between two limiting regimes. Predictive
optimization and model-based control are powerful when dynamics, physical
parameters, objectives, and online planning models are accurately specified,
while reinforcement learning can avoid explicit modeling but must infer
long-horizon value signals from sequential data and interaction, making training
slow, high-variance, and difficult to scale in large action spaces. This middle
regime appears in deployed systems such as autonomous driving, warehouse and
fulfillment robotics, urban traffic-network control, and autonomous delivery
drones, where partial geometry, physics, rules, or constraints are known but the
local direction of task progress remains uncertain.
Geometric Distributional Control (GDC) exploits this middle regime by
factorizing control into feasibility and progress. Known geometry, rules,
constraints, and response maps define an executable scaffold; progress-weighted
feasible data learns a Bellman-like local value-gradient that selects directions
on this scaffold. This yields a value-guided controller that avoids both global
Bellman recursion and black-box policy learning. The required background
knowledge can be lightweight and partial, such as geometry, simple dynamics,
safety rules, constraints, projections, or lower-level response maps; it need
not specify the full system dynamics or long-horizon objective. Offline, GDC
fits a progress-tilted conditional distribution from short known-feasible
trajectory or optimization snippets with weak signed progress certificates.
Online, its score is composed with explicit constraints, projected through the
scaffold, and applied in receding-horizon feedback. We prove conditions under
which the score descends a data-induced soft progress value and validate GDC on
structured multilevel optimization and SUMO route-progress driving, where it
improves over known-only solvers and learning baselines while preserving
scaffold-enforced feasibility.
\end{abstract}

\section{Introduction}
\label{sec:intro}

\subsection{Background and motivation}

Many deployed control systems are neither fully specified planning problems nor fully
unknown black-box decision processes. They come with partial but actionable structure,
such as actuator limits, local geometry, collision checks, safety rules, flow
constraints, or lower-level solvers, yet still lack the value information needed to
decide which feasible action best advances the task. Predictive optimization and
model-based control can organize such structure into constrained planning problems when
the relevant dynamics, parameters, objectives, and horizons are sufficiently specified
\cite{mayne2000constrained,buerger2024safelearning}. Reinforcement learning offers a
different route by learning value or policy information from sequential data and
interaction \cite{sutton2018reinforcement,levine2018reinforcement}, but long-horizon
value estimation and exploration can be slow, high-variance, and hard to scale when
actions are high-dimensional or safety constraints are strict
\cite{achiam2017cpo,kumar2020cql,kostrikov2022iql}. The systems motivating this paper,
including autonomous driving, warehouse and fulfillment robotics, urban traffic-network
control, and autonomous delivery drones, live in the gap between these two views: they
have enough structure to constrain action, but not enough global knowledge to determine
which feasible local action best advances the task.

The missing information is closely related to the role of a Bellman value in dynamic
programming and optimal control \cite{bellman1957dynamic,bertsekas2017dynamic,
bardi1997optimal}. If the relevant value function were available, it would provide a
principled way to rank locally feasible actions by their downstream consequences. In the
partially structured settings considered here, however, that global Bellman function is
not directly available: interaction effects, future constraints, objectives, or
cross-level couplings may be incomplete or changing. The practical issue is therefore
not only how to enforce feasibility, but how to obtain a useful value-like direction on
top of feasibility.

This paper treats progress, rather than the full Bellman function, as the missing
learned object. We do not ask data to recover a global cost-to-go, relearn the known
scaffold, or provide complete expert demonstrations. Instead, the data consist of short
known-feasible trajectory or optimization segments with weak signed progress
certificates. A segment may wait, yield, reverse, or reposition; it need not be a
globally optimal action. Aggregating such snippets defines a progress-tilted
distribution over feasible local motions, whose score acts as a local progress-value
gradient in data-supported contexts.

GDC implements this separation by using partial structure to anchor the support of the
learned distribution. Known constraints, rules, response maps, and projections define a
compact latent scaffold $\Mhat$ on which feasible behavior is represented, while
score-based distributional learning \cite{score_matching,song2020score,
debortoli2022riemannian} biases motion toward directions with lower progress energy.
Online, the learned score is composed with explicit known-constraint potentials, filters,
and final projections before execution. This makes feasibility easier to preserve than
in a purely black-box policy interface: learning guides the direction of motion, but the
known scaffold still determines what can be executed.

\subsection{Related work}

GDC is a methodology for converting feasible-progress data into an executable geometric
control field. We therefore organize related work by methodological role: value,
distribution, geometry, and feasibility, rather than by the application domains used for
evaluation.

\textbf{Value functions, predictive control, and partial structure.}
Dynamic programming and optimal control use value functions to rank actions by their
future consequences \cite{bellman1957dynamic,bertsekas2017dynamic,bardi1997optimal}.
Model-predictive control uses a specified dynamics model, objective, and horizon to
compute a local action by online constrained optimization
\cite{mayne2000constrained,buerger2024safelearning}. GDC starts from a different
information pattern: the controller has reliable partial structure that can rule out
infeasible actions, but it does not have a complete model or Bellman value that ranks the
remaining feasible directions. The methodological target is therefore a local
progress-value gradient on a known-feasible scaffold, not a full cost-to-go or a complete
online planner.

\textbf{Dataset-supported value learning.}
Offline reinforcement learning learns policies or value functions from fixed datasets
while controlling distribution shift. Conservative and implicit value methods avoid
overly optimistic extrapolation away from the dataset support
\cite{kumar2020cql,kostrikov2022iql}. GDC shares the concern that learned decisions
should remain data-supported, but it weakens the learning target. The dataset need only
contain known-feasible local segments with signed progress certificates; it need not
contain expert rollouts or enough information to identify a reward-optimal policy. The
learned object is a progress energy whose score chooses among feasible local motions, and
the execution layer still enforces explicit constraints.

\textbf{Score-based distributions as control energies.}
Score matching and score-based generative modeling learn gradients of log densities
\cite{score_matching,song2020score}. In decision-making settings, diffusion models have
been used as trajectory planners or expressive policy classes: Diffuser performs planning
through guided denoising \cite{janner2022diffuser}, Diffusion Policy samples actions from
a conditional denoising process \cite{chi2023diffusion}, and diffusion policies have also
been used in offline RL or fine-tuned with policy-gradient objectives
\cite{wang2023diffusionql,zhu2023diffusionsurvey,ren2024dppo}. GDC uses the same
statistical object--a learned conditional density and its score--with a different
interface to control. The distribution is not the policy to execute. Its score is
interpreted as a progress-energy gradient and inserted into a structured latent vector
field together with known-constraint and task potentials.

\textbf{Geometry of learned distributions.}
Riemannian score-based modeling extends score learning to manifold-supported
distributions \cite{debortoli2022riemannian}, while variational formulations of the
Fokker--Planck equation view diffusion dynamics as Wasserstein gradient flows of free
energy \cite{jordan1998variational}. Neural differential equations and physics-informed
learning show complementary ways to encode continuous dynamics or physical residuals in
learned models \cite{chen2018node,raissi2019pinn}. GDC uses these geometric ideas for a
control-specific purpose: partial structure defines the latent scaffold, metric,
potentials, and projection maps on which the learned score is allowed to act. The
geometry is not only a generative modeling device; it is the interface that makes the
progress score composable with explicit feasibility constraints.

\textbf{Explicit feasibility layers.}
Safe and constrained learning methods encode safety as constraints on expected costs,
state-wise conditions, or policy updates
\cite{achiam2017cpo,zhao2023statewise,wachi2024constraints,su2025review}. Runtime shields instead
filter or repair actions proposed by a learned policy
\cite{carr2023shielding,konighofer2025shields}. GDC follows the broader principle that
feasibility should not be left entirely to a black-box policy, but changes where the
constraints enter the method. Known constraints define the training support, latent
scaffold, online potentials, candidate filters, and final projection. Thus feasibility is
not merely a penalty or a post-hoc correction; it is the coordinate system in which the
learned progress direction is represented.

\textbf{Structured solvers as known scaffolds.}
Differentiable optimization layers and implicit-differentiation methods expose the
solution maps of specified optimization problems to learning systems
\cite{amos2017optnet,agrawal2019diffcvx,bertrand2020implicit}. These methods are most
direct when the relevant optimization problem and its solution map are the object to
differentiate through. GDC uses structured solvers more abstractly: a response map,
projection, conservation law, or constraint checker can be part of $\Cknown$, while the
missing object remains the progress direction over the resulting feasible scaffold. This
keeps the methodology independent of any one structured application.

The resulting mathematical shift is:
\begin{equation}
  \underbrace{u_t \sim p_\theta(u \mid o_t)}_{\text{diffusion policy: sampling = control}}
  \quad \longrightarrow \quad
  \underbrace{\dot{z} = s_\theta(z,o,g,\Cknown,H^\star)
  - \nabla\Phi_{\mathrm{known}}(z)
  - \nabla\Phi_{\mathrm{task}}(z)}_{\text{GDC: progress score + known constraints}}.
  \label{eq:core_shift}
\end{equation}
The learned path therefore does not certify feasibility by itself and is not asked to
rediscover the known physics from data. Its role is to rank and connect feasible local
behaviors in the directions that progress certificates support. This energy view also
makes the controller modular: new safety margins, task costs, resource penalties, or
operator preferences can be added as potentials without retraining the progress path.

\subsection{Contributions}

We propose Geometric Distributional Control for partial-physics control problems where
known constraints are reliable but the full directional solution geometry is unavailable.
The main contributions are:

\begin{itemize}[leftmargin=2em]
  \item \textbf{Feasible-progress scaffold for partial physics.}
        GDC constructs an empirical manifold $\Mhat$ from known-feasible action segments
        and conditions it on observations, goals, and explicit known constraints. This
        replaces an unavailable analytic solution geometry with a data-supported
        feasible-progress geometry.

  \item \textbf{Progress-tilted probabilistic paths.}
        Signed progress certificates tilt the conditional distribution of short feasible
        segments toward locally better behavior in expectation, without requiring every
        segment to improve monotonically. Diffusion-style paths are one implementation,
        but any differentiable path or density model that yields a score can be used.

  \item \textbf{Score-guided geometric control.}
        The learned score is used as a progress-energy direction inside a Riemannian
        latent flow, combined with known-constraint potentials, candidate filtering, and
        final projection. Thus probability supplies direction, while explicit physics
        continues to define deployment-time feasibility.

  \item \textbf{Formal connection between progress data and control direction.}
        Under a maximum-entropy demonstration model, we show that the learned score
        approximates the negative gradient of a data-induced soft progress value. We also
        derive stability, score-error robustness, non-accumulating drift, end-to-end
        action-error, and physics--data interpolation results.
\end{itemize}

As supplementary validation, we instantiate the same methodology on parameterized
bilevel and three-level optimization families and on SUMO route-progress driving. These
experiments are used to test whether the method's role separation--learned progress
direction plus explicit known-constraint filtering--transfers across different kinds of
partial structure.

The remainder of the paper is organized as follows.
Section~\ref{sec:formulation} defines the partial-physics control problem.
Section~\ref{sec:manifold} gives the offline and online GDC method.
Section~\ref{sec:formal-guarantees} states the main assumptions and guarantees.
Section~\ref{sec:applications} defines the experimental instantiations and reports the
current simulation results. Section~\ref{sec:discussion} concludes.
Appendix~\ref{sec:appendix_sumo_diagnostics} reports additional SUMO diagnostics, and
Appendix~\ref{sec:theory} contains the expanded progress-value, Riemannian-flow,
stochastic, stability, and physics--data interpolation analysis.

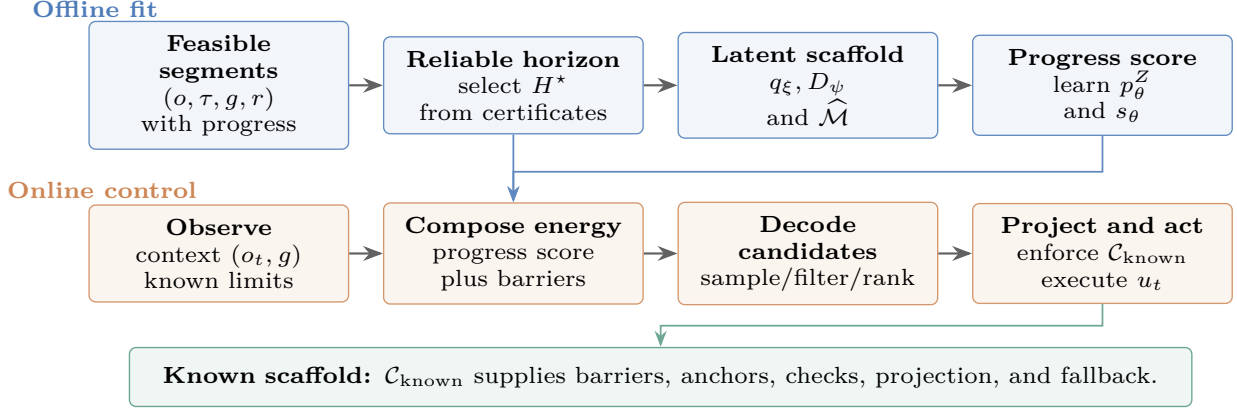
\begin{figure}[!htbp]
\centering
\resizebox{\linewidth}{!}{%
\begin{tikzpicture}[
  font=\scriptsize,
  box/.style={draw=black!45, rounded corners=2pt, align=center, inner sep=4pt,
              minimum height=1.02cm, text width=2.45cm, fill=white},
  offline/.style={box, draw=gdcblue!70, fill=gdcblue!6},
  online/.style={box, draw=gdcorange!70, fill=gdcorange!7},
  guard/.style={draw=gdcgreen!70, fill=gdcgreen!7, rounded corners=2pt,
                align=center, inner sep=4pt, minimum height=0.62cm},
  arrow/.style={-{Stealth[length=2.1mm]}, line width=0.65pt, draw=black!62},
  softarrow/.style={-{Stealth[length=1.7mm]}, line width=0.52pt, draw=black!45},
  offlineflow/.style={softarrow, draw=gdcblue!72},
  scaffoldflow/.style={softarrow, draw=gdcgreen!68}
]
\node[offline] (segments) at (-4.65,1.05)
  {\textbf{Feasible segments}\\$(o,\tau,g,r)$\\with progress};
\node[offline] (horizon) at (-1.55,1.05)
  {\textbf{Reliable horizon}\\select $H^\star$\\from certificates};
\node[offline] (latent) at (1.55,1.05)
  {\textbf{Latent scaffold}\\$q_\xi, D_\psi$\\and $\widehat{\mathcal M}$};
\node[offline] (score) at (4.65,1.05)
  {\textbf{Progress score}\\learn $p_\theta^Z$\\and $s_\theta$};

\node[font=\scriptsize\bfseries, text=gdcblue!80] at (-5.95,1.86) {Offline fit};
\node[font=\scriptsize\bfseries, text=gdcorange!82] at (-5.88,-0.05) {Online control};

\node[online] (observe) at (-4.65,-0.72)
  {\textbf{Observe}\\context $(o_t,g)$\\known limits};
\node[online] (energy) at (-1.55,-0.72)
  {\textbf{Compose energy}\\progress score\\plus barriers};
\node[online] (decode) at (1.55,-0.72)
  {\textbf{Decode}\\\textbf{candidates}\\sample/filter/rank};
\node[online] (execute) at (4.65,-0.72)
  {\textbf{Project and act}\\enforce $\Cknown$\\execute $u_t$};

\node[guard, text width=10.9cm] (known) at (0,-2.02)
  {\textbf{Known scaffold:} $\Cknown$ supplies barriers, anchors, checks, projection, and fallback.};

\draw[arrow] (segments) -- (horizon);
\draw[arrow] (horizon) -- (latent);
\draw[arrow] (latent) -- (score);
\draw[arrow] (observe) -- (energy);
\draw[arrow] (energy) -- (decode);
\draw[arrow] (decode) -- (execute);
\draw[offlineflow] (horizon.south) -- (-1.55,0.14) -- (energy.north);
\draw[offlineflow] (score.south) -- (4.65,0.14) -- (-1.55,0.14) -- (energy.north);
\draw[scaffoldflow] (execute.south) -- (4.65,-1.48) -- (0,-1.48) -- (known.north);
\end{tikzpicture}
}
\caption{Overview of GDC. Known-feasible segments with weak progress
certificates are used offline to build a latent feasible-progress geometry and
score field. Online control follows the induced progress-energy direction,
while the explicit known scaffold remains responsible for filtering,
projection, fallback, and execution safety.}
\label{fig:gdc_overview}
\end{figure}

\section{Problem Formulation}
\label{sec:formulation}

\subsection{Notation}

Scalars are written in italic lowercase letters, vectors in bold
lowercase letters when their component structure matters, and sets in
calligraphic uppercase letters. The system state, observation, control
input, and goal at time $t$ are denoted by $x_t\in\mathcal X$,
$o_t\in\mathcal O$, $u_t\in\mathcal U$, and $g\in\mathcal G$,
respectively. A short action segment is
$\tau_t=(u_t,\ldots,u_{t+H-1})$, and its signed progress certificate is
$\Delta\rho(\tau_t,g)$. The explicit known constraint set is
$\Cknown$, while $\Cimpl$ denotes residual implicit feasibility and
preference structure represented statistically from data. The selected
segment horizon is $H^\star(o,g,\Cknown)$; if no horizon has reliable
positive progress, $H^\star=\varnothing$ and the learned progress
direction is not used. The progress-tilted empirical segment distribution
at tilt parameter $\beta$ is denoted by
$q_\beta(\tau\mid o,g,\Cknown,H^\star)$.

Latent action-segment coordinates are denoted by $z\in\mathcal Z$.
The empirical feasible-progress map is $\Mhat\subseteq\mathcal Z$, with
encoder/posterior $q_\xi(z\mid\tau,o,g,\Cknown)$ and horizon-conditioned decoder
$D_\psi(z,o,g,\Cknown,H^\star)$. The learned progress-tilted latent density is
$p_\theta^Z(z\mid o,g,\Cknown,H^\star)$, its score is
$s_\theta(z,o,g,\Cknown,H^\star)=\nabla_z\log p_\theta^Z(z\mid o,g,\Cknown,H^\star)$,
and $G(z)$ denotes the
decoder-pullback metric used for Riemannian gradients. When decoded segment
or trajectory densities appear, we write $p_\theta^\tau$; superscripts are
omitted only where the density domain is unambiguous. We write
$p_\theta^\tau(\cdot\mid z,g,\Cknown,H^\star)$ for the segment law decoded at a latent
point $z$ and $p_\theta^\tau(\cdot\mid o,g,\Cknown,H^\star)$ for its marginal under
$p_\theta^Z(\cdot\mid o,g,\Cknown,H^\star)$. We write $\|\cdot\|$
for the Euclidean norm unless a metric subscript is shown.

\subsection{Problem setting and data}

\paragraph{Explicit and implicit physics.}

Consider a discrete-time dynamical system
\begin{equation}
  x_{t+1} = f(x_t, u_t),
  \label{eq:dynamics}
\end{equation}
where $x_t \in \mathcal{X} \subseteq \R^n$ is the system state and
$u_t \in \mathcal{U} \subseteq \R^m$ is the control input.
The controller has access to partial observations
\begin{equation}
  o_t = h(x_t) + \varepsilon_t,
  \label{eq:observation}
\end{equation}
where $h : \mathcal{X} \to \mathcal{O}$ is the (possibly nonlinear) observation map
and $\varepsilon_t$ is measurement noise.

GDC assumes that some physical structure is explicit enough to be enforced online. This
includes known constraints, actuator limits, geometric boundaries, conservation laws, or
low-order dynamics that can be evaluated reliably. The remaining structure is treated as
\emph{implicit physics}: it may be unknown, but it may also be known in principle and
simply too complex, coupled, high-dimensional, or expensive to model and optimize online.
We write this bookkeeping distinction as
\begin{equation}
  f = f_{\mathrm{explicit}} + f_{\mathrm{implicit}}, \qquad
  \mathcal{C} = \Cknown \cap \Cimpl,
  \label{eq:implicit_physics_decomposition}
\end{equation}
where $f_{\mathrm{explicit}}$ and $\Cknown$ are the parts handled by explicit geometry,
barriers, anchors, or filters, while $f_{\mathrm{implicit}}$ and $\Cimpl$ denote residual
dynamics, contact effects, cross-layer couplings, latent preferences, and feasibility
patterns that are learned statistically from data. This decomposition is not required to
be a physically identifiable additive split; it specifies what is enforced analytically
and what is represented through the learned progress distribution.

\paragraph{Feasible trajectory data with progress certificates.}

We assume access to offline known-feasible trajectories from which short segments can be
extracted. After the horizon-selection procedure below, the training set is written as
\begin{equation}
  \mathcal{D} = \{(o_i, u_i, \tau_i, g_i, r_i)\}_{i=1}^{N},
  \label{eq:dataset}
\end{equation}
where $\tau_i = (u_{i,t}, \ldots, u_{i,t+H_i-1})$ is a selected
horizon-$H_i$ trajectory segment containing $H_i$ controls, $g_i$
is a goal or task specification, and $r_i$ is a weak progress certificate. We do not
assume that the trajectory segment reaches the final goal, nor that every segment has
positive progress. Instead, each segment is known-feasible with respect to $\Cknown$ and
carries a signed certificate such as
\begin{equation}
  r_i = \Delta\rho(\tau_i,g_i)
  = \rho(x_{i,t+H_i},g_i)-\rho(x_{i,t},g_i),
\end{equation}
for a coarse progress measure $\rho$. Positive $r_i$ indicates local improvement, while
negative or near-zero $r_i$ may still correspond to necessary waiting, backing up,
repositioning, or safety-preserving behavior. The training objective uses these
certificates to bias the learned conditional distribution so that progress is positive in
expectation over short trajectories under the data-supported contexts:
\begin{equation}
  \E_{\tau \sim p_{\mathcal{D}}(\cdot \mid o,g,\Cknown)}
  \!\left[\Delta\rho(\tau,g)\right] > 0,
  \label{eq:data_expected_progress}
\end{equation}
rather than imposing monotone progress on every individual snippet. No dense reward or
complete success labels are required.

\paragraph{Offline horizon selection.}

The segment length is not fixed by hand. For each context neighborhood
$\mathcal{N}(o,g,\Cknown)$ and a candidate set
$\mathcal{H}=\{H_1,\ldots,H_M\}$, GDC chooses the shortest horizon whose progress tilt is
statistically reliable. For each $H\in\mathcal{H}$, cut all known-feasible windows
$\tau_i^H=(u_{i,t},\ldots,u_{i,t+H-1})$ in the neighborhood and compute
$r_i^H=\Delta\rho(\tau_i^H,g_i)
=\rho(x_{i,t+H},g_i)-\rho(x_{i,t},g_i)$. Define the clipped progress weight
\begin{equation}
  w_i^H = \exp\!\left(\beta\,\mathrm{clip}(r_i^H,-c,c)\right),
\end{equation}
and the weighted empirical progress
\begin{equation}
  \hat{\mu}_H(o,g,\Cknown)
  =
  \frac{\sum_{i\in\mathcal{N}(o,g,\Cknown)} w_i^H r_i^H}
       {\sum_{i\in\mathcal{N}(o,g,\Cknown)} w_i^H}.
  \label{eq:horizon_weighted_progress}
\end{equation}
Let $\hat{\sigma}_H$ be the corresponding weighted standard deviation, $N_H$ the
number of feasible windows, and
\begin{equation}
  N_H^{\mathrm{eff}}(o,g,\Cknown)
  =
  \frac{\left(\sum_{i\in\mathcal{N}(o,g,\Cknown)} w_i^H\right)^2}
       {\sum_{i\in\mathcal{N}(o,g,\Cknown)} (w_i^H)^2}
  \label{eq:horizon_effective_sample_size}
\end{equation}
the effective sample size of the tilted neighborhood. Candidate horizons with no feasible
windows in the neighborhood are treated as inadmissible, so the lower confidence bound is
evaluated only when $N_H^{\mathrm{eff}}>0$. We use
\begin{equation}
  \widehat{\mathrm{LCB}}_H(o,g,\Cknown)
  =
  \hat{\mu}_H(o,g,\Cknown)
  -
  \alpha_{\mathrm{conf}}
  \frac{\hat{\sigma}_H(o,g,\Cknown)}
       {\sqrt{N_H^{\mathrm{eff}}(o,g,\Cknown)}}.
  \label{eq:horizon_lcb}
\end{equation}
The selected horizon is
\begin{equation}
  H^\star(o,g,\Cknown)
  =
  \min\left\{
  H\in\mathcal{H}:
  \widehat{\mathrm{LCB}}_H(o,g,\Cknown)>\delta_{\mathrm{prog}},
  \; N_H^{\mathrm{eff}}(o,g,\Cknown)\ge N_{\min}
  \right\}.
  \label{eq:selected_horizon}
\end{equation}
Thus $H^\star$ is the minimum reliable positive-progress horizon: short enough for
online receding-horizon control, but long enough that the expected progress of the
tilted segment distribution is positive with confidence.

If the set in \eqref{eq:selected_horizon} is empty, GDC does not claim a
positive-progress guarantee in that context:
\begin{equation}
  H^\star(o,g,\Cknown)=\varnothing.
  \label{eq:no_reliable_horizon}
\end{equation}
Such data may still train the feasible action-segment map, because it describes what can
be executed safely, but the progress tilt is disabled or downweighted for that context.
Online control then falls back to the known-feasible controller or sets $\lambda\to 1$ in
the physics--data interpolation. This abstention rule prevents the learned path from
pretending that missing data contain a reliable direction.
In the notation below, when the selected horizon is context-dependent, $H^\star$ is
treated as part of the conditioning context and is suppressed unless it matters for the
algorithm.

\paragraph{Control objective.}

The goal is to design a feedback controller
$u_t = \pi(o_t, g)$
that produces safe, task-consistent, and \emph{directionally correct} actions by leveraging
both the geometric structure of $\Mhat$ and the distributional knowledge encoded
in $\mathcal{D}$.

Formally, let $\mathcal{C} \subset \mathcal{U}$ denote the full safety constraint set and
$\ell : \mathcal{O} \times \mathcal{U} \to \R$ a task cost. The ideal constrained
feedback objective is
\begin{equation}
  \min_{\pi} \; \E\!\left[\sum_{t=0}^{T} \ell(o_t, u_t)\right]
  \quad \text{subject to} \quad u_t \in \mathcal{C} \;\; \forall t.
\end{equation}

GDC addresses this by enforcing $\Cknown$ explicitly while learning a data potential
whose score represents the implicit physics, coupled feasibility patterns, and
goal-conditioned progress information captured by feasible trajectory demonstrations with
progress certificates. Consequently, the hard execution guarantee is stated for
$\Cknown$; membership in $\Cimpl$ is a data-supported statistical property rather than an
analytic certificate outside the modeled support.

\section{Geometric Distributional Control Method}
\label{sec:manifold}

\subsection{Overview and offline construction}

GDC has an offline construction stage and an online navigation stage. The term
\emph{Geometric Distributional Control} refers to the complete framework, while
\emph{GDC online control} refers to the frozen feedback controller deployed after
offline learning. The two stages are:
\begin{center}
\begin{tabular}{p{0.18\linewidth}p{0.35\linewidth}p{0.35\linewidth}}
\toprule
\textbf{Stage} & \textbf{Inputs} & \textbf{Outputs / actions} \\
\midrule
Offline learning &
Known-feasible trajectories, progress certificates, goals, observations, and
$\Cknown$ &
Horizon selector $H^\star$, posterior model
$q_\xi(z\mid\tau,o,g,\Cknown)$, decoder $D_\psi$ or likelihood $p_\psi$, progress path
$p_\theta^Z(z\mid o,g,\Cknown,H^\star)$, and score $s_\theta$ \\
\midrule
Online control &
Current observation $o_t$, goal $g$, known constraints $\Cknown$, and fixed offline
models &
Evaluate $H_t^\star$, sample/refine latent candidates $z$, decode candidate action
segments, filter by known constraints, execute the first action $u_t$ \\
\bottomrule
\end{tabular}
\end{center}
No model is retrained online. If the selected horizon is unavailable in the current
context, GDC abstains from using the progress guarantee and falls back to the
known-feasible controller or increases the known-physics weight $(\lambda\to 1)$.

The offline stage has one purpose: turn long feasible behavior into a compact, anchored
distribution of local progress patterns. Rather than assuming complete physics or
globally solved optimal trajectories, GDC starts from known-feasible trajectories and cuts them
into short snippets. The known constraints define the feasible scaffold on which those
snippets can live. The progress certificates then tilt a probabilistic path model toward
snippets that move in locally preferable directions. The probabilistic path approximates
the missing or unobserved physics statistically; the known physics remains explicit as
constraints that reduce and anchor the search space.

Thus the offline stage trains the objects that will later be frozen and used by the
online controller. It has two logically separate components:
\begin{enumerate}[leftmargin=2em]
  \item \textbf{Feasible action-segment embedding:} an encoder-decoder learns a compact
        latent coordinate system $\Mhat$ for known-feasible action segments.
  \item \textbf{Progress density learning:} a progress-tilted probabilistic path learns
        how short feasible snippets tend to move on $\Mhat$ and supplies the score used
        online.
\end{enumerate}
Diffusion is therefore not the method itself and not a second encoder path. It is one
possible gradual implementation of the probabilistic path.

\paragraph{Progress-tilted target distribution.}

For contexts with $H^\star(o,g,\Cknown)\neq\varnothing$, let
$q_{\mathcal{D}}^{H^\star}(\tau \mid o,g,\Cknown)$ denote the empirical distribution of
known-feasible snippets at the selected horizon. A signed progress certificate
$r=\Delta\rho(\tau,g)$ defines the progress-tilted distribution
\begin{equation}
  q_\beta(\tau \mid o,g,\Cknown,H^\star)
  =
  \frac{1}{Z_\beta(o,g,\Cknown,H^\star)}
  q_{\mathcal{D}}^{H^\star}(\tau \mid o,g,\Cknown)
  \exp\!\left(\beta\,\mathrm{clip}(\Delta\rho(\tau,g),-c,c)\right),
  \label{eq:progress_tilted_distribution}
\end{equation}
restricted to snippets satisfying the known constraints. Thus the learned path is
constraint-aware because its training support and conditioning variables include
$\Cknown$, but it is not itself a certificate of hard feasibility outside the empirical
support. This is a soft tilt, not a hard filter. Positive-progress snippets receive larger
likelihood, but temporarily regressive snippets may remain when they represent waiting,
yielding, backing up, or reorientation.
The target is not ``every snippet improves.'' The target is that the learned conditional
distribution has positive expected progress in the data-supported region where
\eqref{eq:selected_horizon} succeeds.
With $\tilde r(\tau,g)=\mathrm{clip}(\Delta\rho(\tau,g),-c,c)$,
\begin{equation}
  \E_{\tau\sim q_\beta}\!\left[\tilde r(\tau,g)\right]
  =
  \frac{\partial}{\partial\beta}\log Z_\beta(o,g,\Cknown,H^\star).
  \label{eq:tilted_progress_moment}
\end{equation}
Thus the assumption \eqref{eq:data_expected_progress} is used only after the horizon test:
the offline data already moves in a locally preferable direction in expectation under the
known-feasible context and selected horizon, and the exponential tilt amplifies this bias
without deleting necessary non-progressing maneuvers. Equivalently, for $\beta\ge 0$, the
training target places more probability mass on larger progress certificates while
remaining restricted to $\Cknown$.

\paragraph{Feasible action-segment embedding.}

The latent action-segment representation is essential for scalability. Directly modeling
or searching over a horizon-$H$ action segment lives in a space of dimension $H d_u$,
which can be prohibitive for multi-level systems, power networks, or autonomous driving.
GDC instead learns a low-dimensional latent variable $z\in\mathcal{Z}$ with
$d_z\ll H d_u$, and trains the probabilistic path on the latent distribution induced by
feasible-progress snippets through this encoder. The online controller therefore samples
and refines $z$, then decodes $z$ into an executable action segment.

When the known physics has geometric coordinates, this latent space can be realized as
\begin{equation}
  \mathcal{Z} = \mathbb{R}^{d_0} \times G_1 \times \cdots \times G_L,
  \label{eq:latent_product_space}
\end{equation}
where the Euclidean factors capture unconstrained latent variation and the Lie group
factors capture known geometric degrees of freedom such as pose, phase, orientation, or
network-flow coordinates. This product/geometric structure is optional; the compact
latent action-segment map is not. The essential scaffold is the known-feasible set
$\Mhat_{\mathrm{known}}(o,g)$ inside this learned low-dimensional search space.

GDC does not need a distillation loss from a complete solution map. The offline object
being embedded is already the behavior segment that will be controlled online: a
known-feasible short action trajectory. Since no ground-truth latent code is observed, we
learn the latent scaffold with a regularized variational action-segment autoencoder, where
$H_k$ is the selected horizon of the training segment $\tau_k$:
\begin{equation}
  q_\xi(z\mid \tau_k,o_k,g_k,\Cknown),
  \qquad
  p_\psi(\tau_k\mid z,o_k,g_k,\Cknown,H_k).
  \label{eq:vae_action_segment_model}
\end{equation}
For notation, a latent code can be sampled from the encoder or set to its posterior mean:
\begin{equation}
  z_k \sim q_\xi(\cdot\mid \tau_k,o_k,g_k,\Cknown),
  \qquad
  \bar z_k=\E_{q_\xi}[z\mid \tau_k,o_k,g_k,\Cknown].
  \label{eq:segment_encoder}
\end{equation}
When deterministic notation is convenient, we write
$E_\xi(\tau_k,o_k,g_k,\Cknown)=\bar z_k$ only as shorthand for this posterior mean; the
learned encoder itself is the variational distribution $q_\xi$.
The decoder distribution reconstructs an executable short action segment:
\begin{equation}
  \hat{\tau}_k = D_\psi(z_k,o_k,g_k,\Cknown,H_k)
  \quad\text{or equivalently}\quad
  \hat{\tau}_k\sim p_\psi(\cdot\mid z_k,o_k,g_k,\Cknown,H_k),
  \label{eq:segment_decoder}
\end{equation}
The first action of the decoded segment is the executable control. State reconstruction is
not a primary objective; states enter only through the observation, the known constraints,
and the progress certificate used to label the segment.

\paragraph{Structured offline objective.}

The feasible map is trained with only the losses needed to make the action-segment
coordinates smooth, executable, and known-feasible:
\begin{equation}
  \mathcal{L}_{\mathrm{map}}
  =
  \mathcal{L}_{\mathrm{vae}}
  + \lambda_k\mathcal{L}_{\mathrm{known}}
  \label{eq:map_loss}
\end{equation}
Here
\begin{equation}
\begin{aligned}
  \mathcal{L}_{\mathrm{vae}}
  &=
  \E_k\!\left[
  \E_{z\sim q_\xi(\cdot\mid\tau_k,o_k,g_k,\Cknown)}
  \left[-\log p_\psi(\tau_k\mid z,o_k,g_k,\Cknown,H_k)\right]
  \right. \\
  &\qquad\left.
  +
  \beta_z\,\mathrm{KL}\!\left(
    q_\xi(z\mid\tau_k,o_k,g_k,\Cknown)\,\|\,p_0(z)
  \right)
  \right],
\end{aligned}
  \label{eq:segment_vae_loss}
\end{equation}
where $p_0(z)$ is a simple prior such as $\mathcal{N}(0,I)$, or a structured prior when
the latent coordinates include known geometric factors.
The known-constraint loss penalizes decoded action segments that violate $\Cknown$:
\begin{equation}
  \mathcal{L}_{\mathrm{known}}
  =
  \E_k
  \E_{z\sim q_\xi(\cdot\mid\tau_k,o_k,g_k,\Cknown)}
  \!\left[\sum_r \left[
    f_r^{\mathrm{known}}(D_\psi(z,o_k,g_k,\Cknown,H_k),o_k,g_k)_+
  \right]^2\right].
  \label{eq:known_constraint_loss}
\end{equation}
The KL term and the low-dimensional bottleneck prevent the encoder from becoming an
identity map and make the latent space suitable for sampling. Additional smoothness or
contrastive regularizers may be added, but they are not progress losses.

After minimizing \eqref{eq:map_loss}, the empirical manifold $\Mhat$ is the data-supported
region of the latent space reached by known-feasible action segments.

We separate $\mathcal{L}_{\mathrm{map}}$ from the path objective because they answer
different questions. The map loss learns only the coordinate system for known-feasible
action segments. It does not upweight positive progress or enforce positive expected
progress. The path loss then learns a conditional density on that scaffold: where
progress-biased feasible behavior has probability mass, and therefore which direction the
controller should move online. Conceptually GDC first learns the feasible action-segment
map and then learns the progress-tilted probability path over that map.

\paragraph{Progress-tilted probabilistic path.}

After the map exists, GDC fits a progress-tilted probabilistic path over the latent codes
on $\Mhat$. This path learns the recurring local movement patterns inside the segmented
feasible trajectories: how one known-feasible local behavior tends to transition toward
another under the same goal and constraint context. These patterns are the statistical
approximation of the missing physics and cross-level couplings that were not included in
$\Cknown$. The known physics does not merely appear as a conditioning variable; it shrinks
and anchors the path. Define the known-feasible latent scaffold
\begin{equation}
  \Mhat_{\mathrm{known}}(o,g)
  =
  \left\{
    z\in\Mhat :
    f_r^{\mathrm{known}}(D_\psi(z,o,g,\Cknown,H^\star),o,g)\le 0,\; r=1,\ldots,R
  \right\}.
  \label{eq:known_feasible_latent_scaffold}
\end{equation}
The probability path is learned on, and sampled back toward, this anchored scaffold:
\begin{equation}
  p_\theta^Z(z\mid o,g,\Cknown,H^\star)
  \;\text{is supported or strongly concentrated on}\;
  \Mhat_{\mathrm{known}}(o,g).
  \label{eq:path_known_support}
\end{equation}
Thus the path is not asked to rediscover the known physics from data or explore the full
trajectory space. The known constraints act as fixed anchors that hold the distribution to
the feasible scaffold, while the progress tilt learns how feasible short segments move
across the parts of the system whose physics is missing or only partially observed.

A diffusion-style path is one implementation: it gradually approximates the clean
feasible-progress distribution by moving along a sequence of corrupted distributions. For
a path step $\ell$ and a clean latent point
$z_k\sim q_\xi(\cdot\mid\tau_k,o_k,g_k,\Cknown)$,
\begin{equation}
  z_{k,\ell} =
  \sqrt{\bar\alpha_\ell}\,z_k
  + \sqrt{1-\bar\alpha_\ell}\,\epsilon,
  \qquad \epsilon\sim\mathcal{N}(0,I).
\end{equation}
Following recent direct clean-data prediction views of generative path models, the network
predicts the clean latent behavior rather than the injected perturbation. Let
\begin{equation}
  \hat{z}_{0,k}=F_\theta(z_{k,\ell},o_k,g_k,\Cknown,H_k,\ell),
  \qquad
  \hat{\tau}_{0,k}=D_\psi(\hat{z}_{0,k},o_k,g_k,\Cknown,H_k),
  \qquad
  \hat{u}_{0,k}=[\hat{\tau}_{0,k}]_0 .
\end{equation}
The progress-weighted clean-behavior prediction objective is anchored by the known
constraints:
\begin{equation}
  \mathcal{L}_{\mathrm{path}}
  =
  \E_{k,\ell,\epsilon}\!\left[w_k\,\mathcal{B}_k\right],
  \qquad
  w_k=\exp\!\left(\beta\,\mathrm{clip}(r_k,-c,c)\right).
  \label{eq:path_embedding_loss}
\end{equation}
where
\begin{equation}
  \begin{aligned}
  \mathcal{B}_k
  &=
  \norm{\hat{z}_{0,k}-z_k}^2
  + \lambda_{\tau}\norm{D_\psi(\hat{z}_{0,k},o_k,g_k,\Cknown,H_k)-\tau_k}^2 \\
  &\quad
  + \lambda_{\mathrm{anc}}
  \mathcal{A}_{\mathrm{known}}(\hat{z}_{0,k},\hat{u}_{0,k},\hat{\tau}_{0,k}).
  \end{aligned}
  \label{eq:path_clean_anchor_cost}
\end{equation}
The anchor penalty is
\begin{equation}
  \begin{aligned}
  \mathcal{A}_{\mathrm{known}}(\hat{z},\hat{u},\hat{\tau})
  &=
  \sum_r \left[f_r^{\mathrm{known}}(\hat{u},o,g)_+\right]^2
  + \sum_{t}\sum_r \left[f_r^{\mathrm{known}}(\hat{\tau}_t,o,g)_+\right]^2.
  \end{aligned}
  \label{eq:known_anchor_penalty}
\end{equation}
The weighted objective is the empirical implementation of the tilted target
\eqref{eq:progress_tilted_distribution}. Let $q_\beta^Z$ denote the progress-weighted
aggregate posterior induced by the variational encoder:
\begin{equation}
  q_\beta^Z(z\mid o,g,\Cknown,H^\star)
  =
  \int
  q_\xi(z\mid\tau,o,g,\Cknown)\,
  q_\beta(\tau\mid o,g,\Cknown,H^\star)\,
  \dd\tau.
  \label{eq:latent_tilted_pushforward}
\end{equation}
Minimizing \eqref{eq:path_embedding_loss} fits the latent path density
$p_\theta^Z(z\mid o,g,\Cknown,H^\star)$ to this tilted latent distribution, subject to the
known anchors \eqref{eq:known_feasible_latent_scaffold}--\eqref{eq:known_anchor_penalty}.
Let $p_\theta^\tau$ denote the decoded segment distribution induced by
$z\sim p_\theta^Z(\cdot\mid o,g,\Cknown,H^\star)$ and
$\tau=D_\psi(z,o,g,\Cknown,H^\star)$.
Hence the data moment \eqref{eq:data_expected_progress} appears online, only when
$H^\star\neq\varnothing$, as the progress-bias property
\begin{equation}
  \E_{\tau\sim p_\theta^\tau(\cdot\mid o,g,\Cknown,H^\star)}
  \!\left[\Delta\rho(\tau,g)\right]
  \approx
  \E_{\tau\sim q_\beta(\cdot\mid o,g,\Cknown,H^\star)}
  \!\left[\Delta\rho(\tau,g)\right]
  >0
  \label{eq:path_progress_bias}
\end{equation}
inside the data-supported known-feasible region. The approximation error is the density
fitting error of the probabilistic path model plus decoder error from the latent map.
This objective gives larger density to latent regions associated with positive signed
progress while retaining feasible waiting, yielding, or repositioning snippets when they
are present in the data. The result is a conditional latent density
$p_\theta^Z(z\mid o,g,\Cknown,H^\star)$ and score
$s_\theta(z,o,g,\Cknown,H^\star)=
\nabla_z\log p_\theta^Z(z\mid o,g,\Cknown,H^\star)$ on the learned map.
This score is the online progress direction used by GDC.

\paragraph{Offline output.}

The offline stage returns the action-segment posterior model
$q_\xi(z\mid\tau,o,g,\Cknown)$, the decoder $D_\psi$ (or likelihood $p_\psi$), the clean
latent predictor $F_\theta$, and the
conditional path score
$s_\theta(z,o,g,\Cknown,H^\star)$. The empirical manifold $\Mhat$ is the data-supported,
known-feasible region of the action-segment map, equipped with a progress-tilted path
density.

\subsection{Geometry and progress score}

\begin{definition}[Pullback metric on $\Mhat$]
  The Riemannian metric on $\Mhat$ is the decoder-pullback metric:
  \begin{equation}
    G(z) = J_D(z)^\top J_D(z) + \lambda_m I,
    \label{eq:pullback_metric}
  \end{equation}
  where $J_D$ is the Jacobian of the action-segment decoder, for example
  $D(z)=D_\psi(z,o,g,\Cknown,H^\star)$, and $\lambda_m>0$ ensures positive definiteness.
\end{definition}

For a smooth function $\Phi : \Mhat \to \R$, the Riemannian gradient is
\begin{equation}
  \grad_G \Phi(z) = G(z)^{-1} \nabla_z \Phi(z),
  \label{eq:riemannian_gradient}
\end{equation}
and the Riemannian gradient flow is
\begin{equation}
  \dot{z} = -\grad_G \Phi(z) = -G(z)^{-1} \nabla_z \Phi(z).
  \label{eq:riemannian_flow}
\end{equation}

\begin{proposition}[Invariance]
  The Riemannian gradient flow \eqref{eq:riemannian_flow} is invariant under
  diffeomorphic reparameterizations of $\Mhat$.
\end{proposition}

\begin{proof}
Let $\varphi:\widetilde{\mathcal M}\to\Mhat$ be a $C^1$ diffeomorphism and write
$z=\varphi(\tilde z)$. The pullback metric is
$\tilde G(\tilde z)=J_\varphi(\tilde z)^\top G(z)J_\varphi(\tilde z)$ and the pulled-back
potential is $\tilde\Phi(\tilde z)=\Phi(\varphi(\tilde z))$. By the chain rule,
$\nabla_{\tilde z}\tilde\Phi=J_\varphi^\top\nabla_z\Phi$. The gradient vector field in
the $\tilde z$ coordinates is the unique vector $\widetilde X$ satisfying
$\tilde G(\widetilde X,\widetilde v)=\dd\tilde\Phi[\widetilde v]$ for every tangent
vector $\widetilde v$. Substituting the definitions gives
$J_\varphi\widetilde X=G^{-1}\nabla_z\Phi=\grad_G\Phi$. Therefore the pushed-forward
flow obeys
$\frac{\dd}{\dd t}\varphi(\tilde z_t)=J_\varphi\dot{\tilde z}_t
=-\grad_G\Phi(z_t)$, which is exactly \eqref{eq:riemannian_flow} in the original
coordinates.
\end{proof}

\paragraph{Offline progress-path implementation and score extraction.}
\label{sec:diffusion}

This subsection specifies one offline implementation of the progress-tilted latent path.
After training, the online controller treats $q_\xi$, $D_\psi$, $H^\star$, and
$p_\theta^Z$ as fixed. The learned object is the progress-tilted latent path density
$p_\theta^Z(z\mid o,g,\Cknown,H^\star)$; the online score
$s_\theta=\nabla_z\log p_\theta^Z$ is derived from this density. Known physics is not left
implicit in the data distribution: it appears as constraints, barriers, anchors, and
filters. The path density approximates what is missing from the explicit model: implicit
physics, latent preference structure, cross-level coupling, and long-horizon directional
information.

\paragraph{Diffusion-style gradual path instantiation.}

Let $z_0 \sim q_{\text{data}}(z)$ denote a latent behavior sampled from the data
distribution on $\Mhat$. A convenient implementation is a diffusion-style gradual path:
one moves from clean feasible-progress behaviors to corrupted latent points and learns to
predict the clean behavior back. This is not required by GDC; any differentiable
conditional density or path model that yields $\nabla_z \log p_\theta^Z$ can be used. For
the diffusion-style implementation, the forward path adds noise within the tangent bundle
$T\Mhat$:
\begin{equation}
  q(z_\ell \mid z_{\ell-1}) = \mathcal{N}_{\Mhat}\!\left(z_\ell;\,
  \sqrt{1 - \beta_\ell}\,z_{\ell-1},\, \beta_\ell G(z_{\ell-1})^{-1}\right),
  \label{eq:manifold_forward}
\end{equation}
where $\mathcal{N}_{\Mhat}$ denotes the Riemannian normal distribution
on $\Mhat$, and $G(z)^{-1}$ ensures noise is scaled according to the
local geometry.

In practice, when $\Mhat$ is approximately flat at the scale of the noise
(which holds for well-trained manifolds with sufficient latent dimension), we use the
Euclidean approximation:
\begin{equation}
  q(z_\ell \mid z_{\ell-1}) \approx \mathcal{N}\!\left(z_\ell;\,
  \sqrt{1 - \beta_\ell}\,z_{\ell-1},\, \beta_\ell I\right),
  \label{eq:euclidean_approx_forward}
\end{equation}
with $\{\beta_\ell\}_{\ell=1}^T$ a noise schedule, $\alpha_\ell = 1 - \beta_\ell$, and
$\bar\alpha_\ell = \prod_{j=1}^{\ell} \alpha_j$. The marginal is
\begin{equation}
  q(z_\ell \mid z_0) = \mathcal{N}\!\left(z_\ell;\,
  \sqrt{\bar\alpha_\ell}\,z_0,\, (1 - \bar\alpha_\ell) I\right).
\end{equation}

\paragraph{Clean latent predictor fixed after offline training.}

The probabilistic path is a model of \emph{what action or short action trajectory should be
executed next}, conditioned on the current observation, goal, and known constraints. It is
not primarily a next-state predictor, and it does not directly predict the full
high-dimensional action segment as its core output. The path network predicts the clean
latent behavior, and the fixed decoder $D_\psi$ maps that latent point to the executable
action segment. This choice
follows the clean-data prediction principle advocated by Li and He \cite{li2025back}:
clean data lie on the low-dimensional data manifold, while corrupted quantities generally
do not. In GDC, this means predicting clean latent action-segment codes, not merely a
latent perturbation residual and not a full action sequence in ambient dimension:
\begin{equation}
  \hat{z}_0 = F_\theta(z_\ell,o,g,\Cknown,H^\star,\ell),\qquad
  \hat{\tau}_0 = D_\psi(\hat{z}_0,o,g,\Cknown,H^\star),\qquad
  \hat{u}_0 = [\hat{\tau}_0]_0 .
  \label{eq:clean_behavior_predictor}
\end{equation}
Offline training uses this clean latent target to implement the path objective
\eqref{eq:path_embedding_loss}: for a corrupted latent point
$z_\ell=\sqrt{\bar\alpha_\ell}z_0+\sqrt{1-\bar\alpha_\ell}\epsilon$,
$\epsilon \sim \mathcal{N}(0,I)$, the model predicts $z_0$ and weights the latent
prediction by the progress tilt induced by the signed certificate. A decoded trajectory
loss through $D_\psi$ may be used as a consistency term, but direct action prediction is
not the core path model. State quantities are not part of the core path loss; when
available, they are used only to evaluate known dynamics, check constraints, or compute
progress certificates. During online control $F_\theta$ and $D_\psi$ are fixed; the
clean predictor may be queried to estimate the path score, but it does not replace the
online latent control dynamics.

\begin{remark}[Known constraints versus learned feasibility]
  Since every training sample $z_0$ lies on $\Mhat$ and is filtered or penalized by
  $\Cknown$, the learned path concentrates its density on the data-supported
  progress-biased region. This does not prove feasibility for all implicit physics;
  rather, it makes implicit feasibility a statistical property of feasible trajectory
  demonstrations with progress certificates.
  Known constraints are still enforced explicitly through conditioning, barrier terms,
  anchor penalties, and candidate filtering during control.
\end{remark}

\paragraph{Score induced by clean behavior prediction.}

Although the network predicts clean behavior directly, it still induces a score field for online
control. For the diffusion-style gradual path, under the Gaussian forward kernel
$z_\ell=\sqrt{\bar\alpha_\ell}z_0+\sqrt{1-\bar\alpha_\ell}\epsilon$, replacing $z_0$ by its
clean prediction gives the score approximation
\begin{equation}
  s_{\theta,\ell}^{\mathrm{diff}}(z_\ell,o,g,\Cknown,H^\star)
  \approx
  \frac{\sqrt{\bar\alpha_\ell}\,F_\theta(z_\ell,o,g,\Cknown,H^\star,\ell)-z_\ell}
       {1-\bar\alpha_\ell}.
  \label{eq:score_from_clean_prediction}
\end{equation}
Thus GDC can use direct action/trajectory prediction during training while still obtaining
the score direction needed for Riemannian gradient-flow control.


\subsection{Online progress-energy navigation}
\label{sec:online}

At deployment, all learned objects are frozen. GDC receives the action-segment posterior
$q_\xi$, decoder $D_\psi$, horizon selector $H^\star$, progress-tilted latent density
$p_\theta^Z$, and score $s_\theta=\nabla_z\log p_\theta^Z$. The online stage is therefore not
another learning or diffusion-sampling procedure. It is a geometric navigator on the
learned feasible-progress map: the controller localizes the current context, forms a
composable latent energy, follows its induced vector field for a fixed number of small
steps, decodes an action segment, and applies the final known-constraint projection before
executing the first action.

This is the GDC analogue of Pareto-map navigation in GPC. The difference is the source of
the tangential direction. In GPC, the tangential energy comes from a known Pareto solution
map and semantic priority coordinate. In GDC, no complete Pareto or physics solution map is
available; the tangential data energy is instead induced by the progress-tilted
probabilistic path learned from feasible-progress segments. Thus the learned distribution
is not executed as a policy. It becomes an energy term that can be combined with known
physics, task costs, safety margins, and fallback logic.

\paragraph{Context-consistent latent localization.}
At time $t$, the controller observes $o_t$ and evaluates
$H_t^\star=H^\star(o_t,g,\Cknown)$. If $H_t^\star=\varnothing$, the progress guarantee is
not invoked and the controller falls back to a known-feasible backup. Otherwise, GDC
constructs a local data-supported candidate set
\begin{equation}
  \mathcal M_t
  =
  \left\{
    z\in\Mhat:
    H_t^\star\neq\varnothing,\;
    D_\psi(z,o_t,g,\Cknown,H_t^\star)\text{ is locally known-feasible}
  \right\}.
  \label{eq:online_local_set}
\end{equation}
In implementation, $\mathcal M_t$ is represented by warm-starts from the previous accepted
latent state, samples from $p_\theta^Z(\cdot\mid o_t,g,\Cknown,H_t^\star)$, or nearby latent
codes whose decoded segments pass the known-constraint filter. This plays the role of
manifold localization: it anchors online navigation to the part of the learned map that is
consistent with the current observation and known constraints.

\paragraph{Online energy and nominal vector field.}
The direction of online motion is determined by the composite latent energy
\begin{equation}
  \Phi_{\mathrm{on},t}(z)
  =
  \frac{1}{\epsilon_{\mathrm{s}}}\Phi_{\mathrm{scaf}}(z;o_t,\Cknown)
  + \lambda \Phi_{\mathrm{known}}(z;o_t,\Cknown)
  + (1-\lambda)\Phi_{\mathrm{data}}(z;o_t,g,\Cknown,H_t^\star)
  + \Phi_{\mathrm{task}}(z;o_t,g),
  \label{eq:online_energy}
\end{equation}
where $\Phi_{\mathrm{scaf}}$ is a fast restoration term that keeps $z$ near
$\mathcal M_t$, $\Phi_{\mathrm{known}}$ contains available physics and explicit
constraints, $\Phi_{\mathrm{task}}$ encodes the runtime objective, and
\begin{equation}
  \Phi_{\mathrm{data}}(z;o_t,g,\Cknown,H_t^\star)
  =
  -\log p_\theta^Z(z\mid o_t,g,\Cknown,H_t^\star)
  \label{eq:online_data_energy}
\end{equation}
is the progress energy learned offline. The nominal vector field is the Riemannian descent
flow
\begin{equation}
  \begin{aligned}
  F_t(z)
  :=\dot z
  &= -\grad_G \Phi_{\mathrm{on},t}(z) \\
  &= -G(z)^{-1}\!\left[
    \frac{1}{\epsilon_{\mathrm{s}}}\nabla_z\Phi_{\mathrm{scaf}}(z)
    + \lambda\nabla_z\Phi_{\mathrm{known}}(z)
    - (1-\lambda)s_\theta(z,o_t,g,\Cknown,H_t^\star)
    + \nabla_z\Phi_{\mathrm{task}}(z)
  \right].
  \end{aligned}
  \label{eq:online_latent_flow}
\end{equation}
The scalar $\epsilon_{\mathrm{s}}$ induces the same snap-then-slide structure as a geometric navigator:
the scaffold term rapidly restores the latent state to the known-feasible support, while
the data and task terms move tangentially toward progress. The learned path can therefore
be kept fixed while new physics, safety, preference, or resource terms are added to the
energy.

\paragraph{Geometry-aware latent evolution.}
Given $F_t$, GDC advances each candidate latent state with a fixed-depth integrator rather
than an online nonlinear optimizer. A simple second-order implementation is
\begin{align}
  k_1 &= \Pi_{\|\cdot\|_G\le V_{\max}} F_t(z_n), \\
  z_{\mathrm{mid}} &= z_n + \tfrac{\Delta t}{2}\,k_1, \\
  k_2 &= \Pi_{\|\cdot\|_G\le V_{\max}} F_t(z_{\mathrm{mid}}), \\
  \Delta z_n &= \Delta t\,k_2,
  \label{eq:online_rk2}
\end{align}
where $\Pi_{\|\cdot\|_G\le V_{\max}}$ caps the latent velocity in the decoder-pullback
metric. The increment is then integrated in the latent chart or product geometry:
\begin{equation}
  \bar z_{n+1}
  =
  z_n \oplus \Delta z_n,
  \label{eq:online_latent_integration}
\end{equation}
where $\oplus$ is ordinary addition in a flat latent chart and the corresponding group or
retraction update when structured latent factors are used. Finally, the integrated point
is accepted only after a scaffold correction:
\begin{equation}
  z_{n+1}
  =
  (1-\mu_R)\bar z_{n+1}
  +
  \mu_R\,\Pi_{\mathcal M_t}(\bar z_{n+1}),
  \label{eq:online_retraction}
\end{equation}
with $\mu_R\in[0,1]$. In the common case where $\bar z_{n+1}$ already decodes to a
known-feasible segment, $\Pi_{\mathcal M_t}$ is the identity. Otherwise it is implemented
by the simplest problem-specific retraction, clipping, or backtracking rule. Thus the
bridge from the differential law to decoding is explicit:
\[
  \dot z
  \;\longrightarrow\;
  \Delta z_n
  \;\longrightarrow\;
  \bar z_{n+1}
  \;\longrightarrow\;
  z_{n+1}\in\mathcal M_t
  \;\longrightarrow\;
  D_\psi(z_{n+1},o_t,g,\Cknown,H_t^\star).
\]
The velocity cap and retraction are not optimizers; they are numerical guards that keep
the online state in the region where the decoder and score are trusted.

\paragraph{Decode, project, and execute.}
After $N_r$ latent evolution steps, each candidate has an accepted post-navigation latent
state $z_{t,\mathrm{nav}}^{(k)}:=z_{N_r}^{(k)}\in\mathcal M_t$. This is the latent point
that is decoded:
\begin{equation}
  \tau_t^{(k)}
  =
  D_\psi(z_{t,\mathrm{nav}}^{(k)},o_t,g,\Cknown,H_t^\star).
  \label{eq:online_decode_segment}
\end{equation}
GDC selects a feasible-progress candidate using the task score $J_{\mathrm{sel}}$ and
takes the first decoded action
$u_t^{\mathrm{nom}}=\tau_{t,0}^{\mathrm{nom}}$. The final executed action is the weighted
projection onto the current known-feasible action set:
\begin{equation}
  u_t
  =
  \Pi_{\mathcal U_{\mathrm{known}}(o_t,u_{t-1})}^{W}
  (u_t^{\mathrm{nom}})
  :=
  \argmin_{u\in \mathcal U_{\mathrm{known}}(o_t,u_{t-1})}
  \frac{1}{2}\|u-u_t^{\mathrm{nom}}\|_{W}^{2}.
  \label{eq:online_projection}
\end{equation}
For box constraints this is clipping; for equality constraints it has the standard
weighted-projection closed form. If no decoded candidate or final projection is feasible,
the controller resamples or uses the known-feasible backup. The projection is only the last
execution guard: the main online controller is the energy-induced latent navigation in
\eqref{eq:online_latent_flow}.

\begin{algorithm}[!htbp]
\caption{GDC online progress-energy navigation}
\label{alg:gdc}
\small
\begin{algorithmic}[1]
\STATE \textbf{Input:} current observation $o_t$, goal $g$, previous action $u_{t-1}$,
       frozen models $(q_\xi,D_\psi,H^\star,p_\theta^Z,s_\theta)$, known-constraint
       checker, projection layer, and selection cost $J_{\mathrm{sel}}$.
\STATE \textbf{Output:} executable action $u_t$.
\STATE Evaluate $H_t^\star=H^\star(o_t,g,\Cknown)$. If
       $H_t^\star=\varnothing$, return the known-feasible fallback action.
\STATE Build or warm-start $K$ candidates $z_0^{(1:K)}$ in the local
       known-feasible latent set $\mathcal M_t$.
\STATE Form $\Phi_{\mathrm{on},t}$ and its Riemannian vector field
       $F_t=-\grad_G\Phi_{\mathrm{on},t}$.
\FOR{$k=1,\ldots,K$}
  \STATE Take $N_r$ capped RK2/retraction steps using
         Eqs.~\eqref{eq:online_rk2}--\eqref{eq:online_retraction}.
  \STATE Decode $\tau_t^{(k)}=D_\psi(z_{N_r}^{(k)},o_t,g,\Cknown,H_t^\star)$
         and discard candidates that violate $\Cknown$.
\ENDFOR
\STATE If no decoded candidate is feasible, return the known-feasible fallback action.
\STATE Select a feasible segment minimizing $J_{\mathrm{sel}}(\tau,o_t,g)$ and set
       $u_t^{\mathrm{nom}}$ to its first action.
\STATE Project
       $u_t=\Pi_{\mathcal U_{\mathrm{known}}(o_t,u_{t-1})}^{W}(u_t^{\mathrm{nom}})$;
       if projection fails, return the known-feasible fallback action.
\STATE Execute $u_t$ and replan at the next observation.
\end{algorithmic}
\end{algorithm}

The dominant online cost is fixed-depth navigation:
$O(KN_r(C_\theta+C_{\mathrm{known}})+C_{\mathrm{proj}})$. No model is retrained online and
no full nonlinear planning problem is solved. The repeated online operation is evaluating
the frozen score/energy, integrating a bounded latent vector field, decoding candidate
action segments, and applying the final known-constraint projection.

\section{Formal Assumptions and Main Guarantees}
\label{sec:formal-guarantees}

This section states the main conditions under which GDC turns feasible-progress data into
a usable online controller. The expanded derivations are given in
Appendix~\ref{sec:theory}; here we state only the assumptions and the guarantees they
yield.

\begin{assumption}[Data-supported feasible-progress scaffold]
\label{ass:main_scaffold}
For each context $(o,g,\Cknown)$ in the modeled operating envelope, the offline dataset
contains known-feasible action segments whose decoded support lies within the explicit
known-constraint scaffold. When the horizon selector returns
$H^\star(o,g,\Cknown)\neq\varnothing$, the progress-tilted conditional distribution has
positive expected signed progress:
\begin{equation}
  \E_{\tau\sim q_\beta(\cdot\mid o,g,\Cknown,H^\star)}
  [\Delta\rho(\tau,g)] > 0 .
  \label{eq:main_positive_progress}
\end{equation}
If this condition cannot be certified from data, GDC abstains from using the learned
progress direction and falls back to the known-feasible controller.
\end{assumption}

\begin{assumption}[Regular latent map and explicit execution guard]
\label{ass:main_regular}
The decoder $D_\psi$ is Lipschitz on the data-supported neighborhood of $\Mhat$, the
known-constraint functions are locally Lipschitz after decoding, and the online controller
uses bounded latent velocity, scaffold retraction or backtracking, decoded candidate
filtering, and a final projection onto the current known-feasible action set
$\mathcal U_{\mathrm{known}}(o_t,u_{t-1})$. This set is closed and nonempty whenever the
projection layer is invoked, and the fallback action belongs to it.
\end{assumption}

\begin{theorem}[Progress score as an implicit value gradient]
\label{thm:main_score_value}
Under Assumption~\ref{ass:main_scaffold} and a maximum-entropy model of feasible progress
demonstrations, suppose the continuation partition function is finite, positive, and
$C^1$ in $z$ on the selected data-supported neighborhood, and that it induces
$p^\star(z\mid c)\propto Z_H(z,c)$ with respect to the chosen latent reference measure for
the selected context $c=(o,g,\Cknown,H^\star)$. If the learned latent path satisfies
$\|\nabla_z\log p_\theta^Z-\nabla_z\log p^\star\|\le
\varepsilon_{\mathrm{approx}}$ on the data-supported neighborhood, then the
progress-tilted path density defines a data-induced soft progress value
$V_{\mathrm{prog}}^{\mathrm{soft}}$ on $\Mhat$ such that
\begin{equation}
  s_\theta(z,o,g,\Cknown,H^\star)
  =
  \nabla_z\log p_\theta^Z(z\mid o,g,\Cknown,H^\star)
  =
  -\nabla_z V_{\mathrm{prog}}^{\mathrm{soft}}(z,c)
  + \mathcal O(\varepsilon_{\mathrm{approx}}),
  \label{eq:main_score_value}
\end{equation}
where $\varepsilon_{\mathrm{approx}}$ is the path-model approximation error.
\end{theorem}

\begin{proof}
Fix a context $c=(o,g,\Cknown,H^\star)$ for which the horizon selector does not abstain.
Under the maximum-entropy demonstration model, feasible future segments starting from
latent point $z$ have an unnormalized partition function
$Z_H(z,c)=\int\exp[-C_H(\tau;c)]\,\dd\nu_z(\tau)$ over known-feasible continuations, with
soft progress value $V_{\mathrm{prog}}^{\mathrm{soft}}(z,c)=-\log Z_H(z,c)$. The model
posits that the induced ideal latent density is proportional to this partition function:
$p^\star(z\mid c)=Z_c^{-1}\exp[-V_{\mathrm{prog}}^{\mathrm{soft}}(z,c)]$. Hence
$\nabla_z\log p^\star(z\mid c)=-\nabla_zV_{\mathrm{prog}}^{\mathrm{soft}}(z,c)$. If the
learned path satisfies
$\|\nabla_z\log p_\theta^Z-\nabla_z\log p^\star\|\le
\varepsilon_{\mathrm{approx}}$ on the data-supported neighborhood, then
\eqref{eq:main_score_value} follows. Appendix~\ref{sec:energy} gives the full
regularized statement.
\end{proof}

Theorem~\ref{thm:main_score_value} is the central directional guarantee. It says that the
score is not merely a sampling vector field: after progress-tilted training, it is a
descent direction for the progress value represented by the data policy. Thus the learned
distribution can choose among locally known-feasible options whose explicit constraints
alone do not determine a direction.

\begin{proposition}[Known-feasible online execution]
\label{prop:main_feasible_execution}
Under Assumption~\ref{ass:main_regular}, every action executed by Algorithm~\ref{alg:gdc}
belongs to the current known-feasible action set whenever the projection layer is
available:
\begin{equation}
  u_t \in \mathcal U_{\mathrm{known}}(o_t,u_{t-1}) .
  \label{eq:main_known_feasible_execution}
\end{equation}
If no decoded candidate or projected action is feasible, the controller resamples or uses
the known-feasible fallback.
\end{proposition}

\begin{proof}
Algorithm~\ref{alg:gdc} decodes only candidate segments that pass the known-constraint
filter. Its final executed command is
$u_t=\Pi_{\mathcal U_{\mathrm{known}}(o_t,u_{t-1})}^W(u_t^{\mathrm{nom}})$ whenever the
projection exists. By definition of projection onto a nonempty closed known-feasible set,
the projected point belongs to $\mathcal U_{\mathrm{known}}(o_t,u_{t-1})$. If the
projection layer is unavailable or fails, or if no candidate survives filtering, the
algorithm returns the known-feasible fallback required by
Assumption~\ref{ass:main_regular}.
Therefore every executed non-fallback or fallback action lies in the current
known-feasible action set.
\end{proof}

\begin{corollary}[Non-accumulating GDC action error]
\label{thm:main_action_error}
Suppose the ideal progress value is locally strongly geodesically convex near a target
latent point with modulus $m$, the learned score error is bounded by $\varepsilon$, and
the decoder has Lipschitz constant $K_{D_\psi}$ and error $\delta_{\mathrm{dec}}$. Under
the bounded-drift implementation of Assumption~\ref{ass:main_regular}, with
scaffold-restoration timescale $\epsilon_{\mathrm{s}}>0$, latent-velocity cap $V_{\max}$,
and latent integration step $\Delta t$ as defined for the online controller in
Section~\ref{sec:online} (full statement and proof: Theorem~\ref{thm:end_to_end} in
Appendix~\ref{sec:stability}), the asymptotic executed action error satisfies
\begin{equation}
  \limsup_{t\to\infty}\norm{u_t-u_t^\star}
  \le
  K_{D_\psi}\!\left(
    \sqrt{\frac{\epsilon_{\mathrm{s}}}{2}}V_{\max}\Delta t + \frac{\varepsilon}{m}
  \right)
  +\delta_{\mathrm{dec}} .
  \label{eq:main_action_error}
\end{equation}
The same bound holds pointwise for any receding-horizon update whose local refinement has
reached the $\varepsilon/m$ robustness tube. The three terms correspond to geometric
integration drift, score approximation error, and decoder error; none grows with the
rollout horizon.
\end{corollary}

\begin{proof}
Let $z_t^\star$ denote the ideal in-scaffold latent action segment and let $\bar z_t$ be
the in-scaffold point reached by the local refinement. The perturbed-gradient result in
Theorem~\ref{thm:robustness} gives
$\limsup_{t\to\infty}d_G(\bar z_t,z_t^\star)\le \varepsilon/m$. The
bounded-drift/retraction guard in Assumption~\ref{ass:main_regular} keeps the accepted
latent point $z_t$ within
$\sqrt{\epsilon_{\mathrm{s}}/2}V_{\max}\Delta t$ of $\bar z_t$ in the local metric. Thus, by the
triangle inequality and the same limsup bound,
$\limsup_{t\to\infty}d_G(z_t,z_t^\star)\le
\sqrt{\epsilon_{\mathrm{s}}/2}V_{\max}\Delta t+\varepsilon/m$. Applying the
$K_{D_\psi}$-Lipschitz decoder to the first decoded action and adding the offline decoder
error $\delta_{\mathrm{dec}}$ gives \eqref{eq:main_action_error}; the final projection onto
the convex known-feasible set is non-expansive and fixes the feasible ideal action, so it
does not enlarge the bound. For any particular
receding-horizon update already satisfying
$d_G(\bar z_t,z_t^\star)\le \varepsilon/m$, the identical triangle-inequality argument
gives the pointwise version. The estimate does not sum over the rollout length.
\end{proof}

\begin{theorem}[Approximate optimality]
\label{thm:main_asymptotic_optimality}
Fix a selected context $c=(o,g,\Cknown,H^\star)$ and a compact geodesically convex
data-supported region $U_c\subset\Mhat$. Let $\Phi_{\mathrm{opt}}(\cdot,c)$ denote the
oracle constrained horizon objective on $U_c$, and let
$z^\star=\argmin_{z\in U_c}\Phi_{\mathrm{opt}}(z,c)$ be its unique minimizer. Suppose
the learned GDC potential approximates the oracle objective up to the context-dependent
normalizing constant $a(c)$:
\begin{equation}
  \sup_{z\in U_c}
  \left|
    \Phi_{\mathrm{GDC}}(z,c)-\Phi_{\mathrm{opt}}(z,c)-a(c)
  \right|
  \le \alpha_{\mathrm{val}},
  \label{eq:main_value_consistency}
\end{equation}
and suppose the online refinement returns $\hat z\in U_c$ satisfying
\begin{equation}
  \Phi_{\mathrm{GDC}}(\hat z,c)
  \le
  \inf_{z\in U_c}\Phi_{\mathrm{GDC}}(z,c)+\eta_{\mathrm{opt}} .
  \label{eq:main_approx_argmin}
\end{equation}
If $\Phi_{\mathrm{opt}}$ is $m$-strongly geodesically convex on $U_c$, then
\begin{equation}
  \Phi_{\mathrm{opt}}(\hat z,c)-\Phi_{\mathrm{opt}}(z^\star,c)
  \le
  2\alpha_{\mathrm{val}}+\eta_{\mathrm{opt}},
  \qquad
  d_G(\hat z,z^\star)
  \le
  \sqrt{\frac{2(2\alpha_{\mathrm{val}}+\eta_{\mathrm{opt}})}{m}} .
  \label{eq:main_asymptotic_optimality}
\end{equation}
Consequently, with a $K_{D_\psi}$-Lipschitz decoder, decoder error
$\delta_{\mathrm{dec}}$, and final projection perturbation $\delta_{\mathrm{proj}}$,
the executed first action obeys
\begin{equation}
  \|u_t-u_t^\star\|
  \le
  K_{D_\psi}
  \sqrt{\frac{2(2\alpha_{\mathrm{val}}+\eta_{\mathrm{opt}})}{m}}
  +\delta_{\mathrm{dec}}+\delta_{\mathrm{proj}} .
  \label{eq:main_optimal_action_error}
\end{equation}
Thus, along any sequence of datasets, models, and refinement budgets for which
$\alpha_{\mathrm{val}},\eta_{\mathrm{opt}},\delta_{\mathrm{dec}}$, and
$\delta_{\mathrm{proj}}$ vanish, GDC converges to the constrained optimal first action
on the selected data-supported region.
\end{theorem}

\begin{proof}
Let $\widehat\Phi=\Phi_{\mathrm{GDC}}(\cdot,c)$ and
$\Phi^\star=\Phi_{\mathrm{opt}}(\cdot,c)$. By \eqref{eq:main_value_consistency},
$\Phi^\star(\hat z)\le \widehat\Phi(\hat z)-a(c)+\alpha_{\mathrm{val}}$. By
\eqref{eq:main_approx_argmin}, this is at most
$\widehat\Phi(z^\star)-a(c)+\alpha_{\mathrm{val}}+\eta_{\mathrm{opt}}$. Applying
\eqref{eq:main_value_consistency} again at $z^\star$ gives
$\Phi^\star(\hat z)-\Phi^\star(z^\star)\le
2\alpha_{\mathrm{val}}+\eta_{\mathrm{opt}}$. Strong geodesic convexity gives
$\Phi^\star(\hat z)-\Phi^\star(z^\star)\ge
(m/2)d_G(\hat z,z^\star)^2$, yielding the distance bound. The action bound follows from
the decoder Lipschitz property, then adding the decoder reconstruction and final
projection perturbation errors. Letting the four error terms go to zero proves the
asymptotic statement.
\end{proof}

\begin{corollary}[Coverage-controlled rate]
\label{cor:main_optimality_rate}
In the setting of Theorem~\ref{thm:main_asymptotic_optimality}, let
$\mathcal Z_n(c)\subset U_c$ be the finite latent training anchors for context $c$, with
fill distance
$h_n=\sup_{z\in U_c}\inf_{\zeta\in\mathcal Z_n(c)}d_G(z,\zeta)$. Suppose
$\Phi_{\mathrm{GDC}}(\cdot,c)$ and $\Phi_{\mathrm{opt}}(\cdot,c)$ are Lipschitz on
$U_c$ with constants $L_n$ and $L_{\mathrm{opt}}$, and suppose their
anchor error, up to a context normalizer $a(c)$, is at most $\varepsilon_n$:
\[
  \sup_{\zeta\in\mathcal Z_n(c)}
  |\Phi_{\mathrm{GDC}}(\zeta,c)-\Phi_{\mathrm{opt}}(\zeta,c)-a(c)|
  \le \varepsilon_n .
\]
Then Theorem~\ref{thm:main_asymptotic_optimality} holds with
\[
  \alpha_{\mathrm{val}}
  \le
  \varepsilon_n+(L_n+L_{\mathrm{opt}})h_n .
\]
Consequently, if the online refinement error after $K$ refinement steps is $\eta_K$,
the oracle gap is
\[
  \Phi_{\mathrm{opt}}(\hat z,c)-\Phi_{\mathrm{opt}}(z^\star,c)
  =
  \mathcal O(\varepsilon_n+h_n+\eta_K),
\]
and the executed first-action error is
\[
  \|u_t-u_t^\star\|
  =
  \mathcal O\!\left(\sqrt{\varepsilon_n+h_n+\eta_K}
  +\delta_{\mathrm{dec}}+\delta_{\mathrm{proj}}\right).
\]
For example, quasi-uniform coverage on a $d$-dimensional region gives
$h_n=\mathcal O(n^{-1/d})$, while iid coverage with density bounded below gives
$h_n=\mathcal O_p((\log n/n)^{1/d})$; the covering argument is given in
Corollary~\ref{cor:coverage_rate}.
\end{corollary}

\begin{proof}
For any $z\in U_c$, choose $\zeta\in\mathcal Z_n(c)$ with
$d_G(z,\zeta)\le h_n$. The triangle inequality and Lipschitz bounds give
\[
  |\Phi_{\mathrm{GDC}}(z,c)-\Phi_{\mathrm{opt}}(z,c)-a(c)|
  \le
  \varepsilon_n+(L_n+L_{\mathrm{opt}})h_n .
\]
Substituting this bound for $\alpha_{\mathrm{val}}$ in
Theorem~\ref{thm:main_asymptotic_optimality} gives the value-gap and action-error rates.
\end{proof}

\begin{remark}[Physics--data interpolation]
\label{prop:main_interpolation}
For $\lambda\in[0,1]$, the online potential
\begin{equation}
  \Phi_\lambda(z)
  =
  \lambda\Phi_{\mathrm{known}}(z)
  +(1-\lambda)\Phi_{\mathrm{data}}(z,o,g,\Cknown,H^\star)
  +\Phi_{\mathrm{task}}(z)
  \label{eq:main_interpolation}
\end{equation}
is affine in $\lambda$, so it interpolates continuously between known-physics control
($\lambda=1$, leaving $\Phi_{\mathrm{known}}+\Phi_{\mathrm{task}}$) and data-driven
progress completion ($\lambda=0$, leaving $\Phi_{\mathrm{data}}+\Phi_{\mathrm{task}}$).
The support restriction, anchors, filters, and final projection are those of the known
scaffold at every $\lambda$, so intermediate values simply blend the two vector-field
contributions. When data support is sparse or the horizon selector abstains, setting
$\lambda\to1$ removes $\Phi_{\mathrm{data}}$ from the descent law and recovers the
known-feasible controller.
\end{remark}

These statements are intentionally conditional. GDC does not claim global optimality or
implicit feasibility outside the data-supported scaffold. Approximate optimality requires
the stronger value-consistency condition in Theorem~\ref{thm:main_asymptotic_optimality}.
Without that condition, the guarantee is narrower and operational: known physics enforces
what can be executed, while progress-tilted feasible data supplies a local direction
where the horizon selector certifies positive expected progress.

\section{Experimental Validation}
\label{sec:applications}

We instantiate GDC in two domains with different meanings of partial structural
knowledge. In SUMO, the known structure consists of explicit traffic and actuator
constraints, while route-directed behavior is learned from weak driving traces. In the
bilevel problem, the known structure consists of the nested problem form, analytic
objectives and gradients, feasibility projection, and local relaxation updates, while
useful joint motion across the levels is learned from solver traces. Each domain is
evaluated on its own numerical benchmark: the multilevel problems test whether the
learned progress field supplies fast cross-level directions beyond a known local
scaffold, and the SUMO routes test whether the same principle improves closed-loop route
progress while explicit traffic constraints remain enforced online.

\subsection{Application I: Route-progress control in SUMO}
\label{sec:app_sumo}

\subsubsection{Task and known scaffold}

SUMO gives a concrete partial-physics instance of GDC
\cite{lopez2018sumo,krajzewicz2012sumo}. The ego vehicle must complete a sampled route
through a traffic network, but the controller is not treated as an end-to-end perception
system. SUMO provides localization, lane, leader-vehicle, and traffic-light signals;
explicit rules check speed, acceleration, headway, lane-change, and collision
constraints. The missing component is the local direction of route progress: among
actions that satisfy these known tests, the controller must decide which one advances the
route without relying on a full online optimal-control solve.

At each control step the controller observes
\begin{equation}
 o_t=(p_t^x,p_t^y,v_t,\sin\psi_t,\cos\psi_t,\ell_t,s_t,L_t,v_t^{\max},
 \iota_t,N_t^{\mathrm{route}},d_t,g_t^{\mathrm{lead}},\Delta v_t^{\mathrm{lead}},c_t^{\mathrm{tls}},
 c_t^{\mathrm{col}},c_t^{\mathrm{arr}})\in\R^{17},
\label{eq:sumo_observation}
\end{equation}
where $d_t$ is the remaining route distance. The action
$u_t=(a_t,q_t)$ contains a longitudinal acceleration and a lane-change command
$q_t\in\{-1,0,1\}$. After executing one action in SUMO, the controller receives a new
observation and replans.

The known-feasible filter checks
\begin{align}
 -a_{\mathrm{dec}}^{\max}&\leq a_t\leq a_{\mathrm{acc}}^{\max},
 &0&\leq v_t+\Delta t_{\mathrm{env}}\,a_t\leq v_t^{\max}+\varepsilon_v, \\
 q_t&\in\{-1,0,1\},
 &|q_t-q_{t-1}|&\leq\Delta q_{\max}, \\
 g_t^{\mathrm{lead}}&\geq g_{\min}+T_hv_t,
 &c_t^{\mathrm{col}}&=0.
 \label{eq:sumo_constraints}
\end{align}
We use $a_{\mathrm{acc}}^{\max}=2.0\,\mathrm{m/s^2}$,
$a_{\mathrm{dec}}^{\max}=4.5\,\mathrm{m/s^2}$, $g_{\min}=2.5\,\mathrm{m}$,
$T_h=0.8\,\mathrm{s}$, and environment control period
$\Delta t_{\mathrm{env}}=0.5\,\mathrm{s}$ (distinct from the latent integration step
$\Delta t$ of the online navigator). These constraints determine
what may be executed, but they do not select among feasible route choices. A conservative
fallback segment is constructed from this same scaffold by tracking a target speed,
braking for near leaders or restrictive signals, and avoiding discretionary lane changes.
The fallback is a feasible anchor and a last-resort action source; it is not the learned
policy.

\subsubsection{Progress data and online decision rule}

Offline rollouts come from a deliberately weak and diverse behavior policy with randomized
acceleration, braking, coasting, exploration, and lane changes. It observes traffic-light
state and slows at red or yellow signals, but is neither expert nor globally optimal.
The data therefore include progress, waiting, braking, and imperfect local maneuvers.
GDC does not treat these traces as expert labels. It uses them to learn which feasible
segments tend to reduce remaining route distance.

Each rollout is divided into action segments
\begin{equation}
 \tau_t=(u_t,u_{t+1},\ldots,u_{t+H-1}).
\end{equation}
Each segment receives the signed certificate
\begin{equation}
 \Delta\rho_{\mathrm{SUMO}}(\tau_t,g)=d_t-d_{t+H}
 -\kappa_{\mathrm{col}}N_{\mathrm{col}}(\tau_t)
 +\kappa_{\mathrm{arr}}\mathbb{I}\{\text{arrival in }\tau_t\}.
 \label{eq:sumo_certificate}
\end{equation}
The target distribution is the progress-tilted empirical distribution
\begin{equation}
 q_{\beta}^{\mathrm{SUMO}}(\tau\mid o,g,\Cknown)
 \propto q_{\mathcal D}^{\mathrm{SUMO}}(\tau\mid o,g,\Cknown)
 \exp\!\left(\beta\,\operatorname{clip}
 (\Delta\rho_{\mathrm{SUMO}}(\tau,g),-c,c)\right).
 \label{eq:sumo_tilt}
\end{equation}

At test time, GDC samples or initializes candidate segments near the fallback and near
data-supported local motions, refines them with the learned score, decodes
$\tau_t(z)=D_\psi(z,o_t,g,\Cknown,H_t^\star)$, filters them with
\eqref{eq:sumo_constraints}, and ranks the remaining candidates by
\begin{equation}
 J_t(z)=w_E\Phi_{\mathrm{data}}(z,o_t,g,\Cknown,H_t^\star)
 -w_P\widehat{\Delta\rho}_{\mathrm{SUMO}}(\tau_t(z),g)
 +w_U\|\tau_t(z)\|^2+w_CV_{\Cknown}(\tau_t(z))
 +w_T\|z-z_t^{\mathrm{fb}}\|^2.
 \label{eq:sumo_online_objective}
\end{equation}
Here $V_{\Cknown}$ is the aggregate violation returned by the explicit
known-constraint checker, and $z_t^{\mathrm{fb}}$ is the latent representation of the
fallback segment.
Only the first action of the selected segment is executed. The next observation starts a
new planning step. Thus the learned distribution is not executed open loop; it supplies a
progress energy for choosing among known-feasible short motions. If no learned candidate
improves sufficiently over the fallback, the controller executes the first fallback
action.

\subsubsection{Numerical results}
\label{sec:sim_sumo}

The SUMO evaluation tests the same role separation in a closed-loop driving task:
explicit traffic constraints decide what can be executed, while the learned progress
energy chooses among feasible local motions. We train from $100{,}000$ diverse weak
rollouts and retain at most five million selected segments per horizon. Online
evaluation uses $100$ closed-loop episodes. At every control step the controller scores
candidate segments, filters them with the known scaffold, executes only the first action,
and replans after the next SUMO observation.

\paragraph{Benchmarks.}
We use three increasingly difficult route-progress benchmarks, all based on perturbed
$20\times20$ two-lane grids with $120\,\mathrm{m}$ nominal edge length and
$13.89\,\mathrm{m/s}$ nominal speed. The priority benchmark is the clean route-completion
case. Sparse TLS adds traffic-light delay, where stopping may be the correct local
action. Complex TLS is the harder-map benchmark: it keeps the graph connected but adds
$18$ sparse traffic lights, $64$ very-low-speed detour edges at $2.5\,\mathrm{m/s}$,
and $76$ bottleneck edges at $7.5\,\mathrm{m/s}$ on a $14\,\mathrm{m}$ perturbed grid.
This creates routes where local progress, stopping, and longer-range route intent can
conflict.

\begin{table}[!ht]
\centering
\small
\setlength{\tabcolsep}{5pt}
\begin{tabular}{lrrrrl}
\toprule
\textbf{Benchmark} & \textbf{TLS} & \textbf{Slow detour} & \textbf{Bottleneck} &
\textbf{Removed edges} & \textbf{Controllers} \\
\midrule
Priority grid & $0$ & $0$ & $0$ & $0$ & fixed $H=12$ \\
Sparse TLS & sparse & $0$ & $0$ & $0$ & fixed $H=12$, multi-$H$ \\
Complex TLS & $18$ & $64$ at $2.5\,\mathrm{m/s}$ & $76$ at $7.5\,\mathrm{m/s}$ &
$0$ & fixed $H=12$, multi-$H$ \\
\bottomrule
\end{tabular}
\caption{SUMO route-progress benchmark suite. All maps are perturbed $20\times20$
two-lane grids. Multi-$H$ denotes joint online selection over $H\in\{6,12,24\}$; fixed-H
restricts candidate selection to $H=12$.}
\label{tab:sumo_benchmarks}
\end{table}

\paragraph{Controllers and metrics.}
The fixed-horizon baseline and the multi-horizon controller use the same observation
features, weak rollout source, known-feasible filter, fallback construction, and
receding-horizon first-action execution. The difference is only the online temporal
support: fixed-H scores $H=12$ segments, while multi-H scores candidates from
$H\in\{6,12,24\}$ in a single pool and executes the first action of the best safe
segment. We report arrival rate, normalized route progress, collisions, mean and maximum
known-constraint violation, non-fallback fraction, mean speed, and episode length.

\begin{table}[!htbp]
\centering
\small
\setlength{\tabcolsep}{3.5pt}
\resizebox{\linewidth}{!}{%
\begin{tabular}{llrrrrrrrr}
\toprule
\textbf{Benchmark} & \textbf{Controller} & \textbf{Arr.} & \textbf{Prog.} &
\textbf{Coll.} & \textbf{Mean viol.} & \textbf{Max viol.} &
\textbf{Non-fb.} & \textbf{Speed} & \textbf{Steps} \\
\midrule
Priority grid & fixed $H=12$ & $1.000$ & $1.000$ & $0.000$ &
$1.22{\times}10^{-6}$ & $2.80{\times}10^{-4}$ & $0.999$ & $11.24$ & $202.6$ \\
Sparse TLS & fixed $H=12$ & $1.000$ & $1.000$ & $0.000$ &
$1.48{\times}10^{-7}$ & $5.01{\times}10^{-5}$ & $1.000$ & $10.97$ & $230.3$ \\
Sparse TLS & multi-$H$ & $1.000$ & $1.000$ & $0.000$ &
$0.00$ & $0.00$ & $0.997$ & $10.58$ & $233.9$ \\
Complex TLS & fixed $H=12$ & $1.000$ & $1.000$ & $0.000$ &
$1.59{\times}10^{-5}$ & $5.88{\times}10^{-3}$ & $0.999$ & $9.25$ & $317.3$ \\
Complex TLS & multi-$H$ & $1.000$ & $1.000$ & $0.000$ &
$0.00$ & $0.00$ & $0.991$ & $8.96$ & $322.9$ \\
\bottomrule
\end{tabular}
}
\caption{SUMO closed-loop results over $100$ episodes. Arr. is arrival rate, Prog. is
normalized route progress, Coll. is mean collision count, Non-fb. is the fraction of
steps where a learned candidate rather than the fallback supplies the executed action,
Speed is mean speed in $\mathrm{m/s}$, and Steps is mean episode length.}
\label{tab:sumo_results}
\end{table}

\paragraph{Results.}
All reported SUMO controllers reach every route goal with zero collisions
(Table~\ref{tab:sumo_results}). The priority result is the clean route-completion case:
GDC achieves unit progress with only numerical-tolerance known-constraint violation.
Sparse TLS reduces speed and increases episode length because stopping at signals is part
of feasible progress, but both fixed-H and multi-H still complete all routes. The harder
complex map is where the horizon comparison becomes meaningful. Fixed $H=12$ is slightly
faster, but its maximum known-constraint violation rises to $5.88\times10^{-3}$.
Multi-H removes this residual violation entirely, with a $0.29\,\mathrm{m/s}$ speed
reduction and a $5.6$-step increase in mean episode length. Thus the multi-horizon
benefit is not higher speed; it is safer candidate selection when the map creates
conflicting short- and longer-horizon progress cues.

\begin{figure}[!htbp]
\centering
\includegraphics[width=0.92\linewidth]{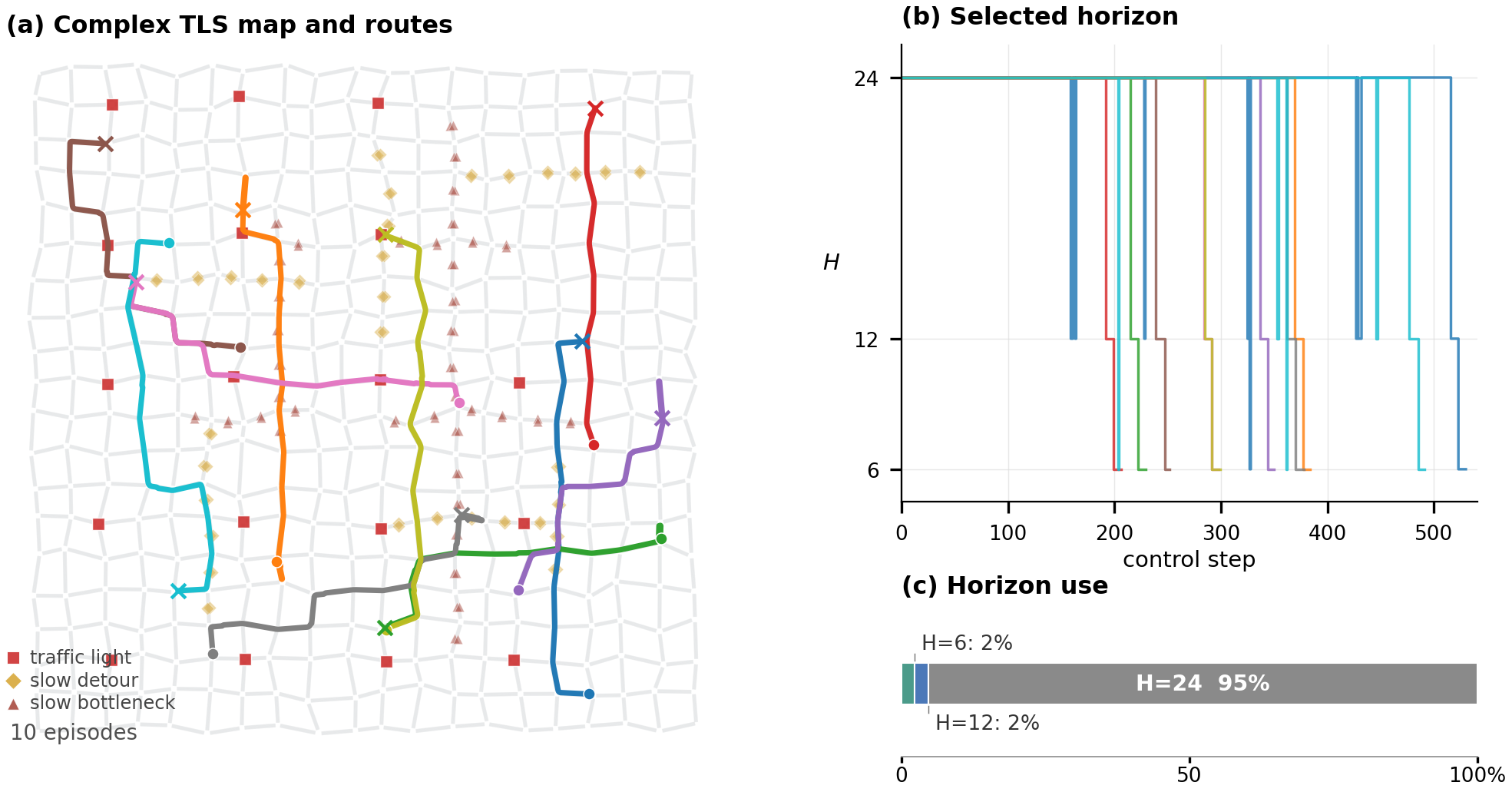}
\caption{Complex TLS multi-horizon evaluation from the fast-equivalent online
implementation. The figure shows ten recorded ego paths on the harder map, selected
horizons over control steps for the same episodes, the horizon-use distribution over all
$100$ episodes, and the dominance of the long horizon on a map with sparse traffic
lights, slow bottleneck corridors, and slow detour edges.}
\label{fig:sumo_multi_h_paths}
\end{figure}

Figure~\ref{fig:sumo_multi_h_paths} gives the qualitative explanation for the complex-map
result. Multi-H does not vote among three independent policies. It scores all safe
candidate segments together and executes the first action of the best segment. The
selected horizon therefore acts as an online diagnostic of temporal scale: longer
horizons dominate during ordinary progress, while shorter horizons appear near local
traffic or constraint-sensitive events. The fast-equivalent implementation gives the same
aggregate statistics as the original multi-horizon controller, showing that this
candidate-scoring step can be batched without changing the closed-loop behavior.

\paragraph{Comparison with rule, imitation, and RL baselines.}
Table~\ref{tab:sumo_rl_benchmarks} compares GDC with a rule safety controller, the weak
behavior source used for data collection, behavior cloning on the offline rollouts, and
PPO/SAC baselines. The comparison uses the same sparse-TLS and complex-TLS route suites.
GDC is the only method that reaches $100\%$ arrival on both maps with zero collisions and
zero mean known-constraint violation in the multi-horizon setting. On the harder complex
TLS benchmark, GDC multi-H improves arrival from $91\%$ for the rule controller,
$51\%$ for behavior cloning, $81\%$ for PPO, and $79\%$ for SAC to $100\%$, while also
reducing the mean episode length relative to all non-GDC baselines.

\begin{table}[!htbp]
\centering
\footnotesize
\setlength{\tabcolsep}{3.0pt}
\renewcommand{\arraystretch}{0.98}
\begin{tabular}{llrrrrrr}
\toprule
\textbf{Benchmark} & \textbf{Method} & \textbf{Prog.} & \textbf{Arr.} &
\textbf{Coll.} & \textbf{Mean viol.} & \textbf{Speed} & \textbf{Steps} \\
\midrule
Sparse TLS & GDC fixed-H12 & $1.000$ & $1.000$ & $0.00$ & $1.48{\times}10^{-7}$ & $10.97$ & $230.3$ \\
Sparse TLS & GDC multi-H & $\mathbf{1.000}$ & $\mathbf{1.000}$ & $0.00$ & $\mathbf{0.00}$ & $10.58$ & $\mathbf{233.9}$ \\
Sparse TLS & Rule safety & $1.000$ & $0.990$ & $0.00$ & $1.33{\times}10^{-1}$ & $11.05$ & $330.6$ \\
Sparse TLS & Weak behavior & $0.962$ & $0.840$ & $0.00$ & $1.44{\times}10^{-2}$ & $7.86$ & $449.8$ \\
Sparse TLS & Behavior cloning & $0.931$ & $0.760$ & $0.00$ & $3.55{\times}10^{-3}$ & $6.26$ & $537.1$ \\
Sparse TLS & PPO & $0.997$ & $0.970$ & $0.00$ & $1.27{\times}10^{-2}$ & $7.87$ & $451.4$ \\
Sparse TLS & SAC & $1.000$ & $1.000$ & $0.00$ & $1.82{\times}10^{-2}$ & $10.72$ & $346.3$ \\
\midrule
Complex TLS & GDC fixed-H12 & $1.000$ & $1.000$ & $0.00$ & $1.59{\times}10^{-5}$ & $\mathbf{9.25}$ & $\mathbf{317.3}$ \\
Complex TLS & GDC multi-H & $\mathbf{1.000}$ & $\mathbf{1.000}$ & $0.00$ & $\mathbf{0.00}$ & $8.96$ & $322.9$ \\
Complex TLS & Rule safety & $0.989$ & $0.910$ & $0.00$ & $1.48{\times}10^{-1}$ & $9.00$ & $441.2$ \\
Complex TLS & Weak behavior & $0.841$ & $0.550$ & $0.00$ & $1.44{\times}10^{-2}$ & $5.65$ & $546.9$ \\
Complex TLS & Behavior cloning & $0.770$ & $0.510$ & $0.00$ & $3.90{\times}10^{-3}$ & $4.74$ & $587.7$ \\
Complex TLS & PPO & $0.965$ & $0.810$ & $0.00$ & $1.29{\times}10^{-2}$ & $8.36$ & $461.7$ \\
Complex TLS & SAC & $0.960$ & $0.790$ & $0.00$ & $2.65{\times}10^{-2}$ & $8.51$ & $436.0$ \\
\bottomrule
\end{tabular}
\caption{SUMO controller benchmark comparison. GDC and non-SAC rows are evaluated over
$100$ episodes; the available sparse-TLS SAC run contains $90$ episodes. Prog. is
normalized route progress, Arr. is arrival rate, Coll. is mean collision count, Speed is
mean speed in $\mathrm{m/s}$, and Steps is mean episode length.}
\label{tab:sumo_rl_benchmarks}
\end{table}

\subsection{Application II: Parameterized bilevel optimization}
\label{sec:app_bilevel}

\subsubsection{Problem family and known structural knowledge}

For each parameter vector $\vartheta$, consider a nonconvex parameterized bilevel
energy over $y,x\in\R^d$ \cite{colson2007overview,dempe2020bilevel}. The lower-level
merit is a quartic-well objective
\begin{equation}
  f_{\mathrm{lower},\vartheta}(x;y)
  =
  \frac{1}{4d}\norm{x\odot x-\ell_\vartheta(y)}_2^2
  +\frac{\mu_x}{2d}\norm{x}_2^2
  +\frac{\lambda_x}{2(d-1)}\sum_{i=1}^{d-1}(x_{i+1}-x_i)^2,
  \qquad
  \ell_\vartheta(y)=a_\vartheta+M_\vartheta y ,
  \label{eq:bilevel_lower_nonconvex}
\end{equation}
where $M_\vartheta=M_0+\sum_{r=1}^R s_r u_r v_r^\top$ is a low-rank perturbation
of a fixed banded coupling. The upper objective is also nonconvex:
\begin{align}
  F_{\mathrm{upper},\vartheta}(y,x)
  &=
  \frac{\alpha_y}{d}\norm{y-g_y}_2^2
  +\frac{\alpha_x}{d}\norm{x-g_x}_2^2
  +\frac{\kappa_y}{d}\norm{y\odot y-r_y}_2^2
  +\frac{\kappa_x}{d}\norm{x\odot x-r_x}_2^2 \notag\\
  &\quad
  +\frac{\eta_{xy}}{d}\norm{x-y}_2^2
  +\frac{1}{d}x^\top C y
  +\frac{\lambda_y}{d-1}\sum_{i=1}^{d-1}(y_{i+1}-y_i)^2 .
  \label{eq:bilevel_upper_nonconvex}
\end{align}
The parameter vector $\vartheta$ contains the goals $(g_y,g_x)$, lower shifts
$a_\vartheta$, well radii $(r_y,r_x)$, low-rank scales $(s_r)_{r=1}^R$, and the ridge
parameter $\mu_x$. The quartic terms make both the lower and upper landscapes
nonconvex, with multiple attractive basins. Each instance defines the bilevel problem
\begin{equation}
  \min_{y\in[-2,2]^d}
  F_{\mathrm{upper},\vartheta}\bigl(y,x_\vartheta^\star(y)\bigr),
  \qquad
  x_\vartheta^\star(y)
  \in
  \argmin_{x\in[-2,2]^d}
  f_{\mathrm{lower},\vartheta}(x;y).
  \label{eq:bilevel_problem}
\end{equation}
The known structural knowledge consists of the level decomposition, analytic objectives,
gradients, box projection, step-size projection, and local lower-level update; it does
not include a globally certified lower response or upper optimum.

The action is a joint increment $u=(\Delta y,\Delta x)$. The explicit scaffold enforces
\begin{equation}
  y^+,x^+\in[-2,2]^d,\qquad
  \|\Delta y\|_2\leq \delta_y,\qquad
  \|\Delta x\|_2\leq \delta_x .
  \label{eq:bilevel_known_constraints}
\end{equation}

The trilevel scaling experiment uses the same nonconvex construction with state
$\omega=(z,y,x)$. The middle and lower objectives are quartic wells around parameterized targets
$m_\vartheta(z)$ and $\ell_\vartheta(y,z)$:
\begin{align}
  f_{\mathrm{middle},\vartheta}(y;z)
  &=
  \frac{1}{4d}\norm{y\odot y-m_\vartheta(z)}_2^2
  +\frac{\mu_y}{2d}\norm{y}_2^2
  +\frac{\lambda_y}{2(d-1)}\sum_{i=1}^{d-1}(y_{i+1}-y_i)^2, \notag\\
  f_{\mathrm{lower},\vartheta}(x;y,z)
  &=
  \frac{1}{4d}\norm{x\odot x-\ell_\vartheta(y,z)}_2^2
  +\frac{\mu_x}{2d}\norm{x}_2^2
  +\frac{\lambda_x}{2(d-1)}\sum_{i=1}^{d-1}(x_{i+1}-x_i)^2.
  \label{eq:trilevel_middle_lower_nonconvex}
\end{align}
The top objective adds quadratic goal terms, quartic well terms for $z,y,x$, pairwise
cross-level penalties, and indefinite bilinear couplings. These objectives define the
trilevel problem
\begin{equation}
\begin{aligned}
  \min_{z\in[-2,2]^d}\quad
  &F_{\mathrm{top},\vartheta}
  \bigl(z,y_\vartheta^\star(z),
  x_\vartheta^\star(z,y_\vartheta^\star(z))\bigr),\\
  y_\vartheta^\star(z)
  &\in
  \argmin_{y\in[-2,2]^d}
  f_{\mathrm{middle},\vartheta}(y;z),\\
  x_\vartheta^\star(z,y)
  &\in
  \argmin_{x\in[-2,2]^d}
  f_{\mathrm{lower},\vartheta}(x;y,z).
\end{aligned}
\label{eq:trilevel_problem}
\end{equation}
The top level depends on $(z,y,x)$, the middle level depends on $(y;z)$, and the lower
level depends on $(x;y,z)$. Box and step projections are applied separately to $z$, $y$,
and $x$.

\subsubsection{Progress model and online solver}

Weak solver traces are divided into segments of $H=8$ joint increments and conditioned on
$(\omega_t,u_{t-1},\vartheta)$. Segment labels use decrease in a fixed-weight scalar
progress score minus effort:
\begin{align}
  \Delta\rho_{\mathrm{bi}}(\tau_t,\vartheta)
  &=
  \Phi_{\vartheta}^{\mathrm{bi}}(\omega_t)
  -\Phi_{\vartheta}^{\mathrm{bi}}(\omega_{t+H})
  -\eta\sum_{k=0}^{H-1}\|u_{t+k}\|^2 .
  \label{eq:bilevel_certificate}
\end{align}
This scalar score is a training and selection certificate for candidate motion; it is not
a single-level relaxation used to define the bilevel or trilevel problem.
The training distribution is tilted by $\exp(\beta\Delta\rho_{\mathrm{bi}}(\tau,\vartheta))$. Online, GDC
generates joint segments, refines them using the learned score, projects their increments
into \,\eqref{eq:bilevel_known_constraints}, and compares them with a known-only local
update. Selection combines learned energy, predicted decrease in the scalar score, lower
merit, and a trust-region penalty. Only the first increment is applied; feasibility or
progress failure triggers the known-only fallback.

\subsubsection{Protocol and metrics}

Training uses $50{,}000$ parameterized traces, with goals, well locations, low-rank
couplings, offsets, and ridge parameters resampled per trace. We compare GDC with
known-only and nested local online solvers on held-out parameters. For evaluation only,
offline multistart search and SCIP branch-and-cut solves of the original nonconvex
bilevel and trilevel formulations provide reference values or incumbents; they are not
online runtime competitors. We report
\begin{equation}
  \operatorname{ngap}
  =
  \frac{\Phi_{\mathrm{GDC}}-\Phi_{\mathrm{ref}}}
       {\Phi_{\mathrm{init}}-\Phi_{\mathrm{ref}}}.
  \label{eq:bilevel_metrics}
\end{equation}
This experiment tests whether one learned progress field provides useful cross-level
directions on unseen nonconvex instances while the explicit scaffold maintains the known
box and step constraints.

\subsubsection{Numerical results}
\label{sec:sim_multilevel}

The multilevel evaluation treats GDC as an amortized online controller for parameterized
bilevel and trilevel nonconvex optimization families. The bilevel experiment uses the
family in Section~\ref{sec:app_bilevel}. The trilevel experiment extends the same test
from a state $\omega=(y,x)$ to $\omega=(z,y,x)$ with top, middle, and lower variables.
In both settings, each instance has its own parameter vector, and the controller receives
analytic objectives, gradients, local relaxation updates, and known feasibility
projection, but not a globally certified solution. The role of GDC is therefore
to learn useful joint motion across levels from weak solver traces, while the known
scaffold keeps the online updates inside the explicit box and step constraints.

For the tables and figures only, we evaluate the final iterates using fixed-weight
scalar criteria. These criteria are not the optimization problems solved by the online
methods, and they are not a relaxed single-level replacement for
\eqref{eq:bilevel_problem} or \eqref{eq:trilevel_problem}. For bilevel instances the
reported evaluation criterion is
\begin{equation}
  \Phi(y,x) =
  F_{\mathrm{upper}}(y,x)
  + \lambda_{\mathrm{lower}}\,f_{\mathrm{lower}}(x;y),
  \label{eq:bilevel_phi_eval}
\end{equation}
and for trilevel instances the reported evaluation criterion is
\begin{equation}
  \Phi(z,y,x) =
  F_{\mathrm{top}}(z,y,x)
  + \lambda_{\mathrm{middle}}\,f_{\mathrm{middle}}(y;z)
  + \lambda_{\mathrm{lower}}\,f_{\mathrm{lower}}(x;y,z).
  \label{eq:trilevel_phi_eval}
\end{equation}
Tables~\ref{tab:bilevel_phi_scaling} and~\ref{tab:trilevel_phi_scaling} report held-out
seed-$9001$ scaling results under this evaluation criterion. The normalized GDC gap is
$(\Phi_{\mathrm{GDC}}-\Phi_{\mathrm{ref}})/(\Phi_{\mathrm{init}}-\Phi_{\mathrm{ref}})$.
The reference value is an offline multistart search under the same evaluation criterion;
it is used only as an evaluation reference, not as a fair online runtime competitor.

\begin{table}[!htbp]
\centering
\small
\setlength{\tabcolsep}{4pt}
\begin{tabular}{rrrrrrr}
\toprule
$d$ & Ref. $\Phi$ & Known $\Phi$ & Nested $\Phi$ & GDC $\Phi$ & Norm. gap \% & GDC s \\
\midrule
$6$  & $0.122821439$ & $0.727128$ & $0.699750$ & $\mathbf{0.124122773}$ & $\mathbf{0.096}$ & $44.659$ \\
$8$  & $0.085741765$ & $0.620790$ & $0.601718$ & $\mathbf{0.088084970}$ & $\mathbf{0.194}$ & $57.622$ \\
$10$ & $0.118656182$ & $0.542130$ & $0.508923$ & $\mathbf{0.121497257}$ & $\mathbf{0.229}$ & $52.206$ \\
$12$ & $0.132315706$ & $1.005490$ & $0.914664$ & $\mathbf{0.134485621}$ & $\mathbf{0.179}$ & $83.117$ \\
\bottomrule
\end{tabular}
\caption{Bilevel evaluation-score scaling. GDC is the learned online controller; known
and nested are conventional online baselines; the reference column is an offline
evaluation reference under the same scalar criterion. Normalized gap is reported as
$100(\Phi_{\mathrm{GDC}}-\Phi_{\mathrm{ref}})/
(\Phi_{\mathrm{init}}-\Phi_{\mathrm{ref}})$; lower is better.}
\label{tab:bilevel_phi_scaling}
\end{table}

\begin{table}[!htbp]
\centering
\small
\setlength{\tabcolsep}{4pt}
\begin{tabular}{rrrrrrr}
\toprule
$d$ & Ref. $\Phi$ & Known $\Phi$ & Nested $\Phi$ & GDC $\Phi$ & Norm. gap \% & GDC s \\
\midrule
$4$  & $0.164844401$ & $0.572960$ & $0.658353$ & $\mathbf{0.168670984}$ & $\mathbf{0.345}$ & $79.898$ \\
$6$  & $0.178681018$ & $1.020626$ & $1.144793$ & $\mathbf{0.180060025}$ & $\mathbf{0.093}$ & $116.208$ \\
$8$  & $0.122788129$ & $0.951728$ & $1.072284$ & $\mathbf{0.131214128}$ & $\mathbf{0.640}$ & $138.865$ \\
$10$ & $0.178549407$ & $1.059038$ & $1.146033$ & $\mathbf{0.181326304}$ & $\mathbf{0.187}$ & $138.403$ \\
\bottomrule
\end{tabular}
\caption{Trilevel evaluation-score scaling. Each dimension reports the selected GDC
online evaluation with the lowest final evaluation score among available candidates.
Normalized gap is reported as a percentage; all four trilevel cases are below
$0.65\%$.}
\label{tab:trilevel_phi_scaling}
\end{table}

Table~\ref{tab:scip_phi_benchmark} reports the separate SCIP benchmark on the same
original nonconvex bilevel and trilevel formulations with a four-hour time limit per
case. SCIP is used here as an offline branch-and-cut benchmark. We therefore report its
incumbent evaluation score and optimality gap, and mark all time-limit cases as not globally
certified.

\begin{table}[!htbp]
\centering
\small
\setlength{\tabcolsep}{4pt}
\begin{tabular}{lllrrr}
\toprule
Problem & $d$ & Status & Gap & Incumbent $\Phi$ & Time (s) \\
\midrule
Bilevel  & $6$  & optimal   & $0.000000$ & $0.122821621$ & $22.6$ \\
Bilevel  & $8$  & timelimit & $0.000045$ & $0.085741916$ & $14401.7$ \\
Bilevel  & $10$ & timelimit & $0.000060$ & $0.118656277$ & $14402.9$ \\
Bilevel  & $12$ & timelimit & $0.000270$ & $0.132315706$ & $14400.0$ \\
Trilevel & $4$  & optimal   & $0.000000$ & $0.164844566$ & $327.8$ \\
Trilevel & $6$  & timelimit & $0.000065$ & $0.178681118$ & $14403.2$ \\
Trilevel & $8$  & timelimit & $0.001114$ & $0.122788129$ & $14400.0$ \\
Trilevel & $10$ & timelimit & $0.090036$ & $0.178549407$ & $14400.0$ \\
\bottomrule
\end{tabular}
\caption{SCIP branch-and-cut benchmark on the original nonconvex bilevel and trilevel
formulations with a four-hour time limit per case. Only rows with status ``optimal'' are
globally certified. Rows with status ``timelimit'' report the best incumbent found within
the limit.}
\label{tab:scip_phi_benchmark}
\end{table}

\begin{figure}[!htbp]
\centering
\begin{tikzpicture}
\begin{axis}[
  gdc axis,
  width=0.78\linewidth,
  height=0.36\linewidth,
  ybar=5pt,
  symbolic x coords={B6,B8,B10,B12,T4,T6,T8,T10},
  xtick=data,
  xticklabels={$B6$,$B8$,$B10$,$B12$,$T4$,$T6$,$T8$,$T10$},
  ymin=0,
  ymax=1.08,
  ylabel={wall time / four-hour budget},
  xlabel={problem instance},
  ytick={0,0.25,0.5,0.75,1},
  yticklabels={$0$,$25\%$,$50\%$,$75\%$,$4$h budget},
  enlarge x limits=0.08
]
\addplot+[fill=gdcblue!65, draw=gdcblue] coordinates
  {(B6,0.00310) (B8,0.00400) (B10,0.00363) (B12,0.00577)
   (T4,0.00555) (T6,0.00807) (T8,0.00964) (T10,0.00961)};
\addplot+[fill=gdcorange!32, draw=gdcorange!85] coordinates
  {(B6,0.00157) (B8,1.0) (B10,1.0) (B12,1.0)
   (T4,0.02276) (T6,1.0) (T8,1.0) (T10,1.0)};
\addplot+[black!55, dashed, sharp plot, mark=none, forget plot] coordinates {(B6,1.0) (T10,1.0)};
\node[font=\scriptsize\bfseries,text=gdcorange] at (axis cs:B8,1.04) {//};
\node[font=\scriptsize\bfseries,text=gdcorange] at (axis cs:B10,1.04) {//};
\node[font=\scriptsize\bfseries,text=gdcorange] at (axis cs:B12,1.04) {//};
\node[font=\scriptsize\bfseries,text=gdcorange] at (axis cs:T6,1.04) {//};
\node[font=\scriptsize\bfseries,text=gdcorange] at (axis cs:T8,1.04) {//};
\node[font=\scriptsize\bfseries,text=gdcorange] at (axis cs:T10,1.04) {//};
\node[font=\scriptsize,text=gdcblue,anchor=south,rotate=90] at (axis cs:B6,0.030) {$0.31\%$};
\node[font=\scriptsize,text=gdcblue,anchor=south,rotate=90] at (axis cs:B8,0.030) {$0.40\%$};
\node[font=\scriptsize,text=gdcblue,anchor=south,rotate=90] at (axis cs:B10,0.030) {$0.36\%$};
\node[font=\scriptsize,text=gdcblue,anchor=south,rotate=90] at (axis cs:B12,0.030) {$0.58\%$};
\node[font=\scriptsize,text=gdcblue,anchor=south,rotate=90] at (axis cs:T4,0.030) {$0.55\%$};
\node[font=\scriptsize,text=gdcblue,anchor=south,rotate=90] at (axis cs:T6,0.030) {$0.81\%$};
\node[font=\scriptsize,text=gdcblue,anchor=south,rotate=90] at (axis cs:T8,0.030) {$0.96\%$};
\node[font=\scriptsize,text=gdcblue,anchor=south,rotate=90] at (axis cs:T10,0.030) {$0.96\%$};
\draw[<->,line width=0.8pt,draw=black!58]
  (axis cs:B8,0.00435) -- node[midway,fill=white,inner sep=1.0pt,
  font=\scriptsize\bfseries,rotate=90] {$>249\times$} (axis cs:B8,0.98);
\draw[<->,line width=0.8pt,draw=black!58]
  (axis cs:B10,0.00395) -- node[midway,fill=white,inner sep=1.0pt,
  font=\scriptsize\bfseries,rotate=90] {$>275\times$} (axis cs:B10,0.98);
\draw[<->,line width=0.8pt,draw=black!58]
  (axis cs:B12,0.00625) -- node[midway,fill=white,inner sep=1.0pt,
  font=\scriptsize\bfseries,rotate=90] {$>173\times$} (axis cs:B12,0.98);
\node[font=\scriptsize\bfseries,text=black!58,anchor=west,rotate=90] at (axis cs:T4,0.050) {$4.1\times$};
\draw[<->,line width=0.8pt,draw=black!58]
  (axis cs:T6,0.00875) -- node[midway,fill=white,inner sep=1.0pt,
  font=\scriptsize\bfseries,rotate=90] {$>123\times$} (axis cs:T6,0.98);
\draw[<->,line width=0.8pt,draw=black!58]
  (axis cs:T8,0.0105) -- node[midway,fill=white,inner sep=1.0pt,
  font=\scriptsize\bfseries,rotate=90] {$>103\times$} (axis cs:T8,0.98);
\draw[<->,line width=0.8pt,draw=black!58]
  (axis cs:T10,0.0105) -- node[midway,fill=white,inner sep=1.0pt,
  font=\scriptsize\bfseries,rotate=90] {$>104\times$} (axis cs:T10,0.98);
\end{axis}
\end{tikzpicture}
\caption{Online GDC wall time versus SCIP branch-and-cut wall time on the same held-out
instances, normalized by the four-hour SCIP budget on a linear scale. Blue GDC bars are
intentionally near the baseline; blue labels report the fraction of the four-hour budget
used online. Orange SCIP bars at the budget line are time-limit cases, which provide
incumbents but are not globally certified. Vertical arrows report the SCIP/GDC wall-time
factor; capped SCIP cases are lower bounds.}
\label{fig:scip_time_cap}
\end{figure}
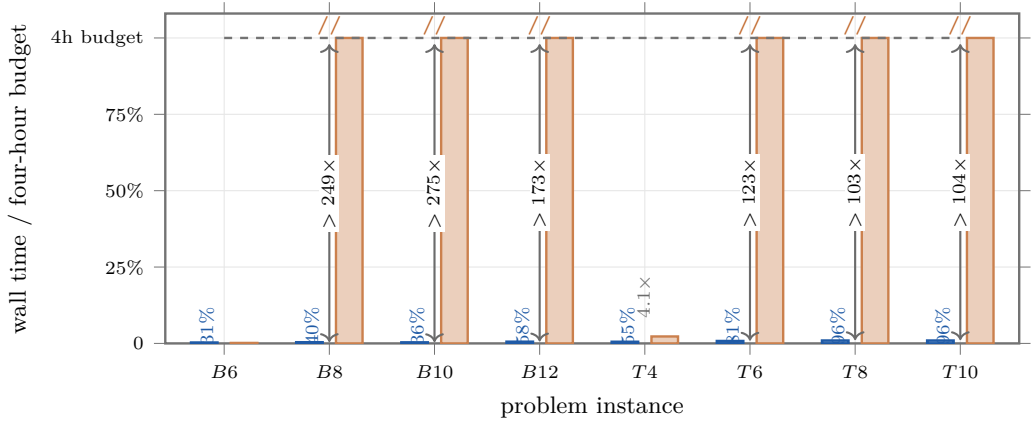

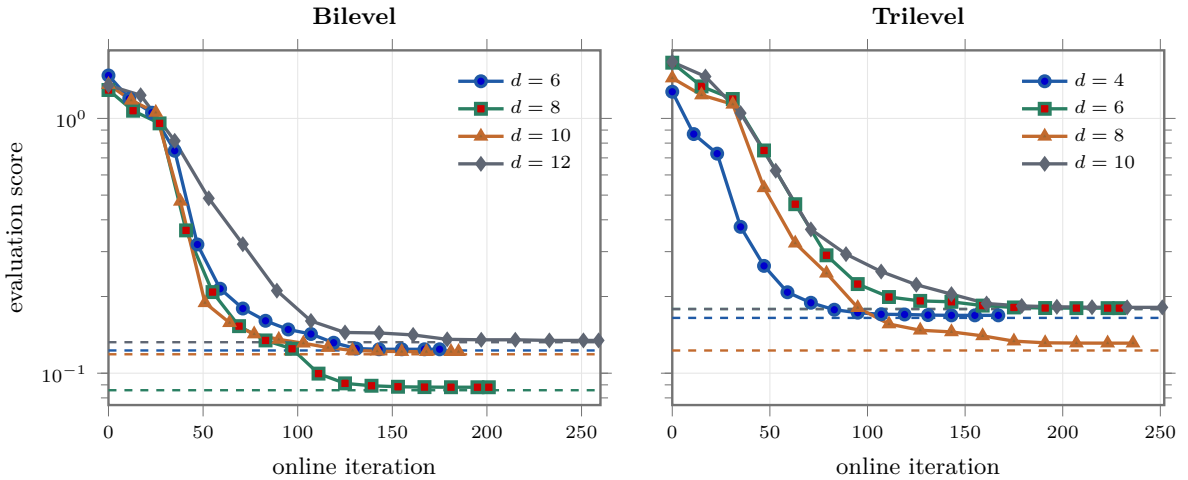
\begin{figure}[!htbp]
\centering
\begin{tikzpicture}
\begin{groupplot}[
  gdc axis,
  group style={group size=2 by 1, horizontal sep=0.95cm},
  width=0.49\linewidth,
  height=0.38\linewidth,
  ymode=log,
  xlabel={online iteration},
  ylabel={evaluation score},
  ymin=0.075,
  ymax=1.85,
  legend pos=north east
]
\nextgroupplot[title={Bilevel}, xmin=0, xmax=260]
\addplot+[color=gdcblue, dashed, mark=none, forget plot] coordinates {(0,0.122821621) (260,0.122821621)};
\addplot+[color=gdcblue, mark=*, line width=1.25pt] coordinates
{(0,1.474059) (11,1.195941) (23,1.052673) (35,0.746912) (47,0.320075)
 (59,0.214555) (71,0.179450) (83,0.160597) (95,0.148461) (107,0.142497)
 (119,0.131686) (131,0.124919) (143,0.124141) (155,0.124123) (167,0.124123)
 (175,0.124123)};
\addlegendentry{$d=6$}
\addplot+[color=gdcgreen, dashed, mark=none, forget plot] coordinates {(0,0.085741916) (260,0.085741916)};
\addplot+[color=gdcgreen, mark=square*, line width=1.25pt] coordinates
{(0,1.292742) (13,1.072039) (27,0.956160) (41,0.363539) (55,0.208154)
 (69,0.153005) (83,0.134475) (97,0.124683) (111,0.099632) (125,0.091220)
 (139,0.089278) (153,0.088465) (167,0.088194) (181,0.088104) (195,0.088085)
 (201,0.088085)};
\addlegendentry{$d=8$}
\addplot+[color=gdcorange, dashed, mark=none, forget plot] coordinates {(0,0.118656277) (260,0.118656277)};
\addplot+[color=gdcorange, mark=triangle*, line width=1.25pt] coordinates
{(0,1.357792) (12,1.178170) (25,1.055354) (38,0.472628) (51,0.188942)
 (64,0.157795) (77,0.142447) (90,0.135712) (103,0.131149) (116,0.126080)
 (129,0.122818) (142,0.121999) (155,0.121555) (168,0.121497) (181,0.121497)
 (185,0.121497)};
\addlegendentry{$d=10$}
\addplot+[color=gdcgray, dashed, mark=none, forget plot] coordinates {(0,0.132315706) (260,0.132315706)};
\addplot+[color=gdcgray, mark=diamond*, line width=1.25pt] coordinates
{(0,1.347819) (17,1.233074) (35,0.816394) (53,0.485806) (71,0.320795)
 (89,0.210838) (107,0.160358) (125,0.144479) (143,0.143873) (161,0.141327)
 (179,0.135801) (197,0.135247) (215,0.135247) (233,0.134486) (251,0.134486)
 (259,0.134486)};
\addlegendentry{$d=12$}

\nextgroupplot[title={Trilevel}, xmin=0, xmax=252, ylabel={}, yticklabels={}]
\addplot+[color=gdcblue, dashed, mark=none, forget plot] coordinates {(0,0.164844566) (252,0.164844566)};
\addplot+[color=gdcblue, mark=*, line width=1.25pt] coordinates
{(0,1.272554) (11,0.868102) (23,0.727338) (35,0.375338) (47,0.263750)
 (59,0.207837) (71,0.189031) (83,0.177492) (95,0.172806) (107,0.170315)
 (119,0.170012) (131,0.169127) (143,0.168682) (155,0.168671) (167,0.168671)};
\addlegendentry{$d=4$}
\addplot+[color=gdcgreen, dashed, mark=none, forget plot] coordinates {(0,0.178681118) (252,0.178681118)};
\addplot+[color=gdcgreen, mark=square*, line width=1.25pt] coordinates
{(0,1.654877) (15,1.335476) (31,1.191823) (47,0.748728) (63,0.460604)
 (79,0.290241) (95,0.223798) (111,0.199105) (127,0.192239) (143,0.190656)
 (159,0.184248) (175,0.181177) (191,0.180230) (207,0.180060) (223,0.180060)
 (229,0.180060)};
\addlegendentry{$d=6$}
\addplot+[color=gdcorange, dashed, mark=none, forget plot] coordinates {(0,0.122788129) (252,0.122788129)};
\addplot+[color=gdcorange, mark=triangle*, line width=1.25pt] coordinates
{(0,1.440241) (15,1.235314) (31,1.136094) (47,0.533969) (63,0.323757)
 (79,0.246822) (95,0.180801) (111,0.155977) (127,0.147457) (143,0.145244)
 (159,0.140132) (175,0.133627) (191,0.131600) (207,0.131291) (223,0.131214)
 (236,0.131214)};
\addlegendentry{$d=8$}
\addplot+[color=gdcgray, dashed, mark=none, forget plot] coordinates {(0,0.178549407) (252,0.178549407)};
\addplot+[color=gdcgray, mark=diamond*, line width=1.25pt] coordinates
{(0,1.665497) (17,1.460077) (35,1.049121) (53,0.623336) (71,0.366756)
 (89,0.293302) (107,0.250577) (125,0.222482) (143,0.204296) (161,0.187691)
 (179,0.183624) (197,0.182244) (215,0.181536) (233,0.181326) (251,0.181326)};
\addlegendentry{$d=10$}
\end{groupplot}
\end{tikzpicture}
\caption{Online convergence of GDC against SCIP branch-and-cut incumbents. Solid curves
show the GDC evaluation score over online iterations; dashed horizontal lines of the same
color show the corresponding SCIP incumbent for that dimension. SCIP is an offline
branch-and-cut reference, not an online competitor.}
\label{fig:gdc_online_curve}
\end{figure}

Across both multilevel families, GDC reaches evaluation scores very close to the offline
reference search and SCIP incumbents. The bilevel normalized gaps are all below
$0.23\%$, and the trilevel normalized gaps are all below $0.65\%$, the largest being
$0.64\%$ at $d=8$. The known and nested online baselines are much faster per
rollout, but they are trapped in substantially worse upper or top-level basins despite
satisfying the explicit feasibility scaffold. This is the intended role separation:
known structure maintains feasibility and local relaxation consistency, while GDC
supplies the missing cross-level progress direction.
The SCIP branch-and-cut results give complementary evidence that certification for the
original nonconvex bilevel and trilevel formulations becomes expensive as dimension
grows; except for the smallest bilevel and trilevel cases, SCIP does not certify global
optimality within the four-hour limit. On the time-limit cases, GDC completes online
evaluation in $52.2$--$138.9$ seconds, using only $0.36\%$--$0.96\%$ of the same
four-hour budget; the resulting SCIP/GDC wall-time factors are therefore conservative
lower bounds, ranging from more than $103\times$ to more than $275\times$.

Figure~\ref{fig:scip_time_cap} compares online GDC wall time with SCIP branch-and-cut
wall time after normalizing both by the four-hour SCIP budget; time-limit rows are
treated as incumbent-only offline references. Figure~\ref{fig:gdc_online_curve} shows the GDC online trajectories
for all reported bilevel and trilevel dimensions approaching their corresponding SCIP
incumbents.

\section{Discussion and Conclusion}
\label{sec:discussion}

GDC is intended for partial-physics settings, not for settings where either complete
model-based optimization or pure behavioral cloning is already sufficient. Its central
design choice is to keep the known scaffold explicit and use feasible-progress data only
for the missing directional structure. This distinguishes GDC from diffusion policies,
where sampling is the control law, and from standard safe-learning pipelines, where hard
constraints are often introduced as costs, shields, or post-hoc corrections.
The present results support two complementary claims. In the multilevel optimization
simulation, the known-only solver preserves the explicit feasibility scaffold but stalls
in poor nonconvex basins, while GDC learns a cross-level progress direction that reaches
near-reference $\Phi$ values. In SUMO, the learned score actively selects almost every action
while the explicit known-constraint layer maintains collision-free route completion and
improves arrival relative to rule, imitation, and RL baselines. These settings use different physics and different
meanings of progress, but the same separation of roles: known constraints determine what
may be executed, and the learned progress path determines which feasible direction is
useful.

\appendix

\section{Additional SUMO Diagnostics}
\label{sec:appendix_sumo_diagnostics}

Figure~\ref{fig:sumo_scalability_tls_paths} and
Figure~\ref{fig:sumo_scalability_long_paths} provide additional diagnostics for the
larger $60\times60$ complex TLS map. These figures are separate scenarios: the first
uses the recorded scalability routes, while the second uses newly recorded long-route
episodes. Table~\ref{tab:sumo_appendix_diagnostics} attaches each SUMO diagnostic
figure to the numerical conclusion it supports, including the two held-out transfer
figures in Section~\ref{sec:appendix_sumo_transfer}. Traffic-light constraints are
rendered as signal markers. Star markers are reserved for slow or other non-TLS
roadblocks; none are present in this TLS-only scalability map.
\begin{figure}[!htbp]
\centering
\includegraphics[width=0.92\linewidth]{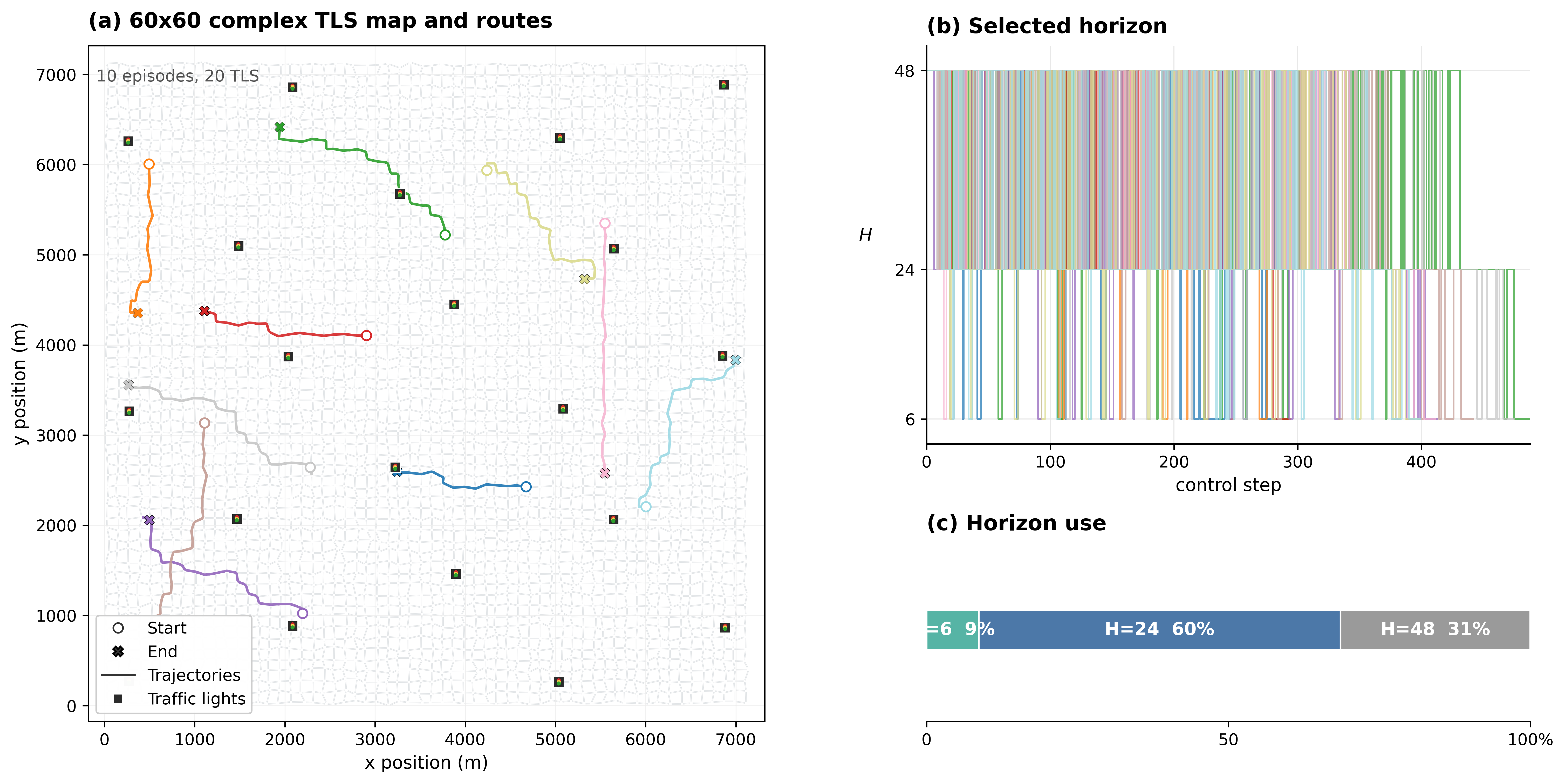}
\caption{Complex TLS scalability diagnostic on the $60\times60$ map. The figure shows
ten recorded ego paths, the selected planning horizon at each control step, and the
aggregate horizon-use distribution for this scenario.}
\label{fig:sumo_scalability_tls_paths}
\end{figure}
\subsection{SUMO Transfer Diagnostics}
\label{sec:appendix_sumo_transfer}

We also include representative qualitative transfer diagnostics for the two held-out
SUMO transfer settings. \emph{Transfer I: blocked-map transfer} keeps the base
perturbed road topology fixed and changes the deployment map through held-out road
closures and slow/bottleneck edges. This tests whether GDC can reuse the learned
progress direction when the feasible scaffold is locally cut or slowed, rather than
only replaying routes from the training map. \emph{Transfer II: graph-and-disturbance
transfer} changes the generated graph itself, including grid dimensions, lane count,
edge length, speed, traffic-light spacing, and geometric perturbation, while also
introducing new closures and bottlenecks. This harder setting tests whether the same
progress model remains useful when both the scaffold geometry and the local disturbances
move out of the original graph. In both settings, the online controller receives the
current known map constraints as $\Cknown$; transfer therefore concerns the learned
direction field on top of a new scaffold, not blind map inference. We show
representative maps rather than the full set of ten visualized maps per setting.

\begin{figure}[!htbp]
\centering
\includegraphics[width=0.92\linewidth]{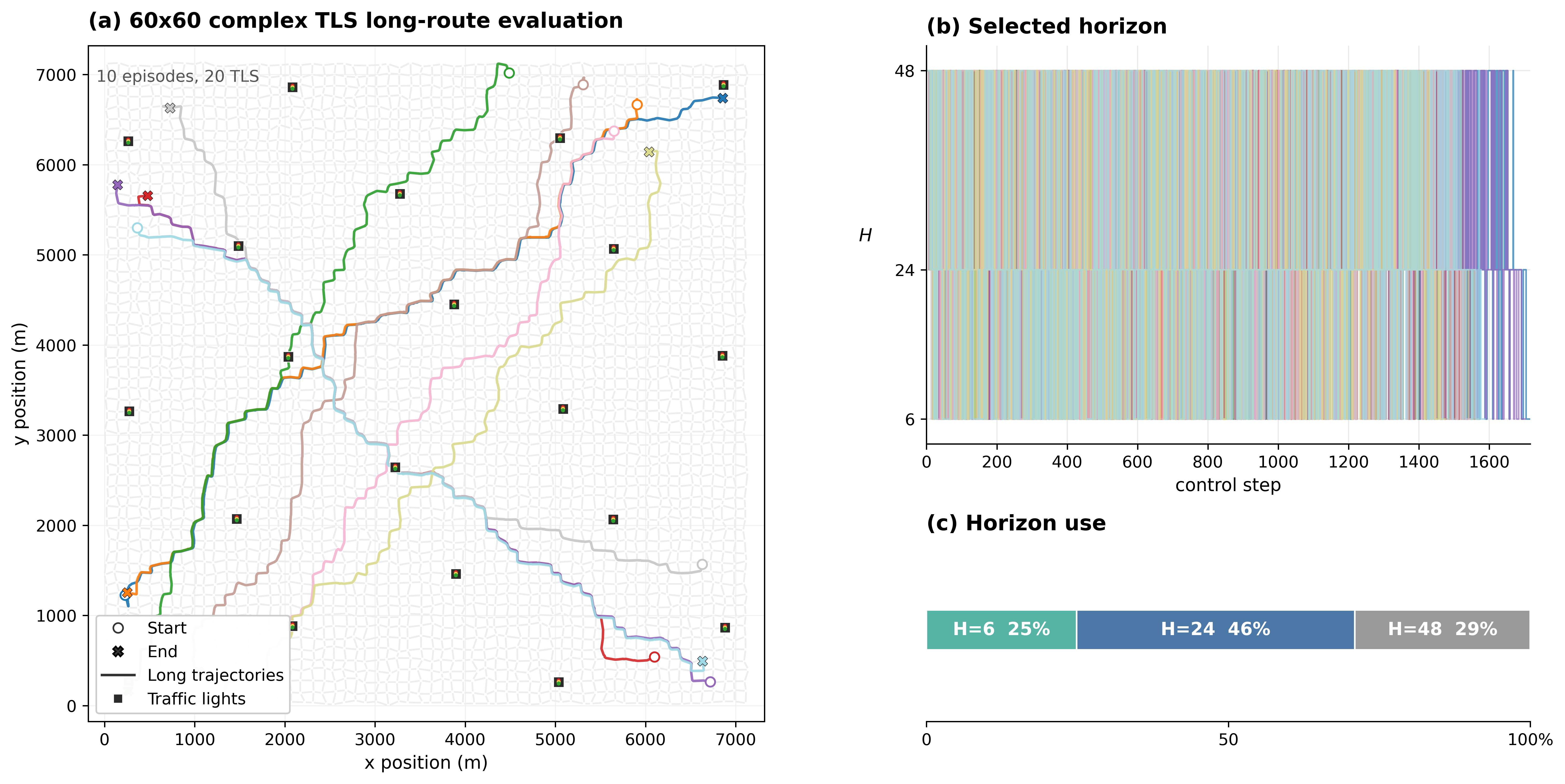}
\caption{Long-route complex TLS scalability diagnostic on the $60\times60$ map. The
ten trajectories use $7.0$--$8.1\,\mathrm{km}$ routes and complete in
$1{,}475$--$1{,}717$ control steps; all recorded episodes reach the route goal with
zero collisions and no timeout.}
\label{fig:sumo_scalability_long_paths}
\end{figure}
\begin{figure}[!htbp]
\centering
\begin{tabular}{@{}cc@{}}
\includegraphics[width=0.47\linewidth]{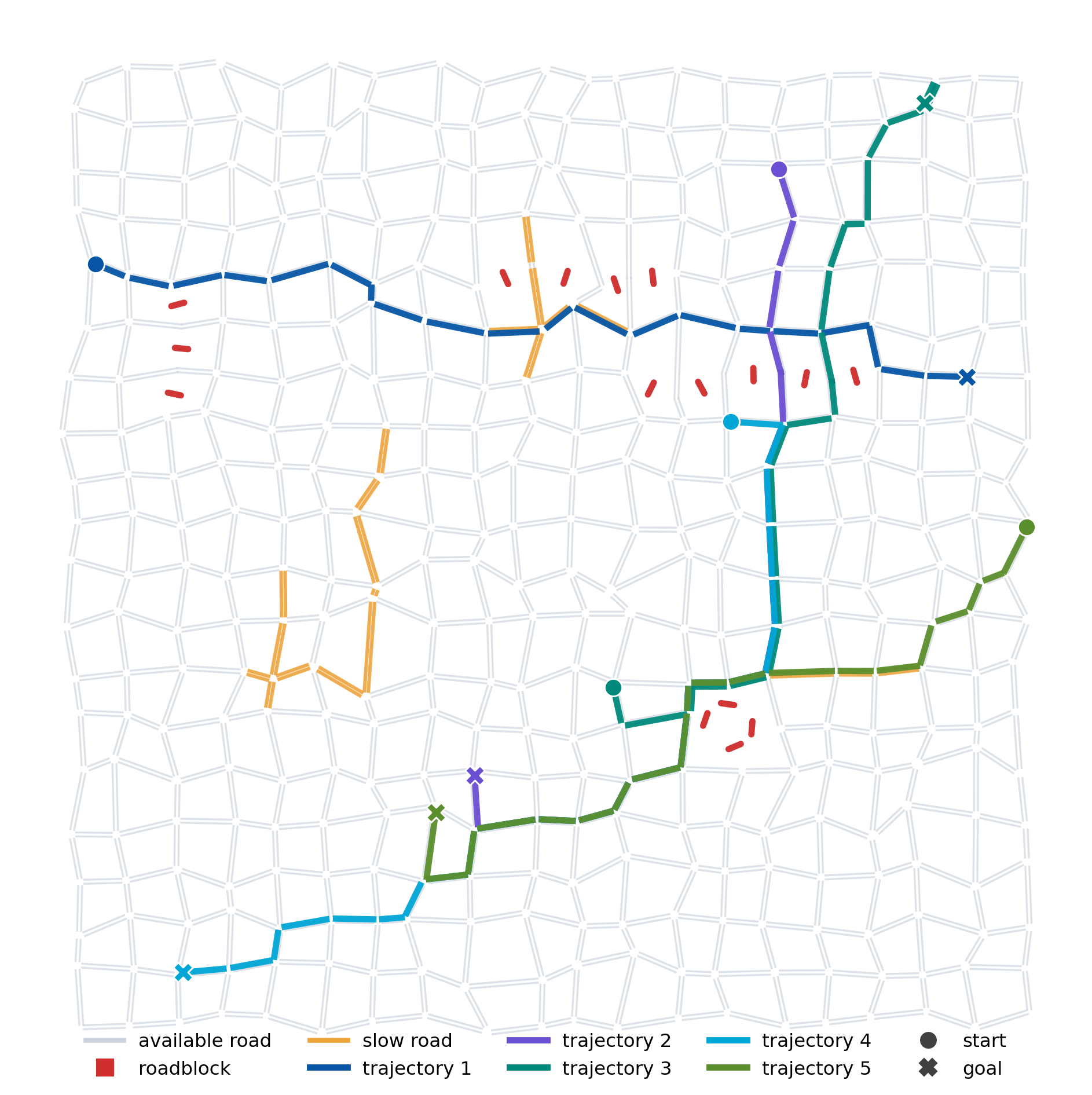} &
\includegraphics[width=0.47\linewidth]{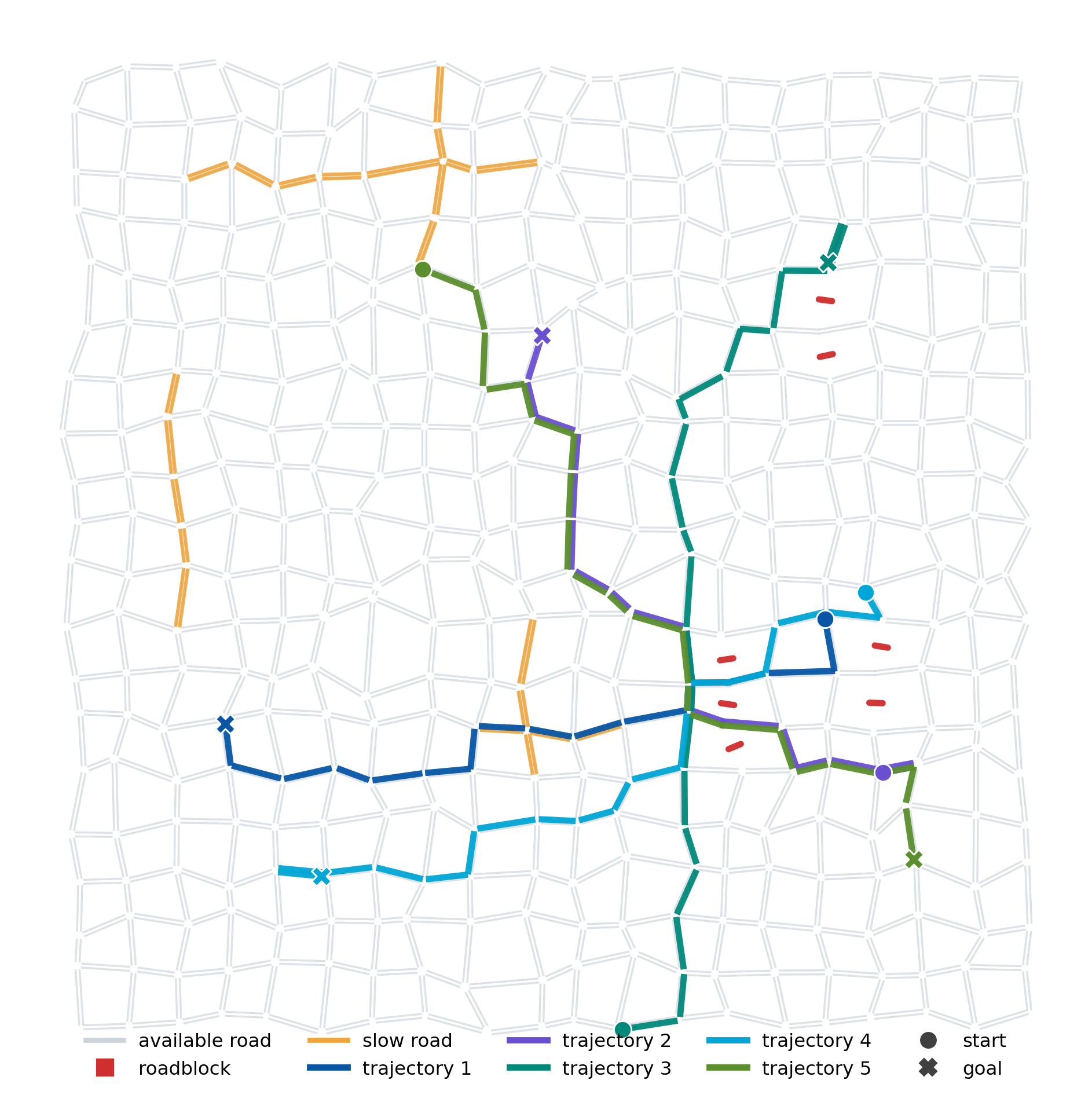} \\
\includegraphics[width=0.47\linewidth]{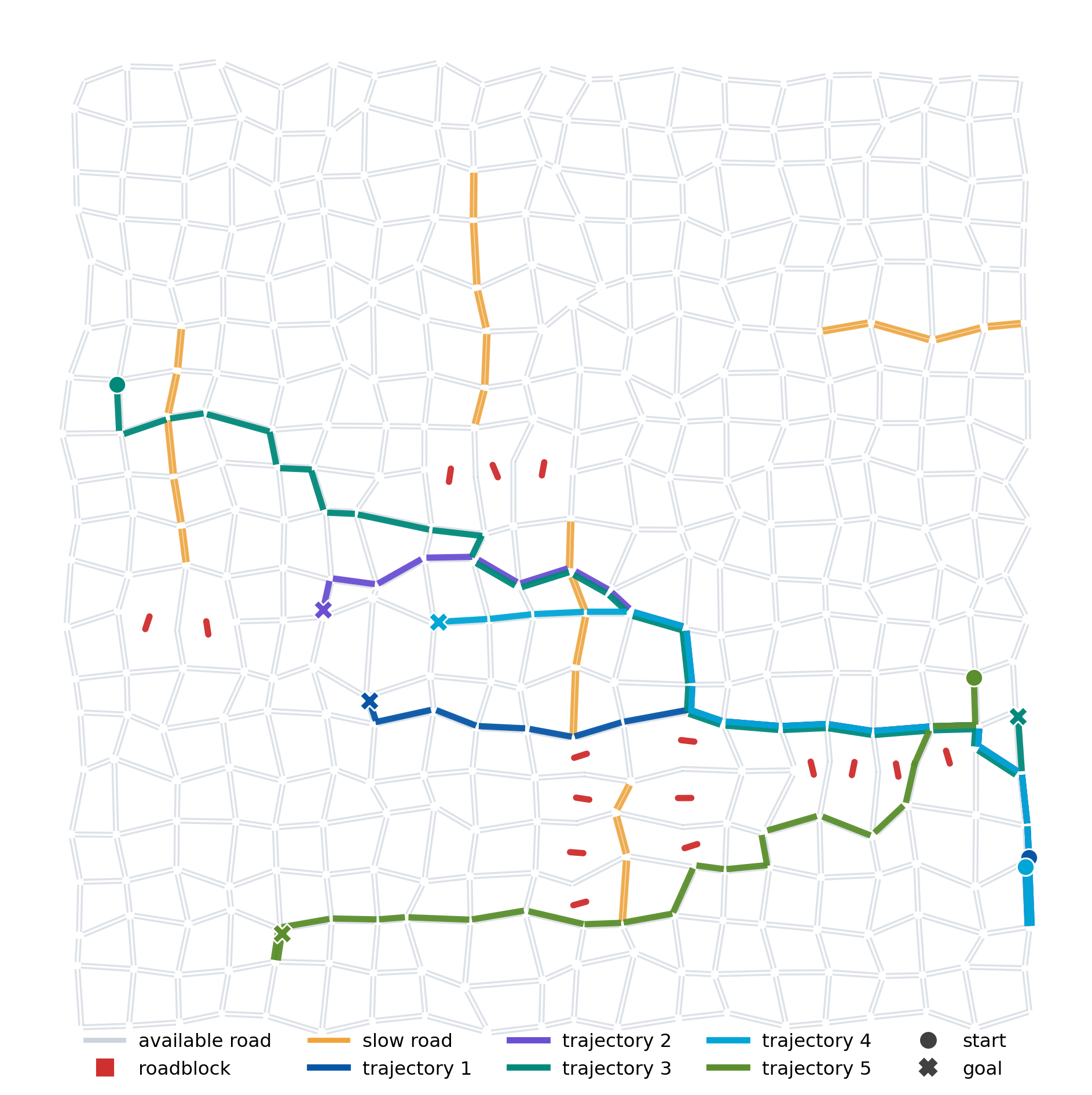} &
\includegraphics[width=0.47\linewidth]{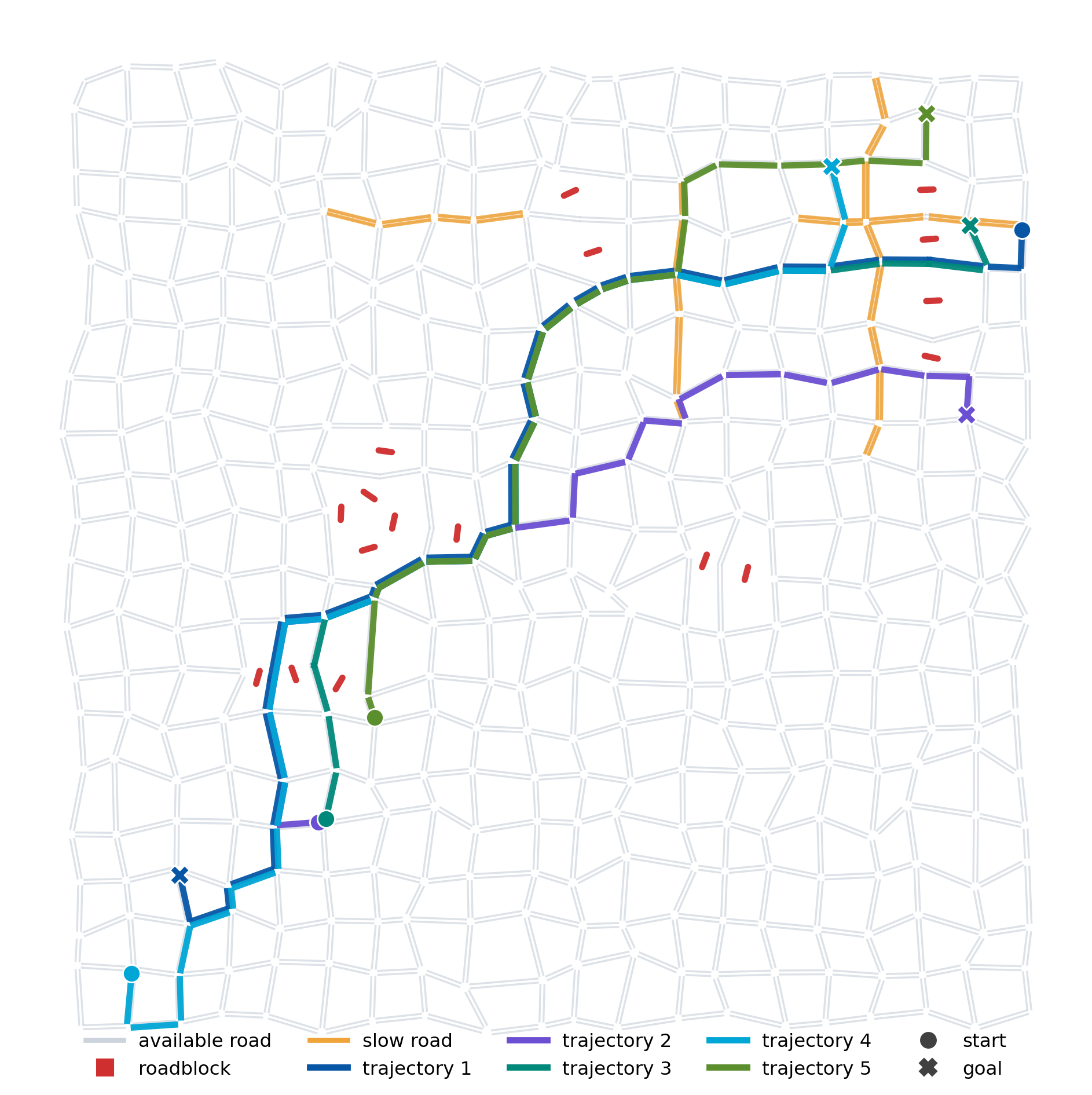}
\end{tabular}
\caption{Representative blocked-map transfer diagnostics. Four held-out maps are shown;
each panel overlays five long routes. Red marks closures, orange marks slow/bottleneck
roads, and the $1.7$--$2.4\,\mathrm{km}$ routes make transfer-time detours inspectable.}
\label{fig:sumo_transfer_blocked_representatives}
\end{figure}

\noindent\emph{Blocked-map transfer.} This setting holds the base perturbed topology
fixed and changes only held-out closures and slow roads. On the corresponding
$100$-map benchmark, GDC keeps arrival/progress at $1.000/1.000$, has zero known
violations, and uses nonfallback learned control on $99.7\%$ of decisions; the plotted
detours therefore summarize the same transfer behavior rather than only qualitative
overlays.

\begin{figure}[!htbp]
\centering
\includegraphics[width=0.96\linewidth]{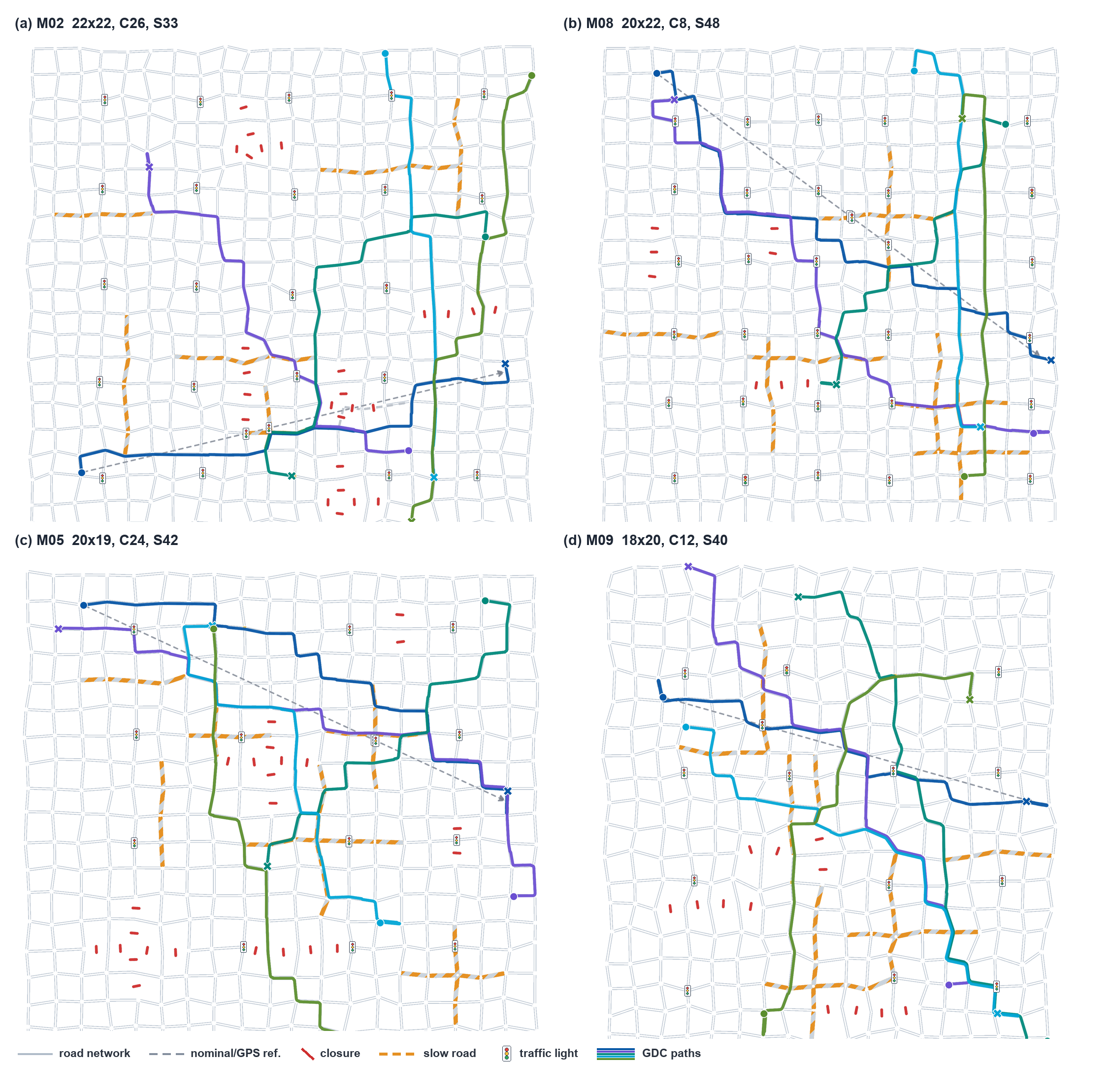}
\caption{Held-out graph-and-disturbance transfer examples. Panels use unseen graph
parameters and disturbances; red ticks mark closures, dashed orange marks slow roads,
icons mark traffic lights, and colored curves are successful long GDC routes.}
\label{fig:sumo_transfer_map_representatives}
\end{figure}

\noindent\emph{Graph-and-disturbance transfer.} This harder setting changes both the
generated graph and its local disturbances. The shown maps span $18\times19$ to
$22\times22$ grids, $8$--$26$ closure pairs, and $33$--$48$ bottleneck pairs; the
$20$ shown closed-loop traces still reach their route goals with zero collisions.

\begin{samepage}
\begin{center}
\phantomsection
\refstepcounter{table}
\label{tab:sumo_appendix_diagnostics}
\begingroup
\footnotesize
\setlength{\tabcolsep}{2.5pt}
\renewcommand{\arraystretch}{0.96}
\begin{tabular}{@{}>{\raggedright\arraybackslash}p{0.15\linewidth}
>{\raggedright\arraybackslash}p{0.16\linewidth}ccc
>{\raggedright\arraybackslash}p{0.23\linewidth}
>{\raggedright\arraybackslash}p{0.19\linewidth}@{}}
\toprule
\textbf{Diagnostic} & \textbf{Routes / maps} & \textbf{Arr./Prog.} &
\textbf{Coll.} & \textbf{Steps} & \textbf{Stressor / diagnostic signal} &
\textbf{Takeaway} \\
\midrule
Main complex TLS,
Fig.~\ref{fig:sumo_multi_h_paths} &
$100$ eps.; $1.15\,[0.85,1.44]\,\mathrm{km}$. &
$1.000/1.000$ & $0$ & $322.9$ &
$H{=}24$ on $95.3\%$ of decisions; mean speed
$8.96\,\mathrm{m/s}$. &
Long-horizon progress is used while the known scaffold preserves feasibility. \\
\addlinespace[0.2pt]
$60\times60$ TLS,
Fig.~\ref{fig:sumo_scalability_tls_paths} &
$10$ routes; $1.72\,[1.20,2.19]\,\mathrm{km}$. &
$1.000/1.000$ & $0$ & $361.1$ &
$H{=}24/48$ on $91.4\%$ of decisions; mean speed
$12.41\,\mathrm{m/s}$. &
Same controller scales to the larger traffic-light graph. \\
\addlinespace[0.2pt]
$60\times60$ long TLS,
Fig.~\ref{fig:sumo_scalability_long_paths} &
$10$ routes; $7.43\,[7.03,8.09]\,\mathrm{km}$. &
$1.000/1.000$ & $0$ & $1578.0$ &
$7.0$--$8.1\,\mathrm{km}$ routes; $H{=}24/48$ on $75.2\%$ of decisions. &
Tests kilometer-scale closed-loop rollout, beyond local completion. \\
\addlinespace[0.2pt]
Transfer I: blocked map,
Fig.~\ref{fig:sumo_transfer_blocked_representatives} &
$4$ shown maps, $20$ routes; $1.97\,[1.49,2.78]\,\mathrm{km}$. &
$1.000/1.000$ & $0$ & $252.0$ &
Fixed topology; held-out closures and slow roads; $100$-map benchmark:
zero known violation, $99.7\%$ nonfallback. &
Detours are backed by transfer metrics, not only overlays. \\
\addlinespace[0.2pt]
Transfer II: graph + disturbance,
Fig.~\ref{fig:sumo_transfer_map_representatives} &
$4$ shown maps, $20$ traces; $2.32\,[1.90,3.06]\,\mathrm{km}$. &
$1.000/1.000$ & $0$ & $579.9$ &
Unseen graph parameters plus disturbances: $18\times19$--$22\times22$ grids,
$8$--$26$ closure pairs, $33$--$48$ bottleneck pairs. &
Tests joint graph and disturbance transfer. \\
\bottomrule
\end{tabular}
\par\vspace{0.35em}
\begin{minipage}{0.96\linewidth}
\small\textbf{Table~\thetable:} Summary of the SUMO diagnostic figures. Bracketed route lengths report
minimum and maximum values. Transfer I fixes the base topology and changes
disturbances, while Transfer II changes graph parameters and disturbances together.
The transfer rows summarize the shown routes and the corresponding held-out outcomes.
\end{minipage}
\endgroup
\end{center}
\end{samepage}
 
\FloatBarrier

\section{Additional Theoretical Analysis}
\label{sec:theory}

The following results expand the guarantees stated in
Section~\ref{sec:formal-guarantees}. They analyze the fixed objects learned offline and
used online: the empirical manifold, the progress-tilted density, its score field, and the
induced latent control flow. These results are not additional offline training steps.

\subsection{Notation and standing assumptions for the proofs}
\label{sec:proof_notation}

For the proof sections, write
$c=(o,g,\Cknown,H^\star)$ for the complete conditioning context, including the selected
horizon when it is available. All densities on $\Mhat$ are densities with respect to the
Riemannian volume measure $\dd\volG$. Euclidean gradients in a latent chart are denoted
by $\nabla_z$, while $\grad_G$ denotes the Riemannian gradient. The corresponding
Riemannian norm and dual norm are
$\|v\|_G^2=v^\top G(z)v$ and $\|\alpha\|_{G^{-1}}^2=\alpha^\top G(z)^{-1}\alpha$.
For two probability measures $\mu,\nu$ on $\Mhat$, $\|\mu-\nu\|_{\TV}$ denotes total
variation distance with the convention
$\|\mu-\nu\|_{\TV}=\sup_A|\mu(A)-\nu(A)|=\frac12\|\mu-\nu\|_1$. The logarithmic score is the chart gradient
$s_\theta=\nabla_z\log p_\theta^Z$; the associated Riemannian score vector is
$G^{-1}s_\theta=\grad_G\log p_\theta^Z$.

\begin{assumption}[Standing regularity]
\label{ass:standing_theory}
The following conditions are used in the appendix whenever the corresponding object
appears.
\begin{enumerate}[label=(\alph*)]
  \item $\Mhat$ is a connected $C^2$ finite-dimensional manifold equipped with a $C^1$
        positive-definite metric $G$. The metric is uniformly elliptic on the
        data-supported region considered in the theorem. Either $\Mhat$ has no boundary,
        or the boundary condition is reflecting/no-flux for all gradient-flow and
        diffusion arguments.
  \item The potentials
        $\Phi_{\mathrm{known}}$, $\Phi_{\mathrm{data}}$, $\Phi_{\mathrm{task}}$, and
        their sums are $C^1$ on the relevant region and bounded below; in results using
        Hessians or strong convexity they are $C^2$ on a geodesically convex
        neighborhood that does not intersect the cut locus of the minimizer.
  \item The conditional latent path densities $p_\theta^Z(\cdot\mid c)$ and their ideal targets
        are strictly positive and $C^1$ on the data-supported region. Score errors are
        measured in the dual Riemannian norm unless an explicit Euclidean constant is
        stated.
  \item The decoder $D_\psi(\cdot,c)$ is $K_{D_\psi}$-Lipschitz on the trusted latent
        region, and its reconstruction error for the first decoded action is bounded by
        $\delta_{\mathrm{dec}}$.
  \item The known-feasible execution set
        $\mathcal U_{\mathrm{known}}(o_t,u_{t-1})$ is closed and nonempty whenever the
        projection layer is invoked. If no learned candidate is accepted, the fallback
        action is assumed to belong to this set.
  \item Progress certificates are bounded on every selected finite horizon:
        $|\Delta\rho(\tau,g)|\le R_\rho$. When clipped certificates are used, $R_\rho=c$.
\end{enumerate}
\end{assumption}

\begin{remark}[Scope of the guarantees]
The proofs are local to the data-supported known-feasible region where
$H^\star\neq\varnothing$. Outside this region the algorithm may still execute a
known-feasible fallback, but the progress-value and score-direction claims are not
invoked.
\end{remark}

\begin{remark}[Frozen context]
The stationary-distribution, free-energy, and Wasserstein statements, and every
$\limsup_{t\to\infty}$ bound, are stated for a fixed conditioning context
$c=(o,g,\Cknown,H^\star)$. The deployed controller re-evaluates $H_t^\star$ and $o_t$ at
each step, so these results describe the frozen-context limit that the online iteration
tracks between context changes, not a single time-homogeneous process.
\end{remark}

\subsection{Energy-based interpretation and implicit value function}
\label{sec:energy}

\subsubsection{Data energy function}

The marginal likelihood $p_\theta^Z(z \mid o, g, \Cknown,H^\star)$ induced by the probabilistic
path on $\Mhat$ defines an energy function over the learned empirical manifold.

\begin{definition}[Data energy]
  The \emph{data energy function} on $\Mhat$ is
  \begin{equation}
    \Phi_{\mathrm{data}}(z, o, g, \Cknown,H^\star)
    = -\log p_\theta^Z(z \mid o, g, \Cknown,H^\star).
    \label{eq:data_energy}
  \end{equation}
\end{definition}

Regions of low energy correspond to latent behaviors that are highly consistent with
feasible trajectory demonstrations and their progress certificates under observation $o$,
goal $g$, and the known constraint context.

\subsubsection{Score function on the empirical manifold}

\begin{definition}[Score function]
  The \emph{score function} of the learned distribution on $\Mhat$ is
  \begin{equation}
    s_\theta(z, o, g, \Cknown,H^\star)
    = \nabla_z \log p_\theta^Z(z \mid o, g, \Cknown,H^\star)
    = -\nabla_z \Phi_{\mathrm{data}}(z, o, g, \Cknown,H^\star).
    \label{eq:score}
  \end{equation}
\end{definition}

With a clean-behavior predictor, the score at noise level $\ell$ is induced by the
predicted clean latent:
\begin{equation}
  s_{\theta,\ell}^{\mathrm{diff}}(z_\ell, o, g, \Cknown,H^\star)
  \approx
  \frac{\sqrt{\bar\alpha_\ell}\,F_\theta(z_\ell,o,g,\Cknown,H^\star,\ell)-z_\ell}
       {1-\bar\alpha_\ell}.
  \label{eq:score_from_clean}
\end{equation}

\subsubsection{Implicit progress value theorem}

This is the central theoretical contribution connecting probabilistic paths to control.

The fundamental problem of directional ambiguity is that local sensing defines the
\emph{feasible set} but not the \emph{direction}. At an intersection, the camera
tells you that left, right, and straight are all traversable---but not which
leads to the goal. A local controller $u_t = f(o_t)$ has access only to feasibility,
not to directional intent.

GDC resolves this by replacing the local policy with a goal-conditioned latent
trajectory distribution
$\tau \sim p_\theta^\tau(\tau \mid z, g, \Cknown,H^\star)$, whose latent code $z$ is
selected and refined online.
The key observation is that this distribution is trained on known-feasible trajectory
segments with signed progress certificates. The objective upweights segments that improve
a goal-conditioned progress certificate without requiring every short segment to be
improving. Therefore, the distribution itself encodes directional information as a
statistical bias: at the intersection, right-turning segments receive more mass when they
increase expected progress under similar contexts, while necessary waiting or
repositioning behaviors need not be excluded.

This statistical bias is what the score function extracts:
$s_\theta(z) = \nabla_z \log p_\theta^Z(z \mid o, g, \Cknown,H^\star)$ points toward
latent regions whose decoded segments have higher probability under the progress-tilted
conditional distribution.
The known-constraint potential $\nabla\Phi_{\mathrm{known}}(z)$ removes options that
violate modeled physics; it does not choose among the remaining feasible options or
recover unmodeled physics. The score does.

We now formalize this intuition.

\begin{definition}[Data-induced progress value]
  \label{def:implicit_value}
  The \emph{data-induced progress value} at latent point $z \in \Mhat$ is
  \begin{equation}
    V_{\mathrm{prog}}(z,g) =
    -\log p_\theta^\tau(\text{progress} \mid z,g,\Cknown,H^\star),
    \label{eq:implicit_value}
  \end{equation}
  where
  \begin{equation}
    p_\theta^\tau(\text{progress} \mid z,g,\Cknown,H^\star)
    =
    \int p_\theta^\tau(\tau \mid z,g,\Cknown,H^\star)\,
    \mathbb{I}[\Delta \rho(\tau,g) > 0]\,
    \mathbb{I}[\tau \in \Cknown]\,\dd\tau
  \end{equation}
  is the probability of generating a known-feasible trajectory segment that improves
  the progress certificate $\rho$.
\end{definition}

\begin{theorem}[Score as progress-value gradient]
  \label{thm:score_value}
  Fix a context $c_0=(o,g,\Cknown,H^\star)$ satisfying
  Assumption~\ref{ass:standing_theory}, and write $H:=H^\star$. Let $\Gamma_H(z,c_0)$ be
  the set of known-feasible horizon-$H$ continuations starting from latent point $z$, with
  reference measure $\nu_z$; a continuation $\tau$ carries the latent rollout $(z_j)_{j<H}$
  used to evaluate $C_H$. Under the maximum-entropy model of feasible-progress demonstrations,
  suppose a trajectory continuation has unnormalized weight
  \begin{equation}
    \exp\!\left(-C_H(\tau;c_0)\right),
    \qquad
    C_H(\tau;c_0)=\sum_{j=0}^{H-1}\ell_{\mathrm{prog}}(z_j,u_j;c_0),
    \label{eq:boltzmann_rational}
  \end{equation}
  where lower $\ell_{\mathrm{prog}}$ corresponds to larger signed progress certificate
  improvement. Assume the partition function below is finite, positive, and $C^1$ in
  $z$ on the data-supported region. Define the soft progress value
  \begin{equation}
    V_{\mathrm{prog}}^{\mathrm{soft}}(z,c_0)
    =
    -\log Z_H(z,c_0),
    \qquad
    Z_H(z,c_0)=\int_{\Gamma_H(z,c_0)}
    \exp[-C_H(\tau;c_0)]\,\dd\nu_z(\tau).
    \label{eq:soft_value_partition}
  \end{equation}
  If the ideal progress-tilted latent density is
  $p^\star(z\mid c_0)=Z_{c_0}^{-1}Z_H(z,c_0)$, then
  \begin{equation}
    -\log p^\star(z\mid c_0)
    =
    V_{\mathrm{prog}}^{\mathrm{soft}}(z,c_0)+\log Z_{c_0}.
    \label{eq:soft_value}
  \end{equation}
  This identification is a modeling assumption: it holds when the empirical
  progress-weighted aggregate posterior $q_\beta^Z$ of \eqref{eq:latent_tilted_pushforward},
  whose reference measure is the known-feasible data measure, agrees with the
  maximum-entropy continuation density $Z_{c_0}^{-1}Z_H(\cdot,c_0)$; the theorem does not
  derive this agreement from the offline objective.
  Consequently, if the learned latent path has $C^1$ log-density error
  $\|\nabla_z\log p_\theta^Z-\nabla_z\log p^\star\|\le
  \varepsilon_{\mathrm{approx}}$ on the data-supported region, then
  \begin{equation}
    \boxed{s_\theta(z, o, g, \Cknown,H^\star)
    = \nabla_z \log p_\theta^Z(z \mid o, g, \Cknown,H^\star)
    = -\nabla_z V_{\mathrm{prog}}^{\mathrm{soft}}(z,c_0)
    + \mathcal{O}(\varepsilon_{\mathrm{approx}}),}
    \label{eq:score_is_value_gradient}
  \end{equation}
\end{theorem}

\begin{proof}
  By definition of $p^\star$,
  \[
    \log p^\star(z\mid c_0)
    =
    \log Z_H(z,c_0)-\log Z_{c_0}
    =
    -V_{\mathrm{prog}}^{\mathrm{soft}}(z,c_0)-\log Z_{c_0}.
  \]
  The normalizer $Z_{c_0}$ depends only on the conditioning context, not on $z$.
  Differentiating with respect to $z$ therefore gives
  \[
    \nabla_z\log p^\star(z\mid c_0)
    =
    -\nabla_z V_{\mathrm{prog}}^{\mathrm{soft}}(z,c_0).
  \]
  The assumed $C^1$ regularity of $Z_H$ justifies the differentiation; one sufficient
  condition is a smoothly parameterized continuation set and dominated differentiation
  under the reference measures $\nu_z$. Finally,
  $s_\theta=\nabla_z\log p_\theta^Z$ and the assumed $C^1$ log-density approximation imply
  \[
    \left\|
    s_\theta(z,o,g,\Cknown,H^\star)
    +\nabla_z V_{\mathrm{prog}}^{\mathrm{soft}}(z,c_0)
    \right\|
    \le \varepsilon_{\mathrm{approx}},
  \]
  which is \eqref{eq:score_is_value_gradient}.
\end{proof}

\begin{remark}[Resolving directional ambiguity]
  Theorem~\ref{thm:score_value} explains why GDC resolves the directional ambiguity
  problem. The mechanism is \emph{not} frequency separation between sensor channels.
  Rather, the mechanism is that training with signed progress certificates creates a
  distribution whose selected-horizon samples have positive expected progress when the
  horizon selector \eqref{eq:selected_horizon} succeeds.

  At a decision point (e.g., an intersection), the observation $o_t$ defines the
  feasible set $\{$left, right, straight$\}$ but provides no directional preference.
  The path model $p_\theta^\tau(\tau \mid z, g, \Cknown,H^\star)$, trained with
  progress-weighted feasible segments, concentrates mass on directions that improve the
  progress certificate in expectation under similar contexts when
  $H^\star\neq\varnothing$. The score
  $s_\theta = \nabla_z \log p_\theta^Z$ converts this
  statistical bias into a deterministic gradient that points toward the locally better
  region. The known constraint potential
  $\nabla\Phi_{\mathrm{known}}$ removes explicitly modeled infeasible options; the score
  chooses among known-feasible options and approximates the implicit feasibility structure.

  Concretely, define $V_{\mathrm{prog}}(z,g) =
	  -\log p_\theta^\tau(\text{progress} \mid z,g,\Cknown,H^\star)$. Then if the path model is trained with
  signed progress certificates, the distribution satisfies a \emph{progress-bias property}:
  \begin{equation}
    \E_{\tau \sim p_\theta^\tau(\cdot \mid z, g,\Cknown,H^\star)}
    \!\left[\Delta \rho(\tau,g)\right] > 0.
    \label{eq:value_decrease_intro}
  \end{equation}
  The short-horizon distribution is biased toward improving the progress certificate in
  expectation; individual samples may still wait, backtrack, or reposition when such
  behavior appears in the data-supported strategy.
  The directional information is compressed into the distribution over short trajectories,
  because the training objective weights feasible segments by progress.
  This is why the controller can use the shortest reliable planning horizon rather than a
  long goal-reaching rollout: once \eqref{eq:selected_horizon} succeeds, the directional
  bias is already in the selected-horizon distribution.
\end{remark}

\begin{remark}[Honest caveats]
  The implicit $V_{\mathrm{prog}}(z,g)$ is a progress value of the \emph{data policy},
  not the optimal value function. It encodes ``how likely is local progress if you behave
  like the demonstrations,'' not ``how likely is final success under the best possible
  behavior.''
  The Boltzmann-rational assumption \eqref{eq:boltzmann_rational} is an idealization;
  real demonstrations are only approximately maximum-entropy. The connection is
  strongest when training data is diverse and covers the relevant state space.
\end{remark}

\subsubsection{Structural role separation in the control law}

  The GDC control law
  $\dot{z} = s_\theta(z, o, g, \Cknown,H^\star) - \nabla\Phi_{\mathrm{known}}(z)$ provides a
secondary benefit beyond the primary value-gradient mechanism: the two terms play
structurally distinct roles that prevent high-bandwidth sensor data from drowning
out the sparse directional signal.

In the control law:
\begin{equation}
  \dot{z} = \underbrace{s_\theta(z, o, g, \Cknown,H^\star)}_{\substack{\text{learned implicit information:}\\\text{where to go and how complex physics behaves}}}
  - \underbrace{\nabla_z \Phi_{\mathrm{known}}(z)}_{\substack{\text{known physics:}\\\text{what is explicitly constrained}}},
  \label{eq:role_separation}
\end{equation}
the score $s_\theta$ encodes the goal-conditioned directional information and
unmodeled feasibility patterns learned from feasible trajectories with progress
certificates, while
$\nabla\Phi_{\mathrm{known}}$ encodes explicit feasibility from known physics and
high-frequency sensors. The two signals enter the control law additively
but in structurally different roles:
the score selects \emph{which} known-feasible option to pursue and approximates
implicit effects; the potential gradient removes \emph{which} options violate
known constraints.
This separation ensures that high-bandwidth feasibility data (``the road is clear'')
cannot override the directional signal (``turn right''), because they act on
different components of the control vector.

However, this role separation is a \emph{consequence} of the primary mechanism
(Theorem~\ref{thm:score_value}), not the cause. The directional information
originates from the training data bias induced by progress certificates, which creates
the score. The structural separation merely preserves that information in the
control law.

\subsubsection{Total control potential on the empirical manifold}

Let $\Phi_{\mathrm{task}} : \Mhat \to \R$ encode task-specific objectives.

\begin{definition}[Total control potential]
  The \emph{total control potential} on $\Mhat$ is
  \begin{equation}
  \begin{aligned}
    \Phi_{\mathrm{GDC}}(z)
    &= \Phi_{\mathrm{known}}(z)
    + \Phi_{\mathrm{data}}(z, o, g, \Cknown,H^\star) \\
    &\quad + \Phi_{\mathrm{task}}(z) \\
    &= \Phi_{\mathrm{known}}(z)
    -\log p_\theta^Z(z \mid o, g, \Cknown,H^\star) \\
    &\quad + \Phi_{\mathrm{task}}(z).
  \end{aligned}
    \label{eq:total_potential}
  \end{equation}
\end{definition}

By Theorem~\ref{thm:score_value}, this is approximately, up to a $z$-independent constant:
\begin{equation}
  \Phi_{\mathrm{GDC}}(z) \approx
  \Phi_{\mathrm{known}}(z) + V_{\mathrm{prog}}^{\mathrm{soft}}(z,c_0) + \Phi_{\mathrm{task}}(z),
\end{equation}
so control minimizes the sum of explicit known physics, the data-induced approximation
of implicit information, and the task cost.

\subsection{Score-guided control as Riemannian gradient flow}
\label{sec:score_control}

\subsubsection{Riemannian gradient flow on the empirical manifold}

Rather than performing gradient flow in flat Euclidean space, GDC implements control as
\emph{Riemannian} gradient flow on $\Mhat$, respecting the intrinsic geometry of
the learned behavior manifold while explicitly enforcing known constraints.

\begin{definition}[Riemannian gradient flow control]
  The \emph{Riemannian gradient flow control dynamics} on $\Mhat$ are
  \begin{equation}
    \dot{z} = -\grad_G \Phi_{\mathrm{GDC}}(z)
    = -G(z)^{-1} \nabla_z \Phi_{\mathrm{GDC}}(z).
    \label{eq:riemannian_gradient_flow}
  \end{equation}
\end{definition}

\begin{proposition}[Score-guided Riemannian control]
  \label{prop:score_riemannian}
  Under the decomposition \eqref{eq:total_potential}, the Riemannian gradient flow
  \eqref{eq:riemannian_gradient_flow} is equivalent to
  \begin{equation}
    \dot{z} = G(z)^{-1}\!\left[
      s_\theta(z, o, g, \Cknown,H^\star)
      - \nabla_z \Phi_{\mathrm{known}}(z)
      - \nabla_z \Phi_{\mathrm{task}}(z)
    \right].
    \label{eq:score_guided_riemannian}
  \end{equation}
\end{proposition}

\begin{proof}
  Compute:
  \begin{align}
    \grad_G \Phi_{\mathrm{GDC}}(z) &= G(z)^{-1} \nabla_z \Phi_{\mathrm{GDC}}(z) \\
    &= G(z)^{-1}\!\left[
      -s_\theta(z, o, g, \Cknown,H^\star)
      + \nabla_z \Phi_{\mathrm{known}}(z)
      + \nabla_z \Phi_{\mathrm{task}}(z)
    \right].
  \end{align}
  Substituting into $\dot{z} = -\grad_G \Phi_{\mathrm{GDC}}(z)$ yields
  \eqref{eq:score_guided_riemannian}.
\end{proof}

\begin{remark}[From sampling to gradient flow]
  The mathematical distinction between direct generative policy sampling and GDC is
  precisely the transformation \eqref{eq:core_shift}. In diffusion policy, the score drives
  a \emph{sampling} process: $\dot{u} = s_\theta(u)$, where the dynamics \emph{are}
  the controller. In GDC, the score enters a \emph{structured control law}:
  \begin{equation}
    \dot{z} = G(z)^{-1}\!\left[
      s_\theta(z, o, g, \Cknown,H^\star)
      - \nabla_z \Phi_{\mathrm{known}}(z)
      - \nabla_z \Phi_{\mathrm{task}}(z)
    \right],
  \end{equation}
  where the Riemannian metric $G(z)$ respects manifold curvature,
  $\Phi_{\mathrm{known}}$ encodes modeled physics and known constraints, and
  $\Phi_{\mathrm{task}}$ encodes task objectives. Known constraints are enforced by the
  support restriction $\Mhat_{\mathrm{known}}$, anchor penalties, and candidate filtering.
  The score provides the implicit information (via the implicit value gradient,
  Theorem~\ref{thm:score_value}); the known scaffold provides explicit feasibility and
  stability.

  In short: diffusion policy performs sampling in action space; GDC defines a
  controlled Riemannian gradient flow where the path score provides direction
  and geometry enforces feasibility.
\end{remark}

\subsubsection{Singular perturbation structure for known constraints}

For stable online control, the total online potential on $\Mhat$ admits a
singular perturbation decomposition. Writing $\Phi_\lambda$ for the interpolated
potential \eqref{eq:interpolated_potential}, define
\begin{equation}
  V_\lambda(z) = \frac{1}{\epsilon_{\mathrm{s}}}\,\Phi_{\mathrm{scaf}}(z)
  + \lambda\,\Phi_{\mathrm{known}}(z)
  + (1-\lambda)\,\Phi_{\mathrm{data}}(z, o, g, \Cknown,H^\star)
  + \Phi_{\mathrm{task}}(z),
  \label{eq:singular_perturbation_potential}
\end{equation}
where:
\begin{itemize}[leftmargin=2em]
  \item $\Phi_{\mathrm{scaf}}(z) =
        d_G(z,\mathcal M_t)^2$ is the fast scaffold
        potential that keeps $z$ near the known-feasible action-segment support
        (fast timescale, $O(\epsilon_{\mathrm{s}})$).
  \item $\Phi_{\mathrm{known}}(z)$
        is the tangential known-physics potential that slides along $\Mhat$
        while respecting modeled constraints (slow timescale, $O(1)$).
  \item $\Phi_{\mathrm{data}}(z, o, g, \Cknown,H^\star)
        = -\log p_\theta^Z(z \mid o, g, \Cknown,H^\star)$ is the
        data-induced potential providing implicit physics and directional guidance
        from the probabilistic path model.
  \item $\Phi_{\mathrm{task}}(z)$ is the runtime task potential, entering at
        $O(1)$ alongside $\Phi_{\mathrm{known}}$.
\end{itemize}

The nominal gradient flow is $\dot z_{\mathrm{nom}} = -\grad_G V_\lambda(z)$, that is,
\begin{equation}
  \dot{z}_{\text{nom}} = -G(z)^{-1}\!\left[
  \frac{1}{\epsilon_{\mathrm{s}}}\nabla_z \Phi_{\mathrm{scaf}}(z)
  + \lambda\,\nabla_z \Phi_{\mathrm{known}}(z)
  - (1 - \lambda)\,s_\theta(z, o, g, \Cknown,H^\star)
  + \nabla_z \Phi_{\mathrm{task}}(z)
  \right],
  \label{eq:gdc_nominal_flow}
\end{equation}
using $s_\theta = -\nabla_z\Phi_{\mathrm{data}}$; the parameter $\lambda$ controls the
physics--data balance as in \eqref{eq:interpolated_potential}. This matches the online
latent flow \eqref{eq:online_latent_flow} used by the deployed controller.

\subsubsection{Discrete-time implementation}

For computational implementation, the flow \eqref{eq:gdc_nominal_flow} is discretized
on the learned map:
\begin{equation}
  z_{t+1} = z_t \oplus \bigl(\dot{z}_{\text{nom},t}\,\Delta t\bigr),
  \label{eq:gdc_discrete_update}
\end{equation}
where $\oplus$ denotes ordinary addition in a flat latent chart or a standard retraction
for curved coordinates. Feasibility is maintained by selecting and refining within the
anchored support $\Mhat_{\mathrm{known}}(o,g)$ and by the known-constraint potential
$\Phi_{\mathrm{known}}$.

\subsection{Langevin control dynamics on the empirical manifold}
\label{sec:langevin}

\subsubsection{Manifold Langevin SDE}

The deterministic Riemannian gradient flow \eqref{eq:riemannian_gradient_flow} is
augmented with a stochastic noise term to obtain the \emph{manifold Langevin SDE}:
\begin{equation}
  \dd z_t = -\grad_G \Phi_{\mathrm{GDC}}(z_t)\,\dd t
  + \sqrt{2\beta_{\mathrm{L}}^{-1}}\,G(z_t)^{-1/2}\,\dd W_t,
  \label{eq:manifold_langevin}
\end{equation}
where $W_t$ is a standard Brownian motion in $\R^{d_z}$, $\beta_{\mathrm{L}} > 0$ is an
inverse temperature parameter (a sampling knob, unrelated to the offline progress tilt
$\beta$), and $G(z)^{-1/2}$ ensures the noise respects the Riemannian
geometry.
Equation~\eqref{eq:manifold_langevin} is an intrinsic shorthand: the diffusion is the
Riemannian Langevin process whose generator acts on smooth test functions $\varphi$ as
\begin{equation}
  \mathcal L\varphi
  =
  -\inner{\grad_G\Phi_{\mathrm{GDC}}}{\grad_G\varphi}_G
  +\beta_{\mathrm{L}}^{-1}\Delta_G\varphi .
  \label{eq:langevin_generator}
\end{equation}
In local coordinates this generator includes the usual metric/Christoffel drift terms
associated with Brownian motion on $(\Mhat,G)$; the compact notation above suppresses
those coordinate artifacts.

\subsubsection{Stochastic score-guided control on the empirical manifold}

Incorporating the score decomposition, the stochastic control dynamics on
$\Mhat$ use $H_t^\star=H^\star(o_t,g,\Cknown)$ and are:
\begin{equation}
  \dd z_t = G(z_t)^{-1}\!\left[
    s_\theta(z_t, o_t, g, \Cknown,H_t^\star)
    - \nabla \Phi_{\mathrm{known}}(z_t)
    - \nabla \Phi_{\mathrm{task}}(z_t)
  \right]\dd t
  + \sqrt{2\beta_{\mathrm{L}}^{-1}}\,G(z_t)^{-1/2}\,\dd W_t.
  \label{eq:stochastic_manifold_control}
\end{equation}

The noise term plays the role of stochastic exploration on $\Mhat$:
at high temperature ($\beta_{\mathrm{L}} \to 0$), the dynamics explore broadly across behavior
modes on the manifold; at low temperature ($\beta_{\mathrm{L}} \to \infty$), the dynamics reduce
to the deterministic Riemannian gradient flow.

\subsubsection{Euler--Maruyama discretization}

The stochastic dynamics can be discretized by an Euler--Maruyama step in any chosen
latent chart:
\begin{equation}
  v_t = \Delta t\,G(z_t)^{-1}\!\left[
    s_\theta(z_t, o_t, g, \Cknown,H_t^\star)
    - \nabla \Phi_{\mathrm{known}}(z_t)
    - \nabla \Phi_{\mathrm{task}}(z_t)
  \right]
  + \sqrt{2\Delta t\beta_{\mathrm{L}}^{-1}}\,G(z_t)^{-1/2}\,\xi_t,
\end{equation}
\begin{equation}
  z_{t+1} = z_t + v_t,
  \label{eq:euler_maruyama}
\end{equation}
where $\xi_t \sim \mathcal{N}(0, I)$. If a particular implementation uses
group-valued coordinates, the additive update can be replaced by the corresponding
retraction, but this is not a core component of GDC.

\subsection{Fokker--Planck analysis on the empirical manifold}
\label{sec:fokker_planck}

\subsubsection{The Fokker--Planck equation on the manifold}

Let $p(z, t)$ denote the probability density of the process
\eqref{eq:stochastic_manifold_control} at time $t$ with respect to the Riemannian
volume form on $\Mhat$. The evolution satisfies:
\begin{equation}
  \frac{\partial p}{\partial t} = \dive_G\!\left(p\,\grad_G \Phi_{\mathrm{GDC}}\right)
  + \beta_{\mathrm{L}}^{-1}\,\Delta_G\,p,
  \label{eq:manifold_fokker_planck}
\end{equation}
where $\dive_G$ and $\Delta_G$ are the Riemannian divergence and Laplace--Beltrami
operators on $(\Mhat, G)$.

\subsubsection{Stationary distribution}

\begin{proposition}[Stationary distribution on $\Mhat$]
  \label{prop:stationary}
  Under Assumption~\ref{ass:standing_theory}, suppose
  $Z=\int_{\Mhat}\exp(-\beta_{\mathrm{L}}\Phi_{\mathrm{GDC}}(z))\,\dd\volG(z)<\infty$. Then the
  manifold Langevin control dynamics
  have the Gibbs stationary distribution
  \begin{equation}
    \pi^\star(z) = \frac{1}{Z}\,\exp(-\beta_{\mathrm{L}}\,\Phi_{\mathrm{GDC}}(z)),
    \qquad
    Z = \int_{\Mhat} \exp(-\beta_{\mathrm{L}}\,\Phi_{\mathrm{GDC}}(z))\,\dd\volG(z).
    \label{eq:manifold_gibbs}
  \end{equation}
  If the data-supported region is connected and the diffusion is uniformly elliptic and
  nonexplosive there, this stationary distribution is unique; in particular, this holds
  on a compact connected region with reflecting/no-flux boundary.
\end{proposition}

\begin{proof}
  The Fokker--Planck equation can be written as a conservation law
  $\partial_t p=-\dive_G J_p$ with probability flux
  \[
    J_p=-p\,\grad_G\Phi_{\mathrm{GDC}}-\beta_{\mathrm{L}}^{-1}\grad_G p .
  \]
  For $\pi^\star=Z^{-1}\exp(-\beta_{\mathrm{L}}\Phi_{\mathrm{GDC}})$,
  \[
    \grad_G \pi^\star
    =
    -\beta_{\mathrm{L}} \pi^\star \grad_G\Phi_{\mathrm{GDC}}.
  \]
  Substitution gives $J_{\pi^\star}=0$, so $\partial_t \pi^\star=0$ and the reflecting
  boundary condition, when a boundary is present, is automatically satisfied. Existence
  follows from the assumed finiteness of $Z$. For uniqueness, let $p$ be any smooth
  invariant density in the same no-flux domain and write $p=h\pi^\star$. Since the
  generator \eqref{eq:langevin_generator} is reversible with respect to $\pi^\star$, the
  stationarity equation tested against $h$ gives
  $\int_{\Mhat}\|\grad_G h\|_G^2\,\pi^\star\,\dd\volG=0$. Hence $h$ is constant on each
  connected component. Connectedness and normalization force $h\equiv1$, so the
  stationary density is unique.
\end{proof}

\subsubsection{Implications for control}

The stationary distribution reveals:
\begin{equation}
  \pi^\star(z) \propto \exp\!\left(\beta_{\mathrm{L}}\,\log p_\theta^Z(z \mid o, g, \Cknown,H^\star)\right)
  \cdot \exp\!\left(-\beta_{\mathrm{L}}\,\Phi_{\mathrm{known}}(z)\right)
  \cdot \exp\!\left(-\beta_{\mathrm{L}}\,\Phi_{\mathrm{task}}(z)\right),
\end{equation}
which is the distribution over behaviors on $\Mhat$ that balances data consistency,
known-constraint satisfaction, and task optimality.

\subsubsection{Free energy and convergence}

Define the free energy functional on $\Mhat$:
\begin{equation}
  \mathcal{F}[p] = \E_p[\Phi_{\mathrm{GDC}}] - \beta_{\mathrm{L}}^{-1} H_G(p),
\end{equation}
where $H_G(p) = -\int_{\Mhat} p \log p\,\dd\volG$ is the differential
entropy with respect to the Riemannian volume form.

\begin{proposition}[Free energy descent on $\Mhat$]
  For smooth positive solutions $p_t$ of the Fokker--Planck evolution
  \eqref{eq:manifold_fokker_planck} satisfying the no-flux boundary condition when a
  boundary is present,
  \begin{equation}
    \begin{aligned}
    \frac{\dd}{\dd t}\mathcal{F}[p_t]
    &=
    -\int_{\Mhat} p_t
    \norm{\grad_G\!\left(\Phi_{\mathrm{GDC}}+\beta_{\mathrm{L}}^{-1}\log p_t\right)}_G^2
    \,\dd\volG \\
    &=
    -\beta_{\mathrm{L}}^{-2}\int_{\Mhat} p_t\,
    \norm{\grad_G \log\frac{p_t}{\pi^\star}}^2_G\,
    \dd\volG \le 0 .
    \end{aligned}
  \end{equation}
\end{proposition}

\begin{proof}
Using $H_G(p)=-\int p\log p\,\dd\volG$, the free energy is
\[
  \mathcal F[p]
  =
  \int_{\Mhat}p\Phi_{\mathrm{GDC}}\,\dd\volG
  +\beta_{\mathrm{L}}^{-1}\int_{\Mhat}p\log p\,\dd\volG .
\]
Its first variation is
\[
  \frac{\delta\mathcal F}{\delta p}
  =
  \Phi_{\mathrm{GDC}}+\beta_{\mathrm{L}}^{-1}(\log p+1).
\]
The additive constant has zero gradient. The Fokker--Planck equation is equivalently
\[
  \partial_t p
  =
  \dive_G\!\left(p\,\grad_G
  \left(\Phi_{\mathrm{GDC}}+\beta_{\mathrm{L}}^{-1}\log p\right)\right).
\]
Therefore, integrating by parts on a boundaryless manifold or using the no-flux boundary
condition,
\[
  \frac{\dd}{\dd t}\mathcal F[p_t]
  =
  \int_{\Mhat}\frac{\delta\mathcal F}{\delta p}\,\partial_t p_t\,\dd\volG
  =
  -\int_{\Mhat}p_t
  \norm{\grad_G\!\left(\Phi_{\mathrm{GDC}}+\beta_{\mathrm{L}}^{-1}\log p_t\right)}_G^2
  \dd\volG .
\]
Since
$\log(p_t/\pi^\star)=\log p_t+\beta_{\mathrm{L}}\Phi_{\mathrm{GDC}}+\log Z$, its gradient is
$\beta_{\mathrm{L}}\,\grad_G(\Phi_{\mathrm{GDC}}+\beta_{\mathrm{L}}^{-1}\log p_t)$, which gives the second
identity and nonpositivity.
\end{proof}

\subsection{Trajectory-level probabilistic path control}
\label{sec:trajectory}

\subsubsection{Receding-horizon trajectory control}

When inter-step constraints couple multiple time steps, trajectory-level planning
on the empirical manifold is preferable. This section extends GDC to the trajectory
manifold $\widehat{\mathcal{M}}^{H,*}$: known dynamics and constraints define
the scaffold, while feasible trajectory data with progress certificates complete the
implicit trajectory-level coupling structure.

\subsubsection{Trajectory path model on the empirical horizon manifold}

Define a trajectory of actions over horizon $H$:
$\tau = (u_t, \ldots, u_{t+H-1}) \in \mathcal{U}^{H}$.
The empirical trajectory manifold $\widehat{\mathcal{M}}^{H,*}$ is constructed by
encoding and decoding feasible horizon action sequences, while known dynamics are used
only where they are available to check or penalize decoded sequences:
\begin{equation}
  \mathcal{L}_{\text{horizon-known}}
  =
  \frac{1}{N}\sum_{k=1}^{N}
  \sum_r
  \left[
    f_r^{\mathrm{known}}\!\left(D_\psi^H(z_k,o_k,g_k,\Cknown),o_k,g_k\right)_+
  \right]^2,
\end{equation}
where $D_\psi^H$ decodes an $H$-step action sequence. If a known rollout map is available,
$f_r^{\mathrm{known}}$ may include its dynamic-consistency residuals; if not, it includes
only the available safety, actuator, timing, or resource constraints. Residual couplings
not captured by $f_{\mathrm{known}}$ are learned from feasible trajectory data through
$\Phi_{\mathrm{data}}$.

We train a conditional probabilistic path model on $\widehat{\mathcal{M}}^{H,*}$:
\begin{equation}
  p_\theta^\tau(\tau \mid o_t, g, \Cknown,H^\star)
\end{equation}
with the same clean-latent prediction objective \eqref{eq:path_embedding_loss}, replacing
$z$ with the trajectory latent code on $\widehat{\mathcal{M}}^{H,*}$ and using
$D_\psi^H$ to decode the selected clean latent code into an action sequence.

\subsubsection{Trajectory energy and score}

\begin{definition}[Trajectory energy on $\widehat{\mathcal{M}}^{H,*}$]
  \begin{equation}
    \Phi_\tau(\tau) =
    -\log p_\theta^\tau(\tau \mid o_t, g, \Cknown,H^\star)
    + \Phi_{\mathrm{known}}(\tau)
    + \Phi_{\mathrm{task}}(\tau),
  \end{equation}
  where $\Phi_{\mathrm{known}}(\tau)$ encodes trajectory-level known constraints and
  $\Phi_{\mathrm{task}}(\tau)$ encodes trajectory-level objectives.
\end{definition}

Let
$s_\theta^\tau(\tau,o_t,g,\Cknown,H^\star)
=\nabla_\tau\log p_\theta^\tau(\tau \mid o_t,g,\Cknown,H^\star)$.
The score-guided trajectory dynamics on $\widehat{\mathcal{M}}^{H,*}$ are:
\begin{equation}
  \dot{\tau} = G_\tau(\tau)^{-1}\!\left[
    s_\theta^\tau(\tau, o_t, g, \Cknown,H^\star)
    - \nabla_\tau \Phi_{\mathrm{known}}(\tau)
    - \nabla_\tau \Phi_{\mathrm{task}}(\tau)
  \right],
\end{equation}
where $G_\tau$ is the Riemannian metric on the trajectory manifold.

\subsubsection{Receding horizon execution}

Only the first action is executed:
$u_t = \tau_0^\star$,
where $\tau^\star = \argmin_\tau \Phi_\tau(\tau)$ on $\widehat{\mathcal{M}}^{H,*}$.
The horizon shifts and the process repeats, yielding a receding-horizon controller
that enforces known trajectory-level constraints and uses the path model's
distributional knowledge to approximate implicit trajectory information.

\subsubsection{Progress consistency across horizon}

The trajectory-level formulation provides the sharpest statement of why GDC
resolves directional ambiguity. Define the implicit progress value
$V_{\mathrm{prog}}(z,g) = -\log p_\theta^\tau(\text{progress} \mid z,g,\Cknown,H^\star)$ as in
Definition~\ref{def:implicit_value}.

\begin{proposition}[Progress bias under trajectory sampling]
  \label{prop:value_decrease}
  Let $q_\beta(\tau\mid z,g,\Cknown,H^\star)$ be the progress-tilted known-feasible target
  distribution at a selected horizon. Suppose that, on a data-supported region
  $U\subset\Mhat$, there is a margin $\gamma>0$ such that
  \[
    \E_{q_\beta(\cdot\mid z,g,\Cknown,H^\star)}[\Delta\rho(\tau,g)]\ge\gamma
    \qquad \text{for all }z\in U .
  \]
  Suppose also that the learned trajectory path satisfies
  \[
    \left\|
    p_\theta^\tau(\cdot\mid z,g,\Cknown,H^\star)
    -q_\beta(\cdot\mid z,g,\Cknown,H^\star)
    \right\|_{\TV}
    \le \delta
    \qquad \text{for all }z\in U,
  \]
  with $2R_\rho\delta<\gamma$. Then, for latent points within $U$,
  \begin{equation}
    \boxed{\E_{\tau \sim p_\theta^\tau(\cdot \mid z, g, \Cknown,H^\star)}
    \!\left[\Delta \rho(\tau,g)\right] > 0.}
    \label{eq:value_decrease}
  \end{equation}
  That is, the learned short-horizon trajectory distribution is biased toward
  known-feasible progress in expectation.
\end{proposition}

\begin{proof}
  Fix $z\in U$ and abbreviate $p=p_\theta^\tau(\cdot\mid z,g,\Cknown,H^\star)$ and
  $q=q_\beta(\cdot\mid z,g,\Cknown,H^\star)$. Since
  $|\Delta\rho(\tau,g)|\le R_\rho$ by Assumption~\ref{ass:standing_theory},
  the total-variation variational characterization gives
  \[
    \left|
    \E_p[\Delta\rho(\tau,g)]-\E_q[\Delta\rho(\tau,g)]
    \right|
    \le
    2R_\rho\|p-q\|_{\TV}
    \le 2R_\rho\delta .
  \]
  Therefore
  \[
    \E_p[\Delta\rho(\tau,g)]
    \ge
    \E_q[\Delta\rho(\tau,g)]-2R_\rho\delta
    \ge
    \gamma-2R_\rho\delta
    >0 .
  \]
  The conclusion is conditional on data support and on the fitting error bound; it does
  not assert global optimality or final-goal success for every sampled short segment.
\end{proof}

\begin{remark}[Why this resolves directional ambiguity with short horizons]
  Equation \eqref{eq:value_decrease} is the key formula. It says: even though
  each planned trajectory is only $H$ steps long, the \emph{distribution} over such
  trajectories is biased toward local progress. The directional information is not computed by planning over a long
  horizon; it is compressed into the \emph{distribution} over short trajectories because
  the training objective uses signed progress certificates. At the intersection, if
  right-turning snippets have higher expected progress under similar goals and contexts,
  the distribution concentrates more mass on right-turning short trajectories while still
  allowing waiting or yielding behavior when supported by the data.

  This is fundamentally different from model-predictive control, which must plan
  over a long enough horizon to ``see'' the goal. GDC does not need a long horizon
  to produce a useful local direction because the progress bias is already baked into
  the distribution. It still does not claim that every sampled short segment reaches
  the final goal.
\end{remark}

\subsection{Stability analysis}
\label{sec:stability}

\subsubsection{Lyapunov stability of Riemannian gradient flow}

\begin{theorem}[Lyapunov stability on $\Mhat$]
  \label{thm:lyapunov}
  Let $\Phi_{\mathrm{GDC}} : \Mhat \to \R$ be continuously differentiable and bounded below,
  with $z^\star$ an interior strict local minimum of the data-supported region. Consider
  the Riemannian gradient flow \eqref{eq:riemannian_gradient_flow}. Then:
  \begin{enumerate}[label=(\roman*)]
    \item $z^\star$ is a Lyapunov stable equilibrium on $\Mhat$.
    \item Every trajectory in a sublevel set $\{z \in \Mhat : \Phi_{\mathrm{GDC}}(z) \le c\}$
          remains in that set.
    \item If $\Phi_{\mathrm{GDC}}$ is $m$-strongly geodesically convex near $z^\star$, then
          $z^\star$ is exponentially stable with rate $m$:
          \begin{equation}
            d_G(z_t, z^\star) \le e^{-mt}\,d_G(z_0, z^\star),
          \end{equation}
          where $d_G$ is the geodesic distance on $(\Mhat, G)$.
  \end{enumerate}
\end{theorem}

\begin{proof}
  Since $z^\star$ is an interior strict local minimum and $\Phi_{\mathrm{GDC}}$ is
  continuous, there is a neighborhood $U$ of $z^\star$ such that
  $V(z)=\Phi_{\mathrm{GDC}}(z)-\Phi_{\mathrm{GDC}}(z^\star)$ is positive on
  $U\setminus\{z^\star\}$ and $V(z^\star)=0$. Since $z^\star$ is an interior local
  minimum of a $C^1$ function on the local chart, $\grad_G\Phi_{\mathrm{GDC}}(z^\star)=0$,
  so $z^\star$ is an equilibrium. Along the Riemannian gradient flow,
  \begin{equation}
    \frac{\dd}{\dd t}V(z_t)
    = \dd\Phi_{\mathrm{GDC}}(z_t)[\dot z_t]
    = \inner{\grad_G \Phi_{\mathrm{GDC}}(z_t)}{\dot{z}_t}_G
    = -\norm{\grad_G \Phi_{\mathrm{GDC}}(z_t)}^2_G \le 0.
  \end{equation}
  This proves Lyapunov stability by the standard Lyapunov sublevel argument: for any
  sufficiently small neighborhood $U_0$ of $z^\star$, choose a compact sublevel set of
  $V$ contained in $U_0$; the trajectory cannot leave it because $V$ is nonincreasing.
  The same calculation proves (ii), since every sublevel set of
  $\Phi_{\mathrm{GDC}}$ is forward invariant.

  For (iii), restrict to a geodesically convex normal neighborhood in which the squared
  distance is smooth. Let $W(t)=\frac12 d_G(z_t,z^\star)^2$. The first-variation formula
  gives
  \[
    \dot W(t)
    =
    \inner{-\mathrm{Log}_{z_t}(z^\star)}{\dot z_t}_G
    =
    \inner{\mathrm{Log}_{z_t}(z^\star)}{\grad_G\Phi_{\mathrm{GDC}}(z_t)}_G .
  \]
  Geodesic $m$-strong convexity with minimizer $z^\star$ implies
  $\inner{\grad_G\Phi_{\mathrm{GDC}}(z_t)}{-\mathrm{Log}_{z_t}(z^\star)}_G
  \ge m\,d_G(z_t,z^\star)^2$. Hence
  $\dot W(t)\le -m\,d_G(z_t,z^\star)^2=-2mW(t)$. Gronwall's inequality gives
  $W(t)\le e^{-2mt}W(0)$, and therefore
  $d_G(z_t,z^\star)\le e^{-mt}d_G(z_0,z^\star)$.
\end{proof}

\subsubsection{Robustness to score estimation error}

Define the score estimation error
$\delta(z, o, g, \Cknown,H^\star) =
s_\theta(z, o, g, \Cknown,H^\star) - s^\star(z, o, g, \Cknown,H^\star)$.

\begin{assumption}[Bounded score error]
  \label{ass:score_error}
  The score error is uniformly bounded in the dual Riemannian norm:
  \[
    \sup_{z \in \Mhat, o, g}
    \norm{\delta(z, o, g, \Cknown,H^\star)}_{G^{-1}}
    \le \varepsilon .
  \]
\end{assumption}

\begin{theorem}[Robustness on $\Mhat$]
  \label{thm:robustness}
  Under Assumption~\ref{ass:score_error}, suppose the true total potential is
  $m$-strongly geodesically convex on a geodesically convex data-supported neighborhood
  containing the trajectory and has minimizer $z^\star$.
  Then the approximate dynamics converge to a neighborhood:
  \begin{equation}
    \limsup_{t \to \infty} d_G(z_t, z^\star) \le \frac{\varepsilon}{m}.
  \end{equation}
\end{theorem}

\begin{proof}
  Let $e(z)=G(z)^{-1}\delta(z,o,g,\Cknown,H^\star)$, so
  $\|e(z)\|_G=\|\delta(z,o,g,\Cknown,H^\star)\|_{G^{-1}}\le\varepsilon$. In the geodesically
  convex neighborhood, use $V(z)=\frac12 d_G(z,z^\star)^2$. The perturbed Riemannian
  gradient flow is
  $\dot z=-\grad_G\Phi^\star(z)+e(z)$. By the first-variation formula,
  \begin{align}
    \dot V
    &=
    \inner{-\mathrm{Log}_{z}(z^\star)}{-\grad_G\Phi^\star(z)+e(z)}_G \notag\\
    &=
    \inner{\mathrm{Log}_{z}(z^\star)}{\grad_G\Phi^\star(z)}_G
    -\inner{\mathrm{Log}_{z}(z^\star)}{e(z)}_G \notag\\
    &\le
    -m\,d_G(z,z^\star)^2+\varepsilon\,d_G(z,z^\star) \notag\\
    &=
    -d_G(z,z^\star)\left(m\,d_G(z,z^\star)-\varepsilon\right).
  \end{align}
  Thus $\dot V<0$ whenever $d_G(z,z^\star)>\varepsilon/m$. For every $\eta>0$, the
  closed exterior of the ball of radius $\varepsilon/m+\eta$ has strictly negative
  derivative except possibly at its boundary, so trajectories cannot remain outside that
  enlarged ball indefinitely. Letting $\eta\downarrow0$ gives the stated limsup bound.
\end{proof}

\subsubsection{Non-accumulation of discretization drift}

Local refinement introduces a small discretization drift per step. The next statement
records that the retraction/backtracking guard of the online controller is what keeps
this local integration error uniformly controlled; it is a design guarantee of that
guard, not a consequence of the integrator alone.

\begin{proposition}[Non-accumulation of drift]
  \label{prop:drift}
  Let $\Phi_{\mathrm{scaf}}(z)=d_G(z,\mathcal M_t)^2$, where $\mathcal M_t$ is the local
  known-feasible latent scaffold. Suppose each accepted latent update starts in
  $\mathcal M_t$, has capped metric length at most $V_{\max}\Delta t$, and applies the
  retraction/backtracking guard
  $d_G(z_{t+1},\mathcal M_t)\le \sqrt{\epsilon_{\mathrm{s}}/2}\,V_{\max}\Delta t$. Then the manifold
  deviation satisfies, for every $t$,
  \begin{equation}
    \Phi_{\mathrm{scaf}}(z_t) \le \frac{\epsilon_{\mathrm{s}}}{2}\,V_{\max}^2\,\Delta t^2.
  \end{equation}
  The bound is uniform in $t$, vanishes as $\Delta t \to 0$, and does not accumulate.
\end{proposition}

\begin{proof}
  The capped integrator guarantees that the raw update moves by at most
  $V_{\max}\Delta t$ in the metric. The scaffold correction is accepted only when the
  post-correction point has distance at most
  $\sqrt{\epsilon_{\mathrm{s}}/2}\,V_{\max}\Delta t$ from $\mathcal M_t$. Since
  $\Phi_{\mathrm{scaf}}$ is the squared distance to $\mathcal M_t$,
  \[
    \Phi_{\mathrm{scaf}}(z_t)
    =
    d_G(z_t,\mathcal M_t)^2
    \le
    \frac{\epsilon_{\mathrm{s}}}{2}V_{\max}^2\Delta t^2 .
  \]
  This is a per-step guard on the accepted latent state rather than a sum of local
  truncation errors. Hence the same upper bound holds uniformly over time and cannot
  grow with the rollout horizon.
\end{proof}

\subsubsection{End-to-end action error bound}

\begin{theorem}[End-to-end action error bound]
  \label{thm:end_to_end}
  Under the conditions of Theorem~\ref{thm:robustness} and Proposition~\ref{prop:drift},
  the asymptotic receding-horizon action error, with $K_{D_\psi}$-Lipschitz decoder and
  decoder error $\delta_{\text{dec}}$, satisfies
  \begin{equation}
    \limsup_{t\to\infty}\norm{u_t - u_t^\star} \le K_{D_\psi}\!\left(
    \sqrt{\frac{\epsilon_{\mathrm{s}}}{2}}\,V_{\max}\,\Delta t + \frac{\varepsilon}{m}\right)
    + \delta_{\text{dec}}.
    \label{eq:end_to_end_bound}
\end{equation}
  The first term is the geometric integration drift, the second is the score
  approximation error, and the third is the offline decoder error. All three are
  independently controllable and do not accumulate over time. The same bound holds
  pointwise for any accepted update whose in-scaffold refinement point lies within the
  $\varepsilon/m$ robustness tube around the ideal latent point.
\end{theorem}

\begin{proof}
  Let $z_t^\star$ be the ideal latent point whose first decoded action is
  $u_t^\star$, and let $\bar z_t$ be the in-scaffold point reached by the approximate
  dynamics. Theorem~\ref{thm:robustness} gives
  $\limsup_{t\to\infty}d_G(\bar z_t,z_t^\star)\le \varepsilon/m$.
  Proposition~\ref{prop:drift} bounds the distance from the accepted numerical latent
  point $z_t$ to $\bar z_t$ by
  $\sqrt{\epsilon_{\mathrm{s}}/2}\,V_{\max}\Delta t$. By the triangle inequality in the local
  geodesically convex neighborhood,
  \[
    \limsup_{t\to\infty}d_G(z_t,z_t^\star)
    \le
    \sqrt{\frac{\epsilon_{\mathrm{s}}}{2}}\,V_{\max}\Delta t+\frac{\varepsilon}{m}.
  \]
  The decoder is $K_{D_\psi}$-Lipschitz on this neighborhood, so the first decoded
  actions satisfy
  \[
    \|D_\psi(z_t,c)_0-D_\psi(z_t^\star,c)_0\|
    \le
    K_{D_\psi}d_G(z_t,z_t^\star).
  \]
  Adding the decoder reconstruction error $\delta_{\text{dec}}$ and taking limsup gives
  \eqref{eq:end_to_end_bound}. The final known-constraint projection onto the convex set
  $\mathcal U_{\mathrm{known}}(o_t,u_{t-1})$ is non-expansive and leaves the feasible ideal
  action fixed, so it does not enlarge the bound; when $\mathcal U_{\mathrm{known}}$ is
  nonconvex one adds a projection term $\delta_{\mathrm{proj}}$ as in
  Theorem~\ref{thm:asymptotic_optimality}. If an accepted update satisfies the
  robustness-tube condition pointwise, the same triangle-inequality argument gives the
  pointwise version. The estimate is independent of the number of future rollout steps.
\end{proof}

\subsubsection{Approximate optimality under value consistency}

The previous bounds compare GDC to an ideal latent controller induced by the learned
potential. To turn this into an optimality statement, one additional assumption is
necessary: the data-induced potential must approximate the same constrained objective
whose optimum we want to recover. The following theorem makes this conditional statement
explicit.

\begin{definition}[Oracle constrained horizon objective]
  \label{def:oracle_objective}
  For a selected context $c=(o,g,\Cknown,H^\star)$ and data-supported region
  $U_c\subset\Mhat$, let $\Gamma_H(z,c)$ be the set of known-feasible horizon-$H^\star$
  continuations starting from latent point $z$. The oracle constrained horizon objective
  is a function $\Phi_{\mathrm{opt}}(\cdot,c):U_c\to\R$ whose minimizers are the
  constrained optimal latent decisions on this region. In the deterministic case one may
  take
  \begin{equation}
    \Phi_{\mathrm{opt}}(z,c)
    =
    \inf_{\tau\in\Gamma_H(z,c)}
    C_{\mathrm{opt}}(\tau;c),
    \label{eq:oracle_objective_deterministic}
  \end{equation}
  where $C_{\mathrm{opt}}$ contains the task cost and all known and implicit feasibility
  costs. In the entropy-regularized case, the infimum is replaced by the soft minimum
  $-\beta_{\mathrm{opt}}^{-1}\log\int_{\Gamma_H(z,c)}
  \exp[-\beta_{\mathrm{opt}} C_{\mathrm{opt}}(\tau;c)]\,\dd\nu_z(\tau)$. The theorem
  below applies after fixing either convention. Additive constants in
  $\Phi_{\mathrm{opt}}$ are immaterial.
\end{definition}

\begin{theorem}[Asymptotic optimality]
  \label{thm:asymptotic_optimality}
  Fix a selected context $c=(o,g,\Cknown,H^\star)$ and a compact geodesically convex
  data-supported region $U_c\subset\Mhat$. Let
  $\Phi_{\mathrm{opt}}(\cdot,c):U_c\to\R$ be the oracle constrained horizon objective for
  the same task, horizon, and known-feasible scaffold, and assume it has a unique
  minimizer
  \[
    z^\star=\argmin_{z\in U_c}\Phi_{\mathrm{opt}}(z,c).
  \]
  For a sequence of learned GDC controllers indexed by $n$, write
  $\Phi_n(z,c)$ for the learned potential
  $\Phi_{\mathrm{GDC}}(z,c)$ and let $\hat z_n\in U_c$ be the latent point returned by
  online refinement. Suppose:
  \begin{enumerate}[label=(\roman*)]
    \item \textbf{Uniform value consistency up to normalization.} There are constants
          $a_n(c)$ and errors $\alpha_n\ge0$ such that
          \begin{equation}
            \sup_{z\in U_c}
            \left|
              \Phi_n(z,c)-\Phi_{\mathrm{opt}}(z,c)-a_n(c)
            \right|
            \le \alpha_n .
            \label{eq:value_consistency_sequence}
          \end{equation}
    \item \textbf{Approximate online minimization.} The refinement step is
          $\eta_n$-optimal on $U_c$:
          \begin{equation}
            \Phi_n(\hat z_n,c)
            \le
            \inf_{z\in U_c}\Phi_n(z,c)+\eta_n .
            \label{eq:eta_optimal_sequence}
          \end{equation}
    \item \textbf{Local curvature.} $\Phi_{\mathrm{opt}}(\cdot,c)$ is
          $m$-strongly geodesically convex on $U_c$:
          \begin{equation}
            \Phi_{\mathrm{opt}}(z,c)
            \ge
            \Phi_{\mathrm{opt}}(z^\star,c)
            +\frac{m}{2}d_G(z,z^\star)^2
            \qquad \text{for all }z\in U_c .
            \label{eq:oracle_strong_convexity}
          \end{equation}
  \end{enumerate}
  Then the oracle suboptimality gap and latent distance satisfy
  \begin{equation}
    0
    \le
    \Phi_{\mathrm{opt}}(\hat z_n,c)-\Phi_{\mathrm{opt}}(z^\star,c)
    \le
    2\alpha_n+\eta_n,
    \label{eq:oracle_gap_bound}
  \end{equation}
  and
  \begin{equation}
    d_G(\hat z_n,z^\star)
    \le
    \sqrt{\frac{2(2\alpha_n+\eta_n)}{m}} .
    \label{eq:oracle_distance_bound}
  \end{equation}
  If the first-action decoder is $K_{D_\psi}$-Lipschitz on $U_c$, its first-action
  reconstruction error is at most $\delta_{\mathrm{dec},n}$, and the final projection
  changes the decoded action by at most $\delta_{\mathrm{proj},n}$, then the executed
  action $u_{t,n}$ satisfies
  \begin{equation}
    \|u_{t,n}-u_t^\star\|
    \le
    K_{D_\psi}
    \sqrt{\frac{2(2\alpha_n+\eta_n)}{m}}
    +\delta_{\mathrm{dec},n}
    +\delta_{\mathrm{proj},n}.
    \label{eq:asymptotic_action_optimality}
  \end{equation}
  Consequently, if
  $\alpha_n\to0$, $\eta_n\to0$, $\delta_{\mathrm{dec},n}\to0$, and
  $\delta_{\mathrm{proj},n}\to0$, then
  \[
    \Phi_{\mathrm{opt}}(\hat z_n,c)\to \Phi_{\mathrm{opt}}(z^\star,c),
    \qquad
    d_G(\hat z_n,z^\star)\to0,
    \qquad
    \|u_{t,n}-u_t^\star\|\to0 .
  \]
\end{theorem}

\begin{proof}
  Fix $n$ and abbreviate
  $\widehat\Phi=\Phi_n(\cdot,c)$,
  $\Phi^\star=\Phi_{\mathrm{opt}}(\cdot,c)$, and
  $a=a_n(c)$. By the value-consistency bound
  \eqref{eq:value_consistency_sequence}, evaluated at $\hat z_n$,
  \[
    \Phi^\star(\hat z_n)
    \le
    \widehat\Phi(\hat z_n)-a+\alpha_n .
  \]
  The approximate-minimization condition \eqref{eq:eta_optimal_sequence} gives
  \[
    \widehat\Phi(\hat z_n)
    \le
    \widehat\Phi(z^\star)+\eta_n,
  \]
  since $z^\star\in U_c$. Applying \eqref{eq:value_consistency_sequence} again, now at
  $z^\star$, yields
  \[
    \widehat\Phi(z^\star)
    \le
    \Phi^\star(z^\star)+a+\alpha_n .
  \]
  Combining the three inequalities cancels the arbitrary normalization constant $a$ and
  gives
  \[
    \Phi^\star(\hat z_n)-\Phi^\star(z^\star)
    \le
    2\alpha_n+\eta_n .
  \]
  The lower bound in \eqref{eq:oracle_gap_bound} follows from optimality of
  $z^\star$. Strong geodesic convexity \eqref{eq:oracle_strong_convexity} then implies
  \[
    \frac{m}{2}d_G(\hat z_n,z^\star)^2
    \le
    \Phi^\star(\hat z_n)-\Phi^\star(z^\star)
    \le
    2\alpha_n+\eta_n,
  \]
  which proves \eqref{eq:oracle_distance_bound}.

  Let $D_\psi(z,c)_0$ denote the first decoded action. If
  $u_t^\star$ is the oracle first action decoded from $z^\star$, Lipschitz continuity on
  $U_c$ gives
  \[
    \|D_\psi(\hat z_n,c)_0-D_\psi(z^\star,c)_0\|
    \le
    K_{D_\psi}d_G(\hat z_n,z^\star).
  \]
  The decoder reconstruction error contributes at most $\delta_{\mathrm{dec},n}$, and
  the final projection contributes at most $\delta_{\mathrm{proj},n}$ by assumption.
  Substituting the distance bound gives \eqref{eq:asymptotic_action_optimality}. If the
  four error sequences vanish, the displayed convergence statements follow immediately.
\end{proof}

\begin{corollary}[Coverage and convergence rate]
  \label{cor:coverage_rate}
  In the setting of Theorem~\ref{thm:asymptotic_optimality}, let
  $\mathcal Z_n(c)\subset U_c$ be the finite set of latent anchors covered by the offline data
  for context $c$, and define its fill distance
  \begin{equation}
    h_n
    =
    \sup_{z\in U_c}\inf_{\zeta\in\mathcal Z_n(c)}d_G(z,\zeta).
    \label{eq:fill_distance}
  \end{equation}
  Suppose $\Phi_n(\cdot,c)$ and $\Phi_{\mathrm{opt}}(\cdot,c)$ are Lipschitz on $U_c$
  with constants $L_n$ and $L_{\mathrm{opt}}$, and that the learned potential has anchor
  value error
  \begin{equation}
    \sup_{\zeta\in\mathcal Z_n(c)}
    |\Phi_n(\zeta,c)-\Phi_{\mathrm{opt}}(\zeta,c)-a_n(c)|
    \le \varepsilon_n .
    \label{eq:anchor_value_error}
  \end{equation}
  If the online refinement after $K$ local optimization steps returns
  $\hat z_{n,K}$ with error $\eta_K$, then
  \begin{align}
    \Phi_{\mathrm{opt}}(\hat z_{n,K},c)-\Phi_{\mathrm{opt}}(z^\star,c)
    &\le
    2\varepsilon_n+2(L_n+L_{\mathrm{opt}})h_n+\eta_K,
    \label{eq:coverage_gap_rate}\\
    d_G(\hat z_{n,K},z^\star)
    &\le
    \sqrt{
      \frac{2\left[2\varepsilon_n+2(L_n+L_{\mathrm{opt}})h_n+\eta_K\right]}{m}
    },
    \label{eq:coverage_distance_rate}\\
    \|u_{t,n,K}-u_t^\star\|
    &\le
    K_{D_\psi}
    \sqrt{
      \frac{2\left[2\varepsilon_n+2(L_n+L_{\mathrm{opt}})h_n+\eta_K\right]}{m}
    }
    +\delta_{\mathrm{dec},n}+\delta_{\mathrm{proj},K}.
    \label{eq:coverage_action_rate}
  \end{align}
  Therefore, if $\sup_n L_n<\infty$,
  $\varepsilon_n=\mathcal O(r_n)$, $h_n=\mathcal O(s_n)$, and
  $\eta_K=\mathcal O(q_K)$, then the value gap is
  $\mathcal O(r_n+s_n+q_K)$ and the first-action error is
  \[
    \mathcal O\!\left(\sqrt{r_n+s_n+q_K}
    +\delta_{\mathrm{dec},n}+\delta_{\mathrm{proj},K}\right).
  \]
  In particular, quasi-uniform deterministic coverage of a $d$-dimensional region gives
  $h_n=\mathcal O(n^{-1/d})$. If anchors are sampled iid from a distribution whose
  density is bounded below on $U_c$, then
  $h_n=\mathcal O_p((\log n/n)^{1/d})$. With a linearly convergent local refinement
  $\eta_K\le C_\eta e^{-\kappa K}$, the value gap becomes
  \[
    \mathcal O\!\left(
      r_n+s_n+e^{-\kappa K}
    \right).
  \]
\end{corollary}

\begin{proof}
  For any $z\in U_c$, choose an anchor $\zeta_z\in\mathcal Z_n(c)$ with
  $d_G(z,\zeta_z)\le h_n$. The triangle inequality gives
  \begin{align*}
    &|\Phi_n(z,c)-\Phi_{\mathrm{opt}}(z,c)-a_n(c)|\\
    &\qquad\le
    |\Phi_n(z,c)-\Phi_n(\zeta_z,c)|
    +|\Phi_n(\zeta_z,c)-\Phi_{\mathrm{opt}}(\zeta_z,c)-a_n(c)|\\
    &\qquad\quad
    +|\Phi_{\mathrm{opt}}(\zeta_z,c)-\Phi_{\mathrm{opt}}(z,c)|\\
    &\qquad\le
    (L_n+L_{\mathrm{opt}})h_n+\varepsilon_n .
  \end{align*}
  Thus the uniform value-consistency error in
  \eqref{eq:value_consistency_sequence} can be chosen as
  $\alpha_n=\varepsilon_n+(L_n+L_{\mathrm{opt}})h_n$. Substituting this value of
  $\alpha_n$ into \eqref{eq:oracle_gap_bound}, \eqref{eq:oracle_distance_bound}, and
  \eqref{eq:asymptotic_action_optimality} yields
  \eqref{eq:coverage_gap_rate}--\eqref{eq:coverage_action_rate}. The big-O statements
  follow by replacing each error sequence with its rate.

  For iid coverage, cover $U_c$ by at most $C r^{-d}$ geodesic balls of radius $r/2$.
  If the sampling density is bounded below, each such ball has probability mass at least
  $c r^d$ for small enough $r$. A union bound gives
  \[
    \Pr[h_n>r]\le C r^{-d}\exp(-n c r^d).
  \]
  Choosing $r=A(\log n/n)^{1/d}$ with $A$ large enough makes this probability vanish,
  proving $h_n=\mathcal O_p((\log n/n)^{1/d})$.
\end{proof}

\begin{remark}[What the optimality theorem does and does not claim]
  Theorem~\ref{thm:asymptotic_optimality} is stronger than the feasible-progress
  guarantee and requires a stronger premise. It does not say that arbitrary
  progress-certified demonstrations are automatically globally optimal. It says that
  when the learned data potential becomes uniformly value-consistent with the oracle
  constrained objective on the selected data-supported region, and the online refinement
  error also vanishes, the GDC decision approaches the corresponding constrained optimal
  decision on that region.
\end{remark}

\subsection{Physics--data interpolation}
\label{sec:physics_data_interpolation}

\subsubsection{Interpolated control potential under partial physics}

\begin{definition}[Physics--data control potential]
  For $\lambda \in [0,1]$, the interpolated control potential on $\Mhat$ is
  \begin{equation}
    \Phi_\lambda(z) = \lambda\,\Phi_{\mathrm{known}}(z)
    + (1 - \lambda)\,\Phi_{\mathrm{data}}(z, o, g, \Cknown,H^\star)
    + \Phi_{\mathrm{task}}(z).
    \label{eq:interpolated_potential}
  \end{equation}
\end{definition}

\subsubsection{Interpolation theorem}

\begin{theorem}[Physics--data interpolation under partial physics]
  \label{thm:physics_data_interpolation}
  Let $\lambda \in [0,1]$ and define the control law
  \begin{equation}
    \dot{z} = -\grad_G \Phi_\lambda(z).
    \label{eq:interpolated_control}
  \end{equation}
  Then:
  \begin{enumerate}[label=(\roman*)]
    \item At $\lambda = 1$, \eqref{eq:interpolated_control} uses only the explicit
          known-physics potential and the task potential.
    \item At $\lambda = 0$, \eqref{eq:interpolated_control} uses the data-driven progress
          potential
          $\Phi_{\mathrm{data}} + \Phi_{\mathrm{task}}$, where $\Phi_{\mathrm{data}}$
          supplies implicit progress and dynamics information from feasible trajectory data
          with progress certificates.
    \item For $\lambda \in (0,1)$, \eqref{eq:interpolated_control} is a hybrid controller
          on $\Mhat$ interpolating between explicit known-physics control and
          data-driven completion, with known-feasible support, anchor penalties, and
          singular perturbation structure.
    \item The Riemannian gradient has the explicit form:
          \begin{equation}
            \grad_G \Phi_\lambda = G^{-1}\!\left[
              \lambda\,\nabla_z \Phi_{\mathrm{known}}
              - (1 - \lambda)\,s_\theta(z, o, g, \Cknown,H^\star)
              + \nabla_z \Phi_{\mathrm{task}}
            \right].
          \end{equation}
  \end{enumerate}
\end{theorem}

\begin{proof}
  The potential is affine in $\lambda$:
  \[
    \Phi_\lambda
    =
    \lambda\Phi_{\mathrm{known}}+(1-\lambda)\Phi_{\mathrm{data}}+\Phi_{\mathrm{task}} .
  \]
  Substituting $\lambda=1$ removes the data term and yields
  $\Phi_{\mathrm{known}}+\Phi_{\mathrm{task}}$, proving (i). Substituting $\lambda=0$
  removes the known-potential weight from this interpolated term and yields
  $\Phi_{\mathrm{data}}+\Phi_{\mathrm{task}}$, proving (ii). For any
  $\lambda\in(0,1)$, both $\Phi_{\mathrm{known}}$ and $\Phi_{\mathrm{data}}$ enter the same
  scalar potential on the same manifold $\Mhat$, while the support restriction,
  anchors, filters, and projection layer remain those of the known scaffold; hence the
  induced flow is the stated hybrid controller, proving (iii).

  For (iv), differentiate:
  \[
    \nabla_z\Phi_\lambda
    =
    \lambda\nabla_z\Phi_{\mathrm{known}}
    +(1-\lambda)\nabla_z\Phi_{\mathrm{data}}
    +\nabla_z\Phi_{\mathrm{task}}.
  \]
  Since $\Phi_{\mathrm{data}}=-\log p_\theta^Z$, we have
  $\nabla_z\Phi_{\mathrm{data}}=-s_\theta$. Multiplying by $G^{-1}$ gives the displayed
  Riemannian gradient formula.
\end{proof}

\subsubsection{Data-physics tradeoff path}

\begin{corollary}[Data-physics tradeoff path on $\Mhat$]
  Suppose the relevant data-supported region $K\subset\Mhat$ is compact and each
  $\Phi_\lambda$ is continuous on $K$. Then minimizers of $\Phi_\lambda(z)$ on $K$ exist
  for every $\lambda\in[0,1]$. If the minimizer is unique for each $\lambda$, the minima trace
  a continuous data-physics tradeoff path
  \begin{equation}
    \mathcal{P} = \left\{
      z^\star(\lambda) = \argmin_{z \in K}
      \Phi_\lambda(z)
      \;\middle|\;
      \lambda \in [0,1]
    \right\}.
  \end{equation}
\end{corollary}

\begin{proof}
  Existence follows from the Weierstrass theorem because $\Phi_\lambda$ is continuous on
  the compact region $K$. The joint map $(z,\lambda)\mapsto\Phi_\lambda(z)$ is continuous since
  it is affine in $\lambda$ and continuous in $z$. Berge's maximum theorem implies that
  the argmin correspondence is nonempty, compact-valued, and upper hemicontinuous. When
  the minimizer is unique for every $\lambda$, this correspondence is single-valued; a
  single-valued upper hemicontinuous argmin map on a compact parameter interval is
  continuous. Therefore $z^\star(\lambda)$ traces the stated continuous tradeoff path.
\end{proof}

\subsection{Online control used in the theory}

The stability and interpolation results below analyze the online controller defined in
Section~\ref{sec:online}. In particular, the learned models are frozen, the controller
operates in latent space, known constraints are checked after decoding through $D_\psi$,
and only the first action of the selected segment is executed.

\subsection{Wasserstein gradient flow interpretation}
\label{sec:wasserstein}

\subsubsection{GDC as Wasserstein gradient flow on the empirical manifold}

\begin{theorem}[GDC as Wasserstein gradient flow]
  \label{thm:wasserstein}
  Under Assumption~\ref{ass:standing_theory}, the Fokker--Planck evolution
  \eqref{eq:manifold_fokker_planck} on $\Mhat$
  is the Wasserstein gradient flow of the free energy:
  \begin{equation}
    \mathcal{F}[p] = \E_p\!\left[
      -\log p_\theta^Z(z \mid o, g, \Cknown,H^\star)
      + \Phi_{\mathrm{known}}(z)
      + \Phi_{\mathrm{task}}(z)
    \right]
    - \beta_{\mathrm{L}}^{-1} H_G(p),
  \end{equation}
  where the Wasserstein metric is computed with respect to geodesic distances on
  $(\Mhat, G)$.
\end{theorem}

\begin{proof}
  Write
  \[
    \mathcal F[p]
    =
    \int_{\Mhat}p(z)\Phi_{\mathrm{GDC}}(z)\,\dd\volG
    +\beta_{\mathrm{L}}^{-1}\int_{\Mhat}p(z)\log p(z)\,\dd\volG .
  \]
  Its first variation is
  \[
    \frac{\delta\mathcal F}{\delta p}
    =
    \Phi_{\mathrm{GDC}}+\beta_{\mathrm{L}}^{-1}(\log p+1).
  \]
  In the $2$-Wasserstein geometry induced by the Riemannian metric $G$, the gradient-flow
  equation for $\mathcal F$ is the continuity equation
  \[
    \partial_t p
    =
    \dive_G\!\left(p\,\grad_G\frac{\delta\mathcal F}{\delta p}\right).
  \]
  Dropping the additive constant in the first variation gives
  \[
    \partial_t p
    =
    \dive_G\!\left(p\,\grad_G\Phi_{\mathrm{GDC}}\right)
    +\beta_{\mathrm{L}}^{-1}\dive_G\!\left(p\,\grad_G\log p\right).
  \]
  Since $p\,\grad_G\log p=\grad_G p$, the second term is
  $\beta_{\mathrm{L}}^{-1}\Delta_G p$. This is exactly \eqref{eq:manifold_fokker_planck}.
\end{proof}

\emph{GDC is the steepest descent in probability space over $\Mhat$,
simultaneously minimizing the expected control potential (known physics, data-induced
implicit information, and task cost) and maximizing behavioral entropy on the manifold.}

\FloatBarrier


\end{document}